\documentclass[twoside,11pt]{article}

\usepackage{jmlr2e}
\usepackage{hyperref}
\usepackage{graphicx}
\usepackage{epstopdf}

\usepackage{caption}
\usepackage{subcaption}
\usepackage{amsmath}
\usepackage{graphicx}

\usepackage{algorithm,color}
\usepackage{algpseudocode}
\usepackage{amsfonts}
\usepackage{booktabs}
\usepackage{tikz}
\usepackage{tablefootnote}

\usepackage{url,color,amssymb,amsmath}

\newcommand{\BEAS}{\begin{eqnarray*}}
\newcommand{\EEAS}{\end{eqnarray*}}
\newcommand{\BEA}{\begin{eqnarray}}
\newcommand{\EEA}{\end{eqnarray}}
\newcommand{\BEQ}{\begin{equation}}
\newcommand{\EEQ}{\end{equation}}
\newcommand{\BIT}{\begin{itemize}}
\newcommand{\EIT}{\end{itemize}}
\newcommand{\BNUM}{\begin{enumerate}}
\newcommand{\ENUM}{\end{enumerate}}
\newcommand{\BA}{\begin{array}}
\newcommand{\EA}{\end{array}}
\newcommand{\diag}{\mathop{\rm diag}}
\newcommand{\Diag}{\mathop{\rm Diag}}

\newcommand{\Var}{\mathop{ \rm var{}}}
\newcommand{\argmin}{\mathop{\rm argmin}}

\newcommand{\Tr}{\mathop{ \rm tr}}
\newcommand{\tr}{\mathop{ \rm tr}}
\newcommand{\sign}{\mathop{ \rm sign}}
\newcommand{\idm}{I}
\newcommand{\rb}{\mathbb{R}}

\newcommand\RR{\mathbb{R}} 

\newcommand{\mysec}[1]{Section~\ref{sec:#1}}
\newcommand{\eq}[1]{Eq.~(\ref{eq:#1})}
\newcommand{\myfig}[1]{Figure~\ref{fig:#1}}

\def \E{{\mathbb E}}
\def \P{{\mathbb P}}

\def \E{{\mathbb E}}
\def \P{{\mathbb P}}

\def \t{  ^{ \top} }

\def\half{ \frac{1}{2} }
\def\cplus{ {c_+} }
\def\cplustrans{ {c_+^{\top}} }
\def\cminus{ {c_-} }
\def\cminustrans{ {c_-^{\top}} }
\def \calKM{ {{\mathcal K}{\mathcal M}} }

\def \baru{ {\bar{u}}}
\def \barC{ {\bar{C}}}

\ShortHeadings{}{Flammarion, Palaniappan and Bach}
\firstpageno{1}

\begin{document}
\title{Robust Discriminative Clustering with Sparse Regularizers}

\author{\name Nicolas Flammarion  \email nicolas.flammarion@ens.fr\\ 
\name Balamurugan Palaniappan \email balamurugan.palaniappan@inria.fr\\ 
\name Francis Bach  \email francis.bach@ens.fr \\ 
\addr
INRIA - Sierra Project-team \\
D\'epartement d'Informatique de l'Ecole Normale Sup\'erieure (CNRS - ENS - INRIA) \\
Paris, France}

\editor{}

\maketitle

\begin{abstract}
Clustering high-dimensional data often requires some form of dimensionality reduction, where clustered variables are separated from ``noise-looking'' variables. 
We cast this problem as finding a low-dimensional projection of the data which is well-clustered. This yields a one-dimensional projection in the simplest situation with two clusters, and extends naturally to a multi-label scenario for more than two clusters. 
In this paper, (a) we first show that this  joint clustering and dimension reduction formulation is equivalent to previously proposed discriminative clustering frameworks, thus leading to convex relaxations of the problem; (b) we propose  a novel  sparse extension, which is still cast as a convex relaxation and allows estimation in higher dimensions; (c) we propose a natural extension for the multi-label scenario; (d) we provide a new  theoretical analysis of the performance of these formulations with a simple probabilistic model, leading to scalings over the form $d=O(\sqrt{n})$ for the affine invariant case and $d=O(n)$ for the sparse case, where $n$ is the number of examples and $d$ the ambient dimension;  and finally, (e) we propose an efficient iterative algorithm with running-time complexity  proportional to $O(nd^2)$, improving on earlier algorithms which had quadratic complexity in the number of examples.
\end{abstract}

\section{Introduction}

  Clustering is an important and commonly used pre-processing tool in many machine learning applications, with classical algorithms such as $K$-means \citep{kmeans}, linkage algorithms \citep{single_linkage} or spectral clustering \citep{spectral_clustering}. 
  In high dimensions, these unsupervised learning algorithms typically have problems identifying the underlying optimal discrete nature of the data; for example, they are quickly perturbed by adding a few noisy dimensions.
 Clustering high-dimensional data thus requires some form of dimensionality reduction, where clustered variables are separated from ``noise-looking'' (e.g., Gaussian) variables. 

Several frameworks aim at linearly separating noise from signal, that is finding projections of the data that extracts the signal and removes the noise. They differ in the ways  signals and noise are defined. 
A line of work that dates back to projection pursuit \linebreak \citep{friedman1981projection} and independent component analysis \citep{hyvarinen2004independent} defines the noise as Gaussian while the signal is non-Gaussian \citep{blanchard2006search,roux2011local,diederichs2013sparse}. In this paper, we follow the work of  \citet{de2006discriminative, ding2007adaptive}, along the alternative route where one defines the signal as being clustered while the noise is any non-clustered variable.
 In the simplest situation with two clusters,  we may project the data into a one-dimensional subspace. Given a data matrix $X\in \rb^{n \times d}$ composed of $n$ $d$-dimensional points,  the goal is to find a direction $w \in \rb^d$ such that $X w \in \rb^n$ is well-clustered, e.g., by $K$-means.
This is equivalent to identifying both a direction to project, represented as $w \in \rb^d$ and the labeling $y \in \{-1,1\}^n$ that represents the partition into two clusters.

 Most existing formulations are non-convex and typically perform a form of alternating optimization \citep{de2006discriminative, ding2007adaptive}, where given $y \in \{-1,1\}^n$, the projection $w$ is found by linear discriminant analysis (or any binary classification method), and given the projection $w$, the clustering is obtained by thresholding $Xw$ or running $K$-means on $Xw$. As shown in \mysec{kmeans}, this alternating minimization procedure happens to be equivalent to maximizing the (centered) correlation between $y \in \{-1,1\}^n$ and the projection $X w \in \rb^d$, that is
 $$
\max_{ w \in \rb^d, y \in \{-1,1\}^n } \frac{ (y^\top \Pi_n  X w)^2}{\| \Pi_n  y\|_2^2 \, \| \Pi_n  Xw \|_2^2},
$$
where $\Pi_n = \idm_n - \frac{1}{n} 1_n 1_n^\top$ is the usual centering projection matrix (with $1_n \in \rb^n$ being the vector of all ones, and $\idm_n$ the $n \times n$ identity matrix).
This correlation is equal to one when the projection is perfectly clustered (independently of the number of elements per cluster). Existing methods  are alternating minimization algorithms with no theoretical guarantees.

In this paper, we relate this formulation to discriminative clustering formulations \citep{schuurmans,bach_diffrac_2007}, which consider  the problem 
\BEQ
\label{eq:diffracvb}
\min_{ v \in \rb^d, \ b \in \rb, \ y \in \{-1,1\}^n } \frac{1}{n}\|y - X v - b 1_n\|_2^2, 
\EEQ
with the intuition of finding labels $y$ which are easy to predict by an affine function of the data. In particular, we show that given the relationship between the number of positive labels and negative labels (i.e., the squared difference between the respective number of elements), these two problems are equivalent, and hence discriminative clustering explicitly performs joint dimension reduction and clustering.

While the discriminative framework is based on convex relaxations and has led to interesting developments and applications \citep{zhang2009maximum,joulin2010discriminative,joulin2010efficient,wang2010linear}, it has several shortcomings: 
(a) the running-time complexity of the semi-definite formulations is at least quadratic in $n$, and typically much more, (b) no theoretical analysis has ever been performed, (c) no convex sparse extension has been proposed to handle data with many irrelevant dimensions, (d) balancing of the clusters remains an issue, as it typically adds an extra hyperparameter which may be hard to set. In this paper, we focus on addressing these concerns. 

When there are more than two clusters, one considers either the \emph{multi-label} or the \emph{multi-class} settings. The multi-class problem assumes that the data are clustered into distinct classes, i.e., a single class per observation, whereas the multi-label problem assumes the data share  different labels, i.e., multiple labels per observation. We show in this work that discriminative clustering framework extends more naturally to multi-label scenarios and this extension will have the same convex relaxation. 

A summary of the contributions of this paper follows:

\BIT\setlength{\itemsep}{0pt}

\item In \mysec{kmeans}, we relate discriminative clustering  with the square loss to a joint clustering and dimension reduction formulation.  {The proposed formulation takes care of the balancing hyperparameter implicitly}. 

\item We propose in \mysec{sparse} a novel sparse extension to discriminative clustering and show that it can still be cast through a convex relaxation.

\item When there are more than two clusters, we extend naturally the sparse formulation to a multi-label scenario in \mysec{multilabel}. 

\item  {We then proceed to provide a theoretical analysis of the proposed formulations with a simple probabilistic model in \mysec{theory}, which effectively leads to scalings over the form $d=O(\sqrt{n})$ for the affine invariant case and $d=O(n)$ for the $1$-sparse case}.

\item {Finally, we propose in \mysec{algo} efficient iterative algorithms with running-time complexity for each step equal to $O(nd^2)$, the first to be linear in the number of observations $n$.}

\EIT

 Throughout this paper we assume that $X \in \rb^{n \times d}$ is \emph{centered}, a common pre-processing step in unsupervised (and supervised) learning. This implies that $X^\top 1_n = 0$ and 
$\Pi_n X  =X$.

 \section{Joint Dimension Reduction and Clustering}
 \label{sec:kmeans}

In this section, we focus on the single binary label case, where we first study the usual non-convex formulation, before deriving convex relaxations based on semi-definite programming. 

\subsection{Non-convex formulation}

Following \citet{de2006discriminative,ding2007adaptive,ye2008discriminative}, we consider a cost function which depends on $y \in \{-1,1\}^n$ and $w \in \rb^d$, which is such that alternating optimization is exactly (a) running $K$-means with two clusters on $Xw$ to obtain $y$ given~$w$ (when we say ``running $K$-means'', we mean solving the vector quantization problem exactly), and (b) performing linear discriminant analysis to obtain $w$ given $y$. 

\begin{proposition}[Joint clustering and dimension reduction]
\label{sec:maxcorr}
Given $X \in \rb^{n \times d}$ such that $X^\top 1_n = 0$ and $X$ has rank $d$, consider the optimization problem
\BEQ
\label{eq:maxcorr}
\max_{ w \in \rb^d, y \in \{-1,1\}^n } \frac{ (y^\top  X w)^2}{\| \Pi_n  y\|_2^2 \, \|  Xw \|_2^2}.
\EEQ
Given $y$, the optimal $w$ is obtained as $w = (X^\top X)^{-1} X^\top y$, while given $w$, the optimal $y$ is obtained by  running $K$-means on $Xw$.
\end{proposition}
\begin{proof}
Given $y$, we need to optimize the Rayleigh quotient $\frac{ w^\top X^\top y y^\top X w}{ w^\top X^\top X w}$ with a rank-one matrix in the numerator, which leads to $w=(X^\top X)^{-1} X^\top y$. 
Given $w$, we show in Appendix~\ref{app:kmean}, that the averaged distortion measure of $K$-means once the means have been optimized is exactly equal to $ { (y^\top    X w)^2}/{\| \Pi_n  y\|_2^2 }$.
\end{proof}

\paragraph{Algorithm.}
The proposition above leads to an alternating optimization algorithm. Note that $K$-means in one dimension may be run \emph{exactly} in $O(n \log n)$ \citep{bellman1973note}.
Moreover, after having optimized with respect to $w$ in \eq{maxcorr}, we then need to maximize with respect to $y$  the function
$\frac{ y^\top X (X^\top X)^{-1} X^\top y}{\| \Pi_n  y\|_2^2 }$, which happens to be exactly performing $K$-means on the whitened data (which is now in high dimension and not in 1 dimension). 
At first, it seems that dimension reduction is \emph{simply} equivalent to whitening the data and performing $K$-means; while this is a formally correct statement, the resulting $K$-means problem is not easy to solve as the clustered dimension is hidden in noise; for example, algorithms such as $K$-means++ \citep{arthur2007}, which have a multiplicative theoretical guarantee on the final distortion measure, are not provably effective here because the minimal final distortion is then not small, and the multiplicative guarantee is meaningless.

\subsection{Convex relaxation and discriminative clustering}
\label{sec:diffracy}

The discriminative clustering formulation in~\eq{diffracvb} may be optimized for any $y \in \{-1,1\}^n$ in closed form with respect to $b$ as $b=\frac{1_{n}\t(y-Xv)}{n}=\frac{1_{n}\t y}{n}$ since $X$ is centered. Substituting $b$ in~\eq{diffracvb} leads us to   
\BEQ
\label{eq:diffracequiv}
\min\limits_{ v \in \rb^d }  \frac{1}{n} \| \Pi_n  y - X v \|_2^2  =   \frac{1}{n} \| \Pi_n  y\|_2^2  - \max\limits_{ w \in \rb^d } \frac{ (y^\top X w)^2}{  \| Xw \|_2^2},
\EEQ
where $v$ is obtained from any solution $w$ as $v = w \frac{ y^\top X w} {\| X w\|_2^2}$. 
Thus, given 
\BEQ
\frac{(y^\top 1_n)^2}{n^2} =
\frac{1}{n^2}\big(
\#\{i, y_i=1\} - \#\{i, y_i=-1\}
\big)^2
=
 \alpha \in [0,1],
 \EEQ
 which characterizes the asymmetry between clusters and with $\Vert \Pi_n y \Vert ^2=n(1-\alpha)$,  we obtain from \eq{diffracequiv}, an equivalent formulation to \eq{maxcorr} (with the added constraint) as 
\BEQ
\label{eq:diffracyv}
\min_{ y \in \{-1,1\}^n, \ v \in \rb^d}   \frac{1}{n}  \| \Pi_n  y - X v \|_2^2 \ \mbox{ such that } \ \frac{(y^\top 1_n)^2}{n^2} = \alpha.
\EEQ
This is exactly equivalent to a discriminative clustering formulation with the square loss. Following \citet{bach_diffrac_2007}, we may optimize \eq{diffracyv} in closed form with respect to $v$ as $v = (X^\top X)^{-1} X^\top y$. Substituting $v$ in \eq{diffracyv} leads us to 
\BEQ
\label{eq:diffracyv2}
\min_{ y \in \{-1,1\}^n }    \frac{1}{n} y^\top
\big( \Pi_n  - X(X^\top X)^{-1} X^\top 
\big) y  \ \mbox{ such that } \ \frac{(y^\top 1_n)^2}{n^2} = \alpha.
\EEQ
This combinatorial optimization problem is NP-hard in general \citep{Kar72,GarJohSto76}. Hence in practice, it is classical to consider the following convex relaxation of~\eq{diffracyv2} \citep{relaxclas}. For any admissible $y \in \{-1,+1\}^n$, the matrix $Y=yy\t\in \rb^{n \times n}$ is a rank-one symmetric positive semi-definite matrix with  unit diagonal entries and conversely any such~$Y$ may be written in the form $Y=yy\t$ such that $y$ is admissible for~\eq{diffracyv2}. Moreover by rewriting~\eq{diffracyv2} as
$$
\min_{ y \in \{-1,1\}^n }    \frac{1}{n} \tr y y^\top
\big( \Pi_n  - X(X^\top X)^{-1} X^\top 
\big)  \ \mbox{ such that } \  \frac{1_{n}\t(yy^\top)1_{n} }{n^2} = \alpha,
$$
we see that the objective and constraints    are linear in the matrix  $Y=yy\t$ and ~\eq{diffracyv2} is equivalent to
\[
 \min_{ Y \succcurlyeq 0, \ {\rm rank}(Y)=1 \ \diag(Y) = 1 }    \frac{1}{n}\tr Y 
\big( \Pi_n  - X(X^\top X)^{-1} X^\top 
\big)   \mbox{ such that } \frac{1_n^\top Y 1_n }{n^2} = \alpha.
\]
Then  dropping the non-convex rank constraint leads us to the following classical convex relaxation: 
\begin{equation} \label{eq:diffracalpha}
\min_{ Y \succcurlyeq 0, \ \diag(Y) = 1 }    \frac{1}{n}\tr Y 
\big( \Pi_n  - X(X^\top X)^{-1} X^\top 
\big)   \mbox{ such that } \frac{1_n^\top Y 1_n }{n^2} = \alpha.
\end{equation}
This is the standard (unregularized) formulation, which is cast as a semi-definite program. The complexity of interior-point methods is $O(n^7)$, but efficient algorithms in $O(n^2)$ for such problems have been developed due to the relationship with the max-cut problem \citep{journee2010low,wen2012block}. 

Given the solution $Y$, one may traditionally obtain a candidate $y \in \{-1,1\}^n$ by running $K$-means on the largest eigenvector of $Y$ or by sampling  \citep{maxcut}. In this paper, we show in \mysec{theory} that it may be advantageous to consider the first two eigenvectors.
 
\subsection{Unsuccessful full convex relaxation}

The formulation in~\eq{diffracalpha} imposes an extra parameter $\alpha$ that characterises the cluster imbalance. It is tempting to find a direct relaxation of \eq{maxcorr}. It turns out to lead to a trivial relaxation, which we outline below. 

When optimizing \eq{maxcorr} with respect to $w$, we obtain the following optimization problem
$$
\displaystyle \max_{y \in \{-1,1\}^n}    \frac{ y^\top X (X^\top X )^{-1} X^\top y}{ y^\top \Pi_n  y},
$$
leading to a quasi-convex relaxation as
$$ \displaystyle
\max_{Y \succcurlyeq 0, \ \diag(Y) = 1}    \frac{\tr Y  X (X^\top X )^{-1} X^\top }{ \tr \Pi_n  Y},
$$
whose solution is found by solving a sequence of convex problems \citep[Section 4.2.5]{boyd}. 
As shown in Appendix~\ref{app:unsucces}, this may be exactly reformulated as a single convex problem:
$$ \displaystyle
\max_{M \succcurlyeq 0, \ \diag(M) = 1 + \frac{1^\top M 1}{n^2} }     \tr M  X (X^\top X )^{-1} X^\top.$$
Unfortunately, this relaxation always leads to trivial solutions, and we thus need to consider the relaxation in \eq{diffracalpha} for several values of $\alpha =   {1_n^\top Y 1_n }/{n^2}$ (and then the non-convex algorithm can be run from the rounded solution of the convex problem, using \eq{maxcorr} as a final objective). Alternatively, we may solve the following {\emph{penalized}} problem for several values of $\nu \geqslant 0$:
\BEQ
\label{eq:diffracnu}
\min_{ Y \succcurlyeq 0, \ \diag(Y) = 1 }    \frac{1}{n}\tr Y 
\big( \Pi_n  - X(X^\top X)^{-1} X^\top 
\big)   + \frac{\nu}{n^2}  {1_n^\top Y 1_n }.
\EEQ
For $\nu = 0$, $Y = 1_n 1_n^\top$ is always a trivial solution. 
As outlined in our theoretical section and as observed in our experiments, it is sufficient to consider $\nu \in [0,1]$.

\subsection{Equivalent relaxations}

Optimizing \eq{diffracyv} with respect to $v$ in closed form  as in \mysec{diffracy} is feasible with no regularizer or with a quadratic regularizer. However, if one needs to add more complex regularizers, we need a different relaxation. We start from the penalized version of \eq{diffracyv}, 
\BEQ
\label{eq:diffracyvpen}
\min_{ y \in \{-1,1\}^n, \ v \in \rb^d}   \frac{1}{n}  \| \Pi_n  y - X v \|_2^2+ \nu \frac{(y^\top 1_n)^2}{n^2},
\EEQ
which we expand as:
\BEQ
\label{eq:noncvx}
\min_{ y \in \{-1,1\}^n, \ v \in \rb^d}    \frac{1}{n}   \tr \Pi_n  yy^\top 
-    \frac{2}{n} \tr X v y^\top +    \frac{1}{n} \tr X^\top X v v^\top
+ \nu \frac{(y^\top 1_n)^2}{n^2},
\EEQ
and relax as, using $Y=y y\t $, $P=y v\t$ and $V=v v \t$,
\BEQ
\label{eq:diffracfull}
\min_{ V, P, Y }    \frac{1}{n} \tr \Pi_n  Y -    \frac{2}{n} \tr P^\top X +    \frac{1}{n} \tr X^\top X V + \nu \frac{1_n^\top Y 1_n }{n^2}
\mbox{ s.t.} \ \bigg( \!\! \begin{array}{cc} Y \!\! &\!\! P \\ P^\top\!\! &\!\!  V \end{array} \!\!\bigg) \succcurlyeq 0, \ 
\diag(Y)=1.
\EEQ

  When optimizing \eq{diffracfull} with respect to $V$ and $P$, we get exactly \eq{diffracnu}. 
  Indeed, the optimum is attained for   $ V =   (X^\top X)^{-1} X^\top Y X  (X^\top X)^{-1} $ and $P =  Y X  (X^\top X)^{-1}$ as shown in Appendix~\ref{sec:relaxeq1}. 
 Therefore, the convex relaxation in \eq{diffracfull} is equivalent to \eq{diffracnu}.

However, we get an interesting behavior when optimizing \eq{diffracfull} with respect to $P$ and $Y$ also in closed form. 
For $\nu = 1$, we obtain, as shown in Appendix~\ref{sec:relaxeq2}, the following closed form expressions:  
\BEAS
Y & = & \Diag(\diag(X V X ^\top))^{-1/2} X V X ^\top \Diag(\diag(X V X ^\top))^{-1/2} \\
P & = & \Diag(\diag(X V X ^\top))^{-1/2} X V ,
\EEAS
 leading to the problem: 
\BEQ
\label{eq:relaxW2}
\min_{V \succcurlyeq 0} \ \ 
1 -    \frac{2}{n} \sum_{i=1}^n \sqrt{ ( XVX^\top)_{ii} }  +    \frac{1}{n} \tr ( V X^\top X).
\EEQ
The formulation above in \eq{relaxW2} is interesting for several reasons: (a) it is formulated as an optimization problem in $V \in \rb^{d \times d}$, which will lead to algorithms whose running time will depend on $n$ linearly (see \mysec{algo}), (b) it allows for easy adding of regularizers (see \mysec{sparse}), which may be formulated as convex functions of $V = vv^\top$. However, note that this is valid only for $\nu=1$. 
We now show how to reformulate any problems with $\nu\in [0,1)$ through a simple data augmentation.

\paragraph{Reformulation for any $\nu$.}
When $\nu \in [0,1)$,  we may reformulate the objective function in \eq{diffracyvpen} as follows:
  \BEA   
 \frac{1}{n} \|  \Pi_n y-X v\|_2^2+\nu\frac{( y\t 1_n )^{2}}{n^{2}}
 & = &
  \frac{1}{n} \|  \Pi_n y-X v+\nu\frac{ y\t 1_n}{n}1_n\|_2^2-\big(\nu\frac{ y\t 1_n}{n}\big)^2+\nu\big(\frac{ y\t 1_n }{n}\big)^2 \nonumber \\
 & = &
  \frac{1}{n} \|  y-X v-(1-\nu)\frac{ y\t 1_n}{n}1_n\|_2^2+ \frac{\nu}{ 1- \nu} \big((1-\nu)\frac{ y\t 1_n }{n}\big)^2 \nonumber \\
  &=& \min_{b \in \rb}   
  \frac{1}{n} \| y - X v -  b 1_n \|_2^2 + \frac{\nu}{ 1- \nu} b^2,
  \label{eq:nu_zeroone_reformln}  
  \EEA
since $\frac{1}{n} \| y - X v -  b 1_n \|_2^2 + \frac{\nu}{ 1- \nu} b^2$  can be optimized in closed form with respect to $b$ as $b=(1-\nu)\frac{ y\t1_n }{n}$. Note that the weighted imbalance ratio $(1-\nu)\frac{ y\t 1_n }{n}$ is made as an optimization variable in~\eq{nu_zeroone_reformln}. Thus we have the following reformulation 
\BEA \label{eq:nudiff}
 && \min_{ v \in \rb^d, \ y \in \{-1,1\}^n }  \frac{1}{n} \|  \Pi_n y-X v\|_2^2+\nu\frac{( y\t 1_n )^{2}}{n^{2}} \nonumber \\  
  & = & \min_{ v \in \rb^{d},b\in \RR,  \ y \in \{-1,1\}^n } 
   \frac{1}{n} \| y - X v -  b 1_n \|_2^2 + \frac{\nu}{ 1- \nu} b^2,
\EEA
which is a non-centered penalized formulation on a higher-dimensional problem in the variable  
$\bigl (\begin{smallmatrix}
v \\ b
\end{smallmatrix}\bigr)
\in \rb^{d+1}$.
In the rest of the paper, we will focus on the case $\nu=1$ as it is simpler to present, noticing that by adding a constant term and a quadratic regularizer, we may treat the problem with equal ease when $\nu \in [0,1)$. This enables the use  of the formulation in  \eq{relaxW2}, which is easier to optimize.

\section{Regularization }
\label{sec:sparse}

There are several natural possibilities. We consider norms $\Omega$ such that $\Omega(w)^2 = \Gamma(ww^\top)$ for a certain convex function $\Gamma$; all norms have that form \citep[Proposition 5.1]{fot}. 
When $\nu=1$, \eq{relaxW2}  then becomes 
\BEQ
\label{eq:relaxYY}
\max_{V \succcurlyeq 0}    \frac{2}{n} \sum_{i=1}^n \sqrt{ ( XVX^\top)_{ii} }  -    \frac{1}{n} \tr ( V X^\top X) - \Gamma(V).
\EEQ
The quadratic regularizers $\Gamma(V) =   \tr  \Lambda V$ have already been tackled by \citet{bach_diffrac_2007}. They  consider the regularized version of  problem in \eq{diffracequiv}
\BEQ \label{eq:relaxquad}
\min\limits_{ v \in \rb^d }  \frac{1}{n} \| \Pi_n  y - X v \|_2^2 +  v\t \Lambda v,
\EEQ
  optimize in closed form with respect to $v$ as $v = (X^\top X +n \Lambda)^{-1} X^\top y$. Substituting $v$ in \eq{relaxquad} leads them to 
$$
 \min_{ Y \succcurlyeq 0, \ \diag(Y)=1 }   \frac{1}{n}\tr Y \big( \Pi_n  -   X ( X^\top X + n \Lambda )^{-1} X \big).
$$
In this paper, we formulate a novel sparse regularizer, which is a combination of weighted squared $\ell_1$-norm and $\ell_2$-norm. It leads to 
\[
\Gamma(V) =   \tr [ \Diag(a) V \Diag(a) ] +   \| \Diag(c)  V\Diag(c)  \|_1,
\] 
such that $\Gamma(vv^\top) = \sum_{i=1}^d a_i^2 v_i^2 + \big( \sum_{i=1}^d c_i |v_i| \big)^2$. This allows to treat all situations simultaneously, with $\nu=1$ or with $\nu\in[0,1)$. To be more precise, when $\nu\in[0,1)$, we can consider in \eq{nudiff}, a problem of size $d+1$ with a design matrix $[X, 1_{n}]\in \RR^{n\times (d+1)}$, a direction of projection $\bigl (\begin{smallmatrix}
v \\ b
\end{smallmatrix}\bigr)
\in \rb^{d+1}$ and different weights for the last variable with $a_{d+1}=\frac{\nu}{1-\nu}$ and $c_{d+1}=0$. 

 Note that the sparse regularizers on $V$ introduced in this paper are significantly different when compared to the sparse regularizers on variable $v$ in \eq{diffracequiv}, for example, considered by \citet{multiview_clustering}. A straightforward sparse regularizer on $v$ in \eq{diffracequiv}, despite  leading to a sparse projection, does not yield natural generalizations of the discriminative clustering framework in terms of theory or algorithms. However the sparse regularizers considered in this paper, in addition to their algorithmic appeal for certain applications, also lead to robust cluster recovery under minor assumptions, as will be illustrated on a simple example in Section~\ref{sec:theory}.

\section{Extension to Multiple Labels}
\label{sec:multilabel}

The discussion so far has focussed on two clusters. 
Yet it is key in practice to tackle more clusters. It is worth noting that the discrete formulations in \eq{maxcorr} and \eq{diffracyv} extend directly to more than two clusters. However two different extensions of  the initial problems \eq{maxcorr} or \eq{diffracyv} are conceivable. They lead to problems with different constraints on different optimization domains and, consequently, to different relaxations. We discuss these possibilities next.

One extension is the \emph{multi-class} case. The  multi-class problem which is dealt with by \citet{bach_diffrac_2007} assumes that the data are clustered into $K$ classes and the various partitions of the data points into clusters are represented by the $K$-class indicator matrices $y\in\{0,1\}^{n\times K}$ such that $y 1_K=1_n$. The constraint $y 1_K=1_n$ ensures that one data point belongs to only one cluster. However as discussed by \citet{bach_diffrac_2007}, by letting $Y=yy\t$, it is possible to lift these $K$-class indicator matrices into the outer convex approximations $\mathcal{C}_K=\{Y\in\RR^{n\times n}: Y=Y\t, \diag(Y)=1_n, Y\succcurlyeq0, Y \preccurlyeq \frac{1}{K}1_n1_n\t\}$ \citep{frieze}, which is different for all values of $K$. Note that letting $K=2$ corresponds to the previous sections.

We now discuss the other possible extension, which is the  \emph{multi-label} case. The multi-label problem assumes that the data share $k$ labels and the data-label membership is represented by matrices $y\in\{-1, +1\}^{n\times k}$. In other words, the multi-class problem embeds the data in the extreme points of a simplex, while the multi-label problem does so in the extreme points of the hypercube.

The discriminative clustering formulation of the multi-label problem is
\BEQ
\label{eq:diffracmultilabel}
\min_{ v \in \rb^{d\times k}, \ y \in \{-1,1\}^{n\times k} } \frac{1}{n}\|\Pi_{n}y - X v \|_F^2, 
\EEQ
where the Frobenius norm is defined for any vector or rectangular matrix as $\Vert A \Vert_{F}^{2}= \tr AA\t =\tr A\t A$. Letting $k=1$ here corresponds to the previous sections. The discrete ensemble of matrices $y\in\{-1, +1\}^{n\times k}$ can be naturally lifted into $\mathcal{D}_k=\{Y\in\RR^{n\times n}: Y=Y\t, \diag(Y)=k 1_n, Y\succcurlyeq0\}$, since $\diag(Y)=\diag( yy\t)=\sum_{i=1}^{k}y_{i,i}^{2}=k$.  
As the optimization problems in \eq{diffracalpha} and \eq{diffracnu} have linear objective functions, we can change the variable from $Y$ to $\tilde Y=Y/k$ to change the constraint $\diag(Y)=k 1_n$ to $\diag(\tilde Y)= 1_n$  without changing the optimizer of the problem. Thus the problems can be solved over the relaxed domain $\mathcal{D}=\{Y\in\RR^{n\times n}: Y=Y\t, \diag(Y)= 1_n, Y\succcurlyeq0\}$ which is independent of~$k$. 

Note that the domain $\mathcal{D}$ is similar to that considered in the problems in \eq{diffracnu} and \eq{diffracfull} and these convex relaxations are the same regardless of the value of $k$. Hence the multi-label problem is a more natural extension of the discriminative framework, with a slight change in how the labels $y$ are recovered from the solution $Y$ (we discuss this in Section~\ref{subsec:multilabel_analysis}).

\section{Theoretical Analysis}
\label{sec:theory}

In this section, we provide a theoretical analysis for the discriminative clustering framework. We start with the 2-clusters situation: the non-sparse case is considered first and analysis is provided for both balanced and imbalanced clusters. Our study for the sparse case currently only provides  results for the  simple $1$-sparse solution. However, the analysis also yields valuable insights on the scaling between $n$ and $d$. We then derive results for multi-label situation. 

For ease of analysis, we consider the constrained problem in \eq{diffracalpha}, the penalized problem in \eq{diffracnu} or their equivalent relaxations in \eq{relaxW2} or \eq{relaxYY} under various scenarios, for which we use the same proof technique. We first try to characterize the low-rank solutions of these relaxations and then show in certain simple situations the uniqueness of such solutions, which are then non-ambiguously found by convex optimization. Perturbation arguments could extend these results by weakening our assumptions but are not within the scope of this paper, and hence we do not investigate them further in this section.   

\subsection{Analysis for $2$ clusters: non-sparse problems}\label{sec:theoanalnonsparse}

In this section, we consider several noise models for the problem, either adding irrelevant dimensions or perturbing the label vector with noise.
We consider these separately for simplicity, but they could also be combined (with little extra insight).

\subsubsection{Irrelevant dimensions}

We consider an ``ideal'' design matrix $X\in\RR^{n\times d}$ such that there exists a direction $v$ along which the projection $Xv$ is perfectly clustered into two distinct real values $c_1$ and $c_2$.
 Since \eq{maxcorr} is invariant by affine transformation, we can rotate the design matrix $X$ to have  $X=[ y,Z]$ with $y\in\{-1,1\}^{n}$, which is clustered into $+1$ or $-1$ along the direction $v=\bigl (\begin{smallmatrix}
1 \\ 0_{d-1}
\end{smallmatrix}\bigr)$. Then after being centered, the design matrix is written as $X=[ \Pi_n y,Z]$ with $Z=[z_1,\dots,z_{d-1}]\in\RR^{n \times (d-1)}$. The columns of $Z$ represent the noisy irrelevant dimensions added on top of the signal $y$.

\subsubsection{Balanced problem}

When the problem is well balanced ($y\t 1_{n}=0$), $y$ is already centered and $\Pi_n y=y$. Thus the design matrix is represented as $X=[y,Z]$.
We consider here the penalized formulation in \eq{diffracnu} with $\nu=1$ which is easier to analyze in this setting.

Let us assume that the columns $(z_i)_{i=1,\dots,d-1}$ of $Z$ are  i.i.d.~with symmetric distribution~$z$, with $\E z=\E z^3=0$ and such that  $\Vert z\Vert_{\infty}$ is almost surely bounded by $R\geq0$. We denote by  $\E z^2=m$ its second moment and by   $\E z^4/(\E z^2)^2=\beta$ its (unnormalized) kurtosis.

Surprisingly the clustered  vector $y$  happens to generate a solution $yy\t$ of the relaxation \eq{diffracnu} for all possible values of $Z$  (see Lemma \ref{lemma:ysol} in Appendix~\ref{sec:roy} ). However the problem in \eq{diffracnu} should have a \emph{unique} solution in order to always recover the correct assignment $y$. Unfortunately the semidefinite constraint $Y \succcurlyeq0$  of the relaxation  makes the second-order information arduous to study. 
Due to this reason, we consider the other equivalent relaxation in \eq{relaxW2} for which $V_*=vv\t$ is also solution with $v \propto (X^\top X)^{-1}X^\top y$ (see Lemma \ref{lemma:wsol} in Appendix~\ref{sec:row}). 
Fortunately the semidefinite constraint $V \succcurlyeq0$ of the problem in \eq{relaxW2}  may be ignored since the second-order information in $V$ of the objective function  already provides unicity for the unconstrained  problem. 
Hence 
we are able to ensure the uniqueness of the solution with high probability and the following result provides the first guarantee for discriminative clustering. 
\begin{proposition}\label{prop:unicitew}
Let us assume $d\geq3$, $\beta>1$ and $m^2\geq \frac{\beta-3}{2(d+\beta -4)}$:
\\
(a)
If $n\geq d^2R^4\frac{1+(d+\beta)m^2}{m^2(\beta-1)}$,  $V_*$ is the unique solution of the problem in \eq{relaxW2} with high probability.
\\
(b)
If $n \geq \frac{d^2R^4}{\min\{ m^2(\beta-1),2m^2, 2m\}}$, $v$ is the principal eigenvector  of any solution of the problem in \eq{relaxW2} with high probability.
\end{proposition}

Let us make the following observations:
\begin{itemize}
 \item \textbf{Proof technique}:
 The proof relies on a computation of the Hessian of  $f(V)=\frac{2}{n}\sum_{i=1}^n\sqrt{( XVX^\top)_{ii}}-\frac{1}{n}\tr X\t X V$ which is the objective function in \eq{relaxW2}. We first derive the expectation of $\nabla^2 f(V)$ with respect to the distribution of $X$. By the law of large number, it amounts to have $n$ going to infinity in $\nabla^2f(V)$. Then we expand the spectrum of this operator $\E \nabla^2f(V)$  to lower-bound its smallest eigenvalue. Finally we use concentration theory on matrices, following \citet{tropp},  to bound the Hessian  $\nabla^2 f(V)$  for finite $n$.  
  
 \item  \textbf{Effect of kurtosis}: We remind that $\beta \geqslant 1$, with equality if and only if $z$ follows a Rademacher law ($\mathbb{P}(z=+1)=\mathbb{P}(z=-1)=1/2$). Thus, if the noisy dimensions are clustered, then unsurprisingly, our guarantee is meaningless. Note that the constant $\beta$ behaves like a distance of the distribution $z$ to the Rademacher distribution. Moreover,  $\beta=3$ if $z$ follows a standard normal distribution.

 \item \textbf{Scaling between $d$ and $n$}: If the noisy variables are not evenly clustered between the same clusters $\{\pm1\}$ (i.e., $\kappa>1$), we recover a rank-one solution as long as $n=O(d^3)$; while, as long as $n=O(d^2)$, the solution is not unique but its principal eigenvector recovers the correct clustering. Moreover, as explained in the proof, its spectrum would be very spiky. 

  \item The assumption $m^2\geq \frac{\beta-3}{2(d+\beta -4)}$ is generally satisfied for large dimensions. Note that $m^2 d$ is the total variance of the irrelevant dimensions, and when it is small, i.e.,  when   $m^2\leq \frac{\beta-3}{2(d+\beta -4)}$, the problem is particularly simple, and we can also show that $V_*$ is the unique solution of the problem in \eq{relaxW2} with high probability if   $n\geq \frac{d^2R^4}{m^2}$.
  Finally, note that for sub-Gaussian distributions (where $\kappa \leq 3$), the extra constraint is vacuous, while for super-Gaussian distributions (where $\kappa \geq 3$), this extra constraint only appears for small $m$. 
\end{itemize}

 \subsubsection{Noise robustness for the $1$-dimensional balanced problem}

We assume now that the data are one-dimensional and  
are perturbed by some noise $\varepsilon \in \RR^n$ such that $X=y+\varepsilon$ with $y\in\{-1,1\}^n$. 
The solution of the relaxation in \eq{diffracnu} recovers the correct $y$ in this setting only when each  component of $y$ and $y+\varepsilon$ have the same sign (this is shown in Appendix~\ref{sec:noisey}). This result comes out naturally from the information on whether the signs of $y$ and $y+\varepsilon$ are the same or not. Further if we assume that $y$ and $\varepsilon$ are independent, this condition is equivalent to $\Vert\varepsilon\Vert_{\infty}< 1$ almost surely.

 \subsubsection{Unbalanced problem}
 
When the clusters are imbalanced ($y\t 1_{n}\neq 0$), the natural rank-one candidates  $Y_*=yy\t$ and $V_*=vv\t$  are no longer solutions of the relaxations in \eq{diffracnu} (for $\nu=1$) and \eq{relaxW2}, as proved in Appendix~\ref{sec:nomore}. Nevertheless we are able to characterize some solutions of the penalized relaxation  in \eq{diffracnu}  for $\nu=0$.
\begin{lemma}\label{lemma:solrang2c}
 For $\nu=0$ and  for any  non-negative  $a,b \in \rb$ such that $a+b=1$,
 $$
 Y=a yy\t+b 1_{n}1_{n}\t
 $$
is  solution of the penalized relaxation in \eq{diffracnu}.
\end{lemma}
Hence any eigenvector of this solution $Y$ would be supported by the directions $y$ and $1_n$.  Moreover when the value $\alpha_*=(\frac{1_{n}\t y}{n})^2$ is known, it turns out that we can  characterize some solution of the constrained relaxation in \eq{diffracalpha}, as stated in the following lemma. 
\begin{lemma}\label{lemma:solrang2}
 For $\alpha\geq\alpha_*$, 
 $$
 Y=\frac{1- \alpha}{1-\alpha_*}yy\t+\Big(1-\frac{1- \alpha}{1-\alpha_*}\Big)1_{n}1_{n}\t
 $$
 is a rank-2 solution of the constrained relaxation in \eq{diffracalpha} with constraint parameter $\alpha$.
\end{lemma}
The eigenvectors of $Y$  enable to recover $y$ for $\alpha_*\leq\alpha<1$. We conjecture (and checked empirically) that this rank-2 solution is unique under  similar regimes to those considered for the balanced case. 
The proof would be more involved since, when $\nu\neq1$, we are not able to derive an equivalent problem in $V$ for  the penalized relaxation in \eq{diffracnu}  similar to \eq{relaxW2}  for the balanced case. 

Thus $Y$ being rank-2, one should really be careful and consider the first two eigenvectors when recovering $y$ from a solution~$Y$. This can be done by rounding the principal eigenvector of $\Pi_n Y \Pi_n= \frac{1- \alpha}{1-\alpha_*} \Pi_n y (\Pi_n y)\t$ as discussed in the following lemma. 
\begin{lemma}
 Let $y_{ev}$ be the principal eigenvector of $\Pi_n Y \Pi_n$ where $Y$ is defined in Lemma \ref{lemma:solrang2}, then
 $$
 \sign(y_{ev})=y.
 $$
\end{lemma}

\begin{proof}
 By definition of $Y$, $y_{ev}=\sqrt{\frac{1- \alpha}{1-\alpha_*}}\Pi_n y$ thus $\sign(y_{ev})=\sign(\Pi_n y)$ and since $\alpha\leq1$ then $\sign(\Pi_n y)=\sign(y-\sqrt{\alpha}1_n)=y$.
\end{proof}
In practice, contrary to the standard procedure, we should, for any $\nu$, solve the penalized relaxation in \eq{diffracnu}  and then do  $K$-means on the principal eigenvector of the centered solution  $\Pi_n Y \Pi_n$ instead of the solution $Y$ to recover the correct $y$. This procedure is followed in our experiments on real-world data in Section~\ref{sec:realworld}.

\subsection{Analysis for $2$ clusters: $1$-sparse problems}
\label{sec:theory-sparse}

We assume here that the direction of projection $v$ (such that $Xv=y$) is $l$-sparse (by $l$-sparse we mean $\|v\|_0 = l$). The $\ell_1$-norm  regularized problem  in \eq{relaxYY} is no longer invariant by affine transformation and we cannot consider that $X=[y,Z]$ without loss of generality. Yet the relaxation \eq{relaxYY} seems experimentally to only have   rank-one solutions for the simple $l=1$ situation. Hence we are able to derive some theoretical analysis only for this case. It is worth noting the  $l=1$ case is simple since it can be solved in $O(d)$ by using $K$-means separately on all dimensions and ranking them.  Nonetheless the proposed scaling also holds in practice for $l\geqslant 1$ (see \myfig{phase2}). 

Thereby we consider data $X=[y,Z]$ with $y\in\{-1,1\}^n$ and $Z\in\RR^{n\times (d-1)}$ which are clustered in the direction $v=[1,\underbrace{0,\ldots,0}_{d-1 \text{ terms}}]\t$.
When adding a $\ell_1$-penalty, the initial problem in \eq{diffracyv} for $\nu=1$ is 
$$
\underset{y\in\{-1,1\}^n,\ v\in\RR^d}{\min}\ \frac{1}{n}\Vert y-Xv\Vert_2^2+\lambda \Vert v\Vert_1^2.
$$
When optimizing in $v$ this problem  is close to the Lasso \citep{lasso} and a solution is known to be
$
v^*_i=(y\t y + n\lambda  )^{-1}y\t y=\frac{1}{1+\lambda}, \ \forall i\in J \text{ and } v^*_i=0 , \ \forall i\in{\{1,2,\ldots,d\}\setminus{J}},
$
where $J$ is the support of $v^*$. 
The candidate $V_*=v^*{v^*}\t$ is still a solution of the relaxation in \eq{relaxYY} (see  Lemma \ref{lemma:wsols} in Appendix~\ref{app:solws}) and we will investigate under which conditions on $X$ the solution is unique.
Let us assume as before $(z_i)_{i=1,\dots,d}$ are i.i.d.~with distribution $z$ symmetric with $\E z=\E z^3=0$, and denote by $\E z^2=m$ and $\E z^4/(\E z^2)^2=\beta$. We also assume that  $\Vert z\Vert_{\infty}$ is almost surely bounded by $0\leq R\leq 1$.
We are able to ensure the uniqueness of the solution with high-probability. 

\begin{proposition}\label{prop:sparse}
Let us assume $d\geq3$. 
\\
(a)
If $n\geq dR^2\frac{1+(d+\beta)m^2}{m^2(\beta-1)}$,  $V_*$ is the unique solution of the problem \eq{relaxW2} with high probability.
\\
(b)
If $n \geq \frac{d R^2}{ m^2(\beta-1)}$, $v^*$ is the principal eigenvector  of any solution of the problem \eq{relaxW2} with high probability.
\end{proposition}

The proof technique is very similar to the one of Proposition \ref{prop:unicitew}.  With the function $g(V)=\frac{2}{n}\sum_{i=1}^n\sqrt{( XVX^\top)_{ii}}- \lambda \Vert V\Vert_1 -\frac{1}{n}\tr X\t X V$, we can certify that $g$ will decrease around the solution $V_*$ by analyzing the eigenvalues of its Hessian.  

The rank-one solution $V_*$ is recovered by the principal eigenvector of the solution of the relaxation \eq{relaxYY} as long as $n=O(d)$. Thus we have a much better scaling when compared to the non-sparse setting where $n = O(d^2)$. We also conjecture a scaling of order $n=O(l d)$ for a projection in a $l$-sparse direction (see \myfig{phase2} for empirical results).

The proposition does not state any particular value for the regularizer parameter $\lambda$. This makes sense since the proposition only holds for the simple situation when $l=1$. We propose to use $\lambda=1/\sqrt{n}$ by analogy with the Lasso.

\subsection{Analysis for the multi-label extension} 
\label{subsec:multilabel_analysis}

In this section, the signals share $k$ labels which are corrupted by some extra noisy dimensions. We assume the centered design matrix to be  $X=[\Pi_n y, Z]$ where $y\in\{-1,+1\}^{n\times k }$ and $Z\in \RR ^{n\times(d-k)}$. We also assume that $y$ is full-rank\footnote{This assumption is fairly reasonable since the probability of a matrix whose entries are i.i.d.~Rademacher random variables  to be singular is conjectured to be $1/2+o(1)$ \citep{bourgain}.}. 
We denote by $y=[y_1,\dots, y_k]$ and $\alpha_i=\Big(\frac{y_i\t 1_n}{n}\Big)^2$ for ${i=1,\cdots,k}$.
We consider the discrete constrained problem
\BEQ
\label{eq:diffracmultilabelcons}
\min_{ v \in \rb^{d\times k}, \ y \in \{-1,1\}^{n\times k} } \frac{1}{n}\|\Pi_{n}y - X v \|_F^2 \mbox{ such that } \frac{1_n\t y y\t 1_n}{n^2}=\alpha^2, 
\EEQ
and the discrete penalized problem for $\nu=0$
\BEQ
\label{eq:diffracmultilabelpen}
\min_{ v \in \rb^{d\times k}, \ y \in \{-1,1\}^{n\times k} } \frac{1}{n}\|\Pi_{n}y - X v \|_F^2. 
\EEQ

As explained in Section~\ref{sec:multilabel}, these two discrete problems admit the same relaxations in \eq{diffracalpha} and \eq{diffracnu} we have studied for one label. 
We now investigate when the solution of the problems  in \eq{diffracmultilabelcons} and in \eq{diffracmultilabelpen} generate solutions of the relaxations in \eq{diffracalpha} and \eq{diffracnu}. 

By analogy with Lemma \ref{lemma:solrang2c},  we  want to characterize the solutions of these relaxations  which are supported by the constant vector $1_n$ and the labels $(y_1,\dots,y_k)$.  Their general form is $Y=\tilde y A \tilde y\t$ where $A\in \RR^{k\times k}$ is symmetric semi-definite positive and $\tilde y=[1_n, y]$. However the initial $y$ is easily recovered from the solution $Y$ only when $A$ is diagonal. To that end  the following lemma derives some condition under which the only matrix $A$ such that the corresponding $Y$ satisfies the constraint of the relaxations  in \eq{diffracalpha} and \eq{diffracnu} is diagonal. 
\begin{lemma}\label{lemma:diagdiag}
The solutions of the matrix equation  $\diag (\tilde y A \tilde y\t)=1_n$ with unknown variable $A$ are diagonal if and only if the family $\{1_n,(y_i)_{1\leq i\leq k}, (y_i\odot y_j)_{1\leq i<j\leq k}\}$ is linearly independent where we denoted by $\odot$ the Hadamard (i.e., pointwise) product between matrices.
\end{lemma}

In this way we are able to characterize the solution of relaxations in \eq{diffracalpha}  and \eq{diffracnu} with the following result:

\begin{lemma}\label{lemma:multilabel}
 Let us assume that the family $\{1_n,(y_i)_{1\leq i\leq k}, (y_i\odot y_j)_{1\leq i<j\leq k}\}$ is linearly  independent. If $\alpha\geq \alpha_{\min}=\underset{1\leq i \leq k}{ \min} \{\alpha_{i}\}$ with $(\alpha_i)_{1\leq i\leq k}$ defined above \eq{diffracmultilabelcons}, the solutions of the constrained relaxation in \eq{diffracalpha} supported by the vectors  $(1_n,y_1,\cdots,y_k)$ are of the form:
 $$
 Y= a_0^2 1_n1_n\t +\sum_{i=1}^ka_i^2 y_iy_i\t,
 $$
 where $(a_i)_{0\leq i\leq k}$ satisfies $\sum_{i=0}^k a_i^2=1$ and $a_0^2+\sum_{i=1}^k a_i^2 \alpha_i= \alpha$.
 
Moreover the solutions of the penalized  relaxation in \eq{diffracnu} for $\nu=0$ which are supported by the vectors  $(1_n,y_1,\cdots,y_k)$ are of the forms:
  $$
 Y= a_0^2 1_n1_n\t +\sum_{i=1}^ka_i^2 y_iy_i\t,
 $$
 where $(a_i)_{0\leq i\leq k}$ satisfies $\sum_{i=0}^k a_i^2=1$. 
 \end{lemma}

In the \emph{multi-label} case, some combinations of the constant matrix  $1_n1_n\t$ and the rank-one matrices $y_iy_i\t$ are  solutions of constrained or penalized relaxations. Furthermore,  under some assumptions on the labels  $(y_i)_{1\leq i\leq k}$,  these combinations are the only solutions which are supported by the vectors  $(1_n,y_1,\cdots,y_k)$.  And we conjecture (and checked empirically) that under assumptions similar to those made for the balanced one-label case, all the solutions of the relaxation are supported by the family  $(1_n,y_1,\cdots,y_k)$ and consequently share the same form as in Lemma \ref{lemma:multilabel}. Thus the eigenvector of the solution $Y$ would be in the span of the directions $(1_n,y_1,\cdots,y_k)$. 

Let us  consider an eigenvalue decomposition of $Y=FF\t=\sum_{i=0}^k\lambda_i e_i e_i \t $ and denote by $M=[a_0 1_n, a_1 y_1,\cdots, a_k y_k]$ where $(a_i)_{0\leq i\leq k}$ are defined in Lemma \ref{lemma:multilabel}. Since $M M\t = FF\t$, there is an orthogonal 
transformation $R$ such that $F R=M$. We also denote the product $F R$ by $F R=[\xi_0,\cdots,\xi_K]$. We propose now an alternating minimization procedure to recover the labels $(y_1,\cdots,y_k)$ from $M$.

\begin{lemma}\label{lemma:proc}
 Consider the optimization problem
$$
\min_{M\in\mathcal M , \ R\in \RR^{k\times k}: \ R\t R=I_k} \Vert F R-M\Vert_F^2,
$$
where $\mathcal M = \{ [a_0 1_n, a_1 y_1,\cdots, a_k y_k], a\in \RR^{k+1}: \Vert a \Vert_2=1 , y_{i}\in \{\pm 1\}^{n}\}$.

Given $M$, the problem is equivalent to the orthogonal Procrustes problem  \citep{propro}. Denote by $U\Delta V\t$ a singular value decomposition of $F \t M$. The optimal $R$ is obtained as $R= UV\t$.
While given $R$, the optimal $M$ is obtained as 
\[
M=\frac{1}{\sqrt{\Vert \xi_1\Vert_1^2 + \Vert \xi_2\Vert_1^2 + \ldots +  \Vert \xi_k\Vert_1^2}} [\Vert \xi_0\Vert_1\sign (\xi_0),\cdots,\Vert \xi_k\Vert_1\sign (\xi_k)].                                                 
\]
\end{lemma}
 
\begin{proof}
 We give only the argument for the optimization problem with respect to $M$. Given $R$, the optimization problem in $M$ is equivalent to $\underset{a\in \RR^{k+1}: \Vert a \Vert_2=1 , \  y\in\{-1,1\}^{n\times k}}{\max} \tr (FR)\t M$ and $ \tr (FR)\t M= a_0 \xi_0\t 1_{n}+\sum_{i=1}^k a_i \xi_i\t  y_i$ . Thus by property of the dual norms the solution is given by $y_i=\sign(\xi_i)$ and $a_i= \frac{\Vert \xi_i\Vert_1}{ \sqrt{\Vert \xi_1\Vert_1^2 + \Vert \xi_2\Vert_1^2 + \ldots +  \Vert \xi_k\Vert_1^2}}$.
\end{proof}
The minimization problem in Lemma \ref{lemma:proc} is non-convex; however we observe that performing few alternating optimizations is sufficient to recover the correct $(y_1,\dots,y_k)$ from~$M$.

\subsection{Discussion}
In this section we studied the tightness of convex relaxations under simple scenarios where the relaxed problem admits low-rank solutions generated by the solution of the original non-convex problem.  Unfortunately the solutions lose the characterized rank when the initial problem is slightly perturbed since the rank of a matrix is not a continuous function. Nevertheless,  the spectrum of the new solution is really spiked, and thus these results are quite conservative. We empirically observe that the principal eigenvectors keep recovering the correct information outside these scenarios. However this simple proof mechanism  is not easily adaptable to handle perturbed problems in a straightforward way since it is difficult to characterize the properties of eigenvectors of the solution of a semi-definite program. Hence  we are able to derive a proper theoretical study only for these simple models. 

\section{Algorithms}
\label{sec:algo}

In this section, we present an optimization algorithm which is adapted to large $n$ settings, and avoids the $n$-dimensional semidefinite constraint.

\subsection{Reformulation}

We aim to solve the general regularized problem which correponds to \eq{relaxYY}
\BEQ
\label{eq:relaxY}
\max_{V \succcurlyeq 0} \frac{2}{n} \sum_{i=1}^n \sqrt{ ( XVX^\top)_{ii} }  - \frac{1}{n} \tr   V ( X^\top X
+ n \Diag(a)^2) - \| \Diag(c)  V\Diag(c)  \|_1.
\EEQ	
We consider a slightly different optimization problem:
\BEQ
\label{eq:relaxYref}
\max_{V \succcurlyeq 0} \frac{1}{n} \sum_{i=1}^{n} \sqrt{( XVX^\top)_{ii}} -   \| \Diag(c)  V\Diag(c)  \|_1
\ \ \ {\rm{s.t.}} \   \tr   V ( \frac{1}{n} X^\top X
+   \Diag(a)^2) = 1.
\EEQ
 When $c$ is equal to zero, then \eq{relaxYref} is exactly equivalent to \eq{relaxY};  when  $c$ is small (as will typically be the case in our experiments), the solutions are very similar---in fact, one can show by Lagrangian duality that by a sequence of problems in \eq{relaxYref}, one may obtain the solution to \eq{relaxY}.

\subsection{Smoothing}

By letting $A \!=\! \frac{X^\top X}{n} \!+\! \Diag(a)^2 $, we consider a strongly-convex approximation of \eq{relaxYref} as:
\BEAS
\max_{V \succcurlyeq 0} \frac{1}{n} \sum_{i=1}^{n} \sqrt{( XVX^\top)_{ii}} -   \| \Diag(c)  V\Diag(c)  \|_1  - \varepsilon \tr [(A^\half V A^\half)\log(A^\half V A^\half)]   
\ {\rm{s.t.}} \ \tr (A^\half V A^\half) = 1,
\EEAS
where $- \tr M \log(M)$ is a spectral convex  function called the von-Neumann entropy \citep{vonneumann_entropy}.
The difference in the two problems is known to be $\varepsilon \log(d)$ \citep{Nest_smooth}. As shown in Appendix~\ref{app:dualsmoothing}, the dual problem is
\BEQ
\label{eq:smoothdual}
\min_{u\in\rb^{n}_{+}, C \in \rb^{d \times d}:|C_{ij}| \leqslant c_i c_j}     \frac{1}{2n} \sum_{i=1}^{n} \frac{1}{u_i} + \phi^{\varepsilon}\big(A^{-\half}\big(\frac{1}{2n}X^\top\Diag(u)X - C\big)A^{-\half}\big),
\EEQ
where $\phi^{\varepsilon}(M)$ is an $\varepsilon$-smooth approximation to the maximal eigenvalue of the matrix $M$.

\subsection{Optimization algorithm}

In order to solve \eq{smoothdual}, we split the objective function into a smooth part 
$F(u,C) = \phi^{\varepsilon}\big(A^{-\half}\big(\frac{1}{2n}X^\top\Diag(u)X - C\big)A^{-\half}\big)$ and a non-smooth part $H(u,C) =  \mathbb{I}_{|C_{ij}| \leqslant c_i c_j } + \frac{1}{2n} \sum_{i=1}^{n} \frac{1}{u_i}  $.
We may then apply FISTA \citep{fista} updates to the smooth function $\phi^{\varepsilon}(A^{-\half}(\frac{1}{2n}X^\top\Diag(u)X - C)A^{-\half})$, along with a proximal operator for the
non-smooth terms $\mathbb{I}_{|C_{ij}| \leqslant c_i c_j }$ and $\frac{1}{2n} \sum_{i=1}^{n} \frac{1}{u_i}$, which may be computed efficiently. See details in Appendix \ref{app:algo}.

\paragraph{Running-time complexity.} Since we need to project on the SDP cone of size $d$ at each iteration, the running-time complexity per iteration is $O(d^3+d^2n)$; given that often $n \geqslant d$, the dominating term is $O(d^2n)$.  {It is still an open problem to make this linear in $d$.} Our function being $O(1/\varepsilon)$-smooth, the convergence rate is of the form $O(1/(\varepsilon t^2))$. Since we stop when the duality gap is $\varepsilon  \log(d)$ (as we use smoothing, it is not useful to go lower), the  number of iterations is of order $1 / ( \varepsilon \sqrt{\log (d)})$. 

\section{Experiments}

We implemented the proposed algorithm in Matlab. The code has been made available in~\url{https://drive.google.com/uc?export=download&id=0B5Bx9jrp7celMk5pOFI4UGt0ZEk}. 
Two sets of experiments were performed: one on synthetically generated data sets and the other on real-world data sets. The details about experiments follow. 

\subsection{Experiments on synthetic data}
 
In this section, we illustrate our theoretical results and algorithms on synthetic examples. The synthetic data were generated by assuming a  fixed  clustering with $\alpha_*\in[0,1]$, along a single direction and the remaining variables were whitened. We consider clustering error defined for a predictor $\bar y$ as $1-(\bar y\t y/n)^2$, with values in $[0,1]$ and equal to zero if and only if $y=\bar y$.

\paragraph{Phase transition.}
We first illustrate our theoretical results for the balanced case in \myfig{phasetrans}. We solve the relaxation for a large range of $d$ and $n$ using the \texttt{cvx} solver \citep{ gb08, cvx}.  
We show the results averaged over 4 replications and  take  $\lambda_n=1/{\sqrt{n}}$ for the sparse problems.  
In \myfig{phase} we investigate whether  \texttt{cvx} finds a rank-one solution  for a problem of size $(n,d)$ (the value is $1$ if the solution is rank-one and $0$ otherwise). We compare the performance of the algorithms without  $\ell_1$-regularization in the affine invariant case and with  $\ell_1$-regularization in the 1-sparse case. We observe a phase transition with a scaling over the form $n=O(d^{2})$ for the affine invariant case and  $n=O(d)$ for the $1$-sparse case. This is better than what expected by the theory and corresponds rather to the performance of the principal eigenvector of the solution. It is worth noting  that it may be uncertain to really distinguish between a rank-one solution and a spiked solution.

We also solve the relaxation for  $4$-sparse problems of different sizes $d$ and $n$ and plot the clustering error. We compare the performance of the algorithms without  $\ell_1$-regularization in the affine invariant case and with  $\ell_1$-regularization in the 4-sparse case in \myfig{phase2}. We notice a phase transition of the clustering error with a scaling over the form $n=O(d^{2})$ for the affine invariant case and  $n=O(d)$ for the $4$-sparse case.  It supports our conjecture on the scaling of order $n=O(l d)$ for $l$-sparse problems. Comparing left plots of \myfig{phase} and \myfig{phase2}, we observe that the two phase-transitions occur at the same scaling between $n$ and $d$. Thus there are few  values of $(n,d)$ for which the \texttt{cvx} solver finds a solution whose rank is stricly larger than one and whose principal eigenvector has a low clustering error.  
This illustrates, in practice, this solver aims to find a rank-one solution under the improved scaling $n=O(d^{2})$.

\begin{figure}[!htbp]

\begin{minipage}[c]{1 \linewidth}
\begin{minipage}[c]{.5 \linewidth}
\includegraphics[scale=0.49]{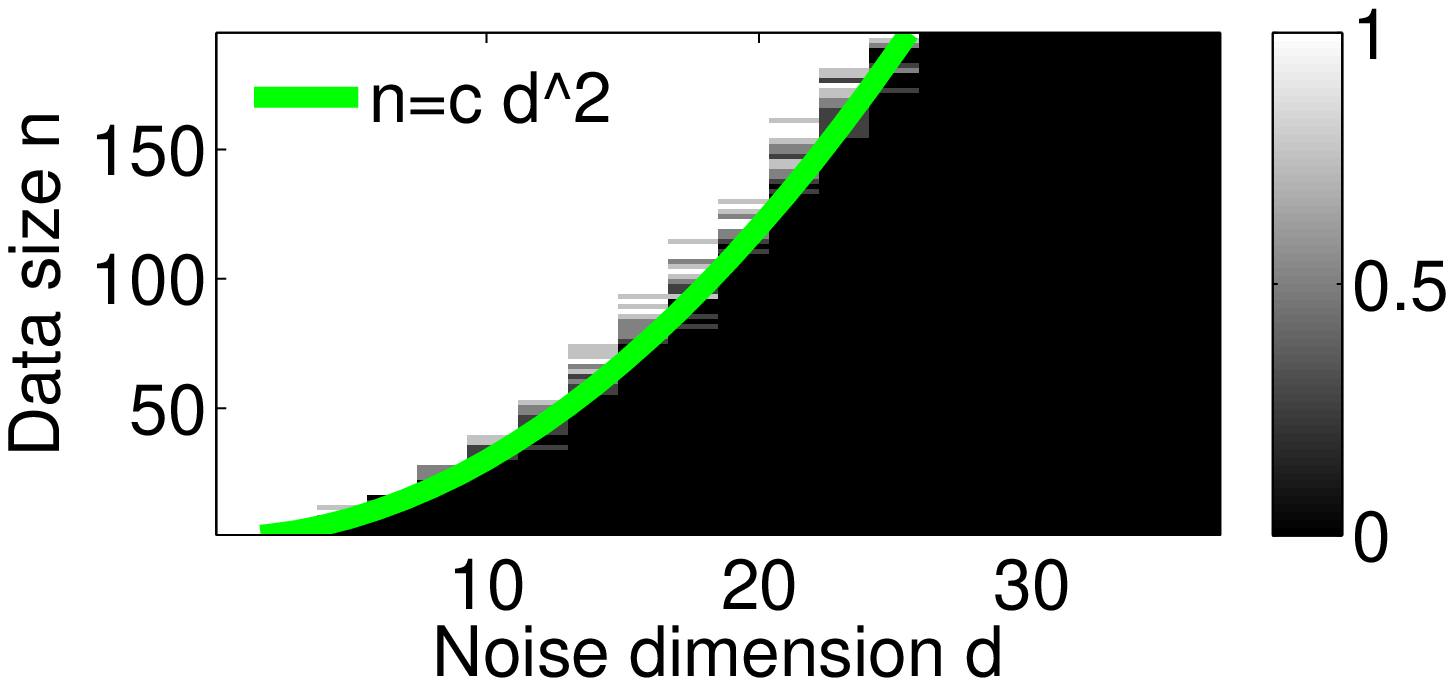}
   \end{minipage} 
   \begin{minipage}[c]{.5 \linewidth}
\includegraphics[scale=0.49]{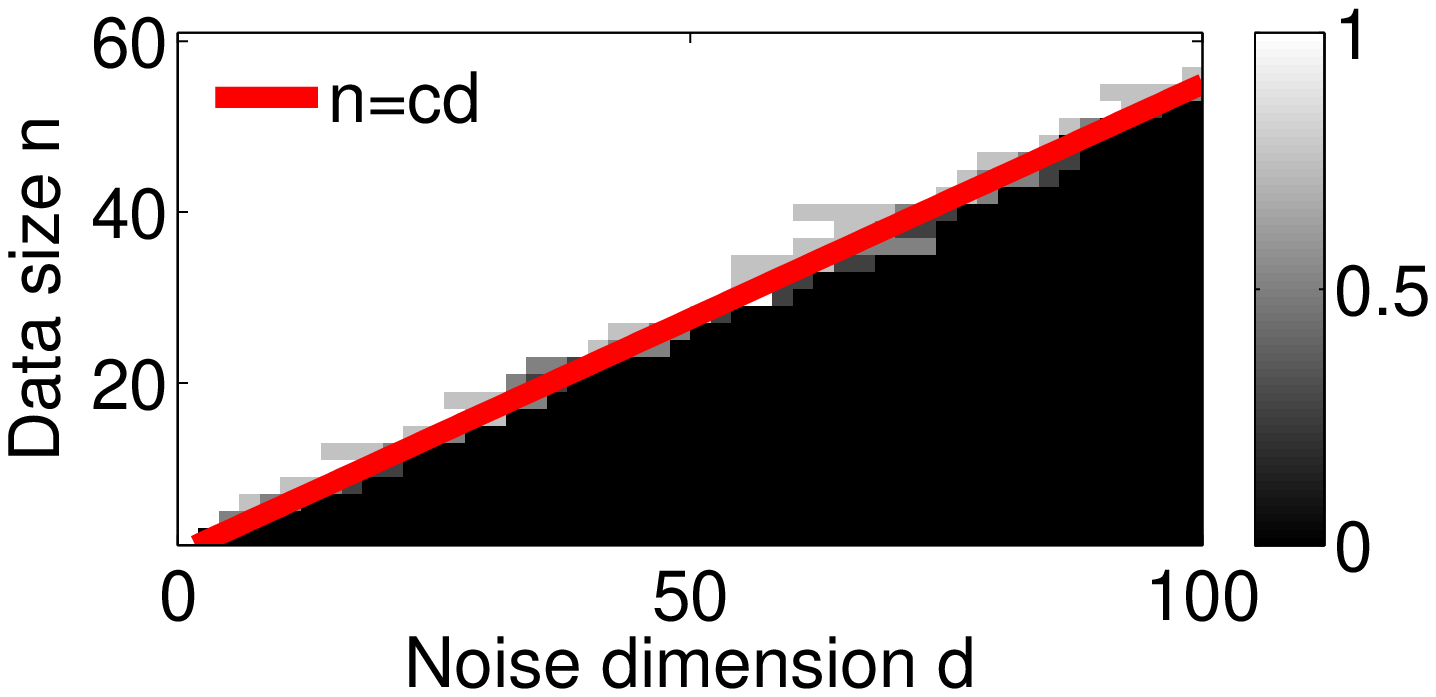}
      \end{minipage}   
		\vspace{-.2in}
    \subcaption{Phase transition for rank-one solution. Left: affine invariant case. Right: 1-sparse case. }
       \label{fig:phase}
    \end{minipage} 

    \begin{minipage}[c]{1 \linewidth}
\begin{minipage}[c]{.5 \linewidth}

\includegraphics[scale=0.49]{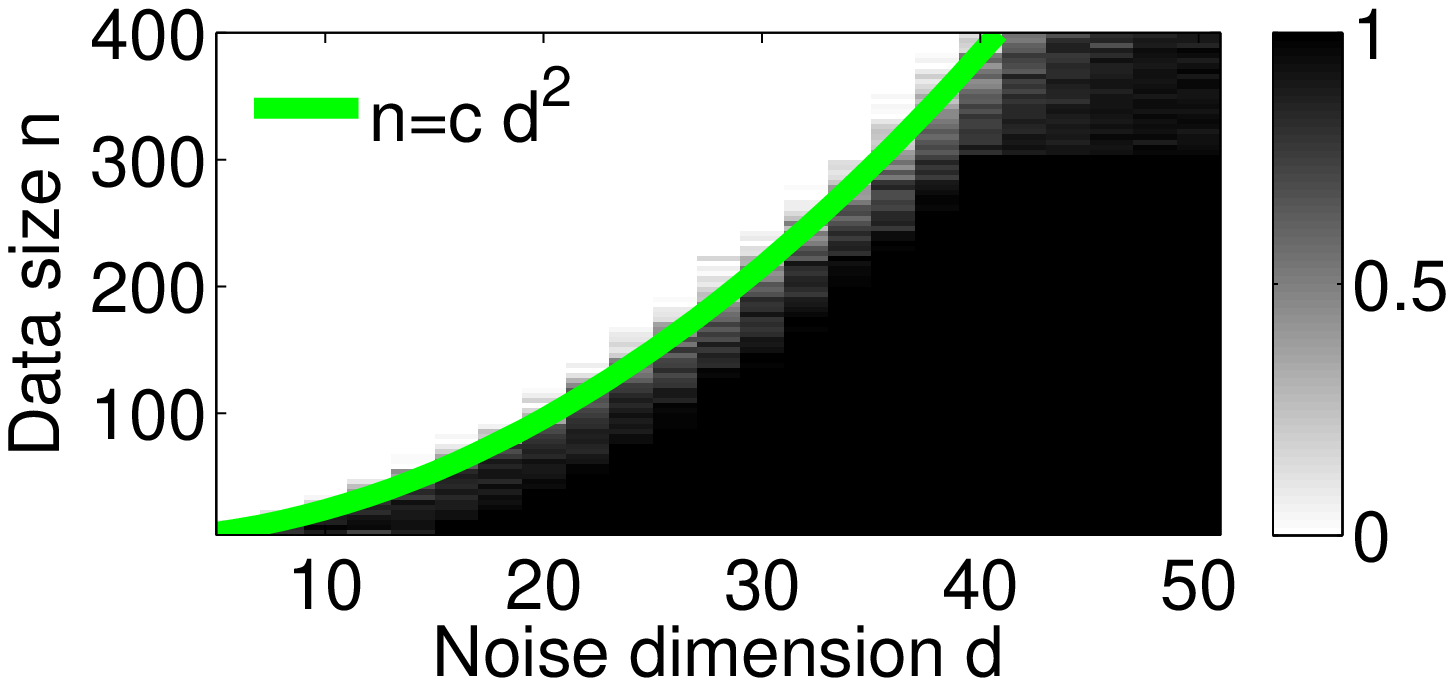}
   \end{minipage} 
   \begin{minipage}[c]{.5 \linewidth}
\includegraphics[scale=0.49]{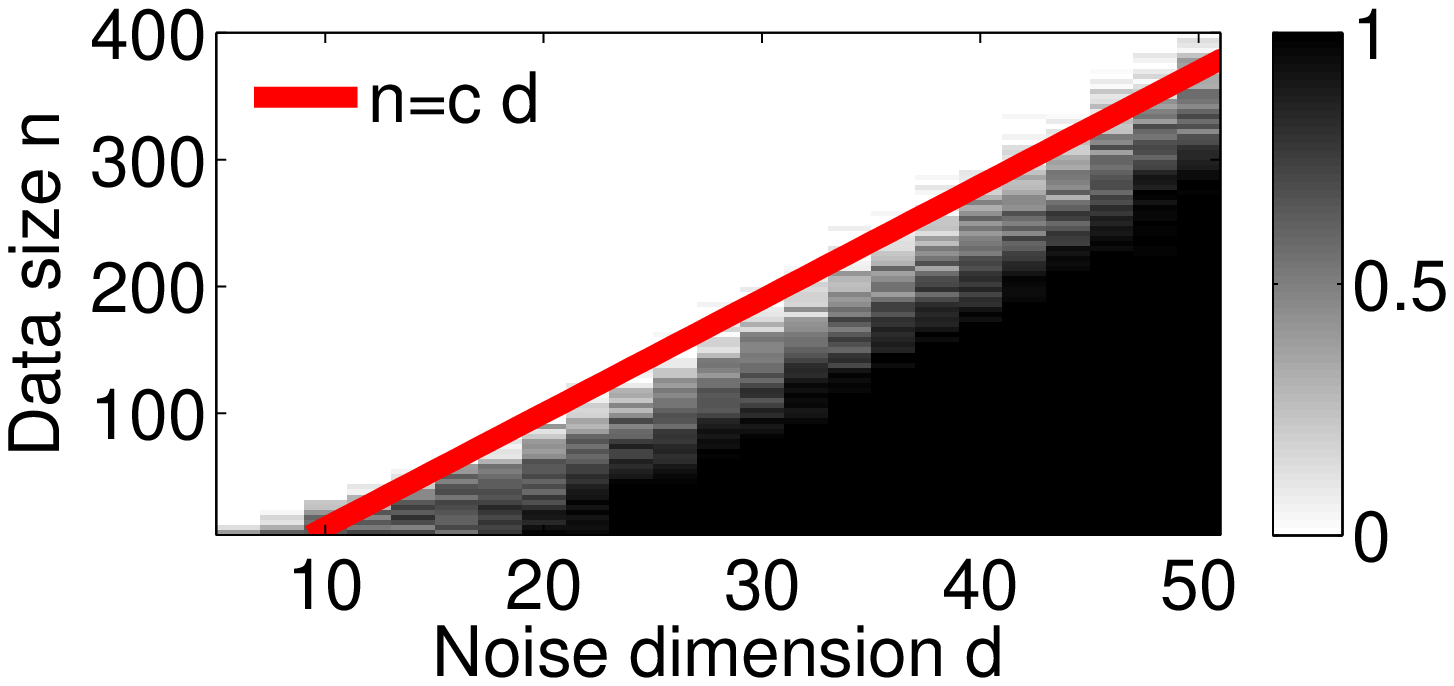}
      \end{minipage}   
		\vspace{-.2in}
    \subcaption{Phase transition for clustering error. Left: affine invariant case . Right: 4-sparse case. }
       \label{fig:phase2}
    \end{minipage} 
 \caption{Phase transition plots.}
         \vspace{-.2in}
    \label{fig:phasetrans}
\end{figure} 

\paragraph{Unbalanced case.}
We generate an unbalanced problem for $d=10$, $n=80$ and $\alpha_*=0.25$ and we average the results over $10$ replications.
We compare the clustering error for the constrained and the penalized relaxations when we consider the sign of the first or second eigenvector and when we use projection technique defined as $(\Pi_n Y_{(2)}\Pi_n)_{(1)}$ where $Y_{(k)}$ is the best rank-$k$ approximation of $Y$, to extract the information of $y$. 
We see in \myfig{unbalanced} that (a) for the constrained case, the range of $\alpha$ such that the sign of $y$ is recovered is cut in two parts where one eigenvector is correct, whereas the projection method performs well on the whole set. 
(b) For the penalized case, the correct sign is recovered for $\nu$ close to $0$ by the first eigenvector and the projection method whereas the second one performs always badly. 
(c) When there is zero noise the rank of the solution is one for $\alpha\in\{\alpha_*,1\}$, two for $\alpha\in(\alpha_*,1)$ and greater otherwise. These findings confirm our analysis. However, when $y$ is corrupted by some noise this result is no longer true.  

\begin{figure} 
    \hspace*{.5cm}
\begin{minipage}[c]{.31\linewidth}
\includegraphics[scale=0.32]{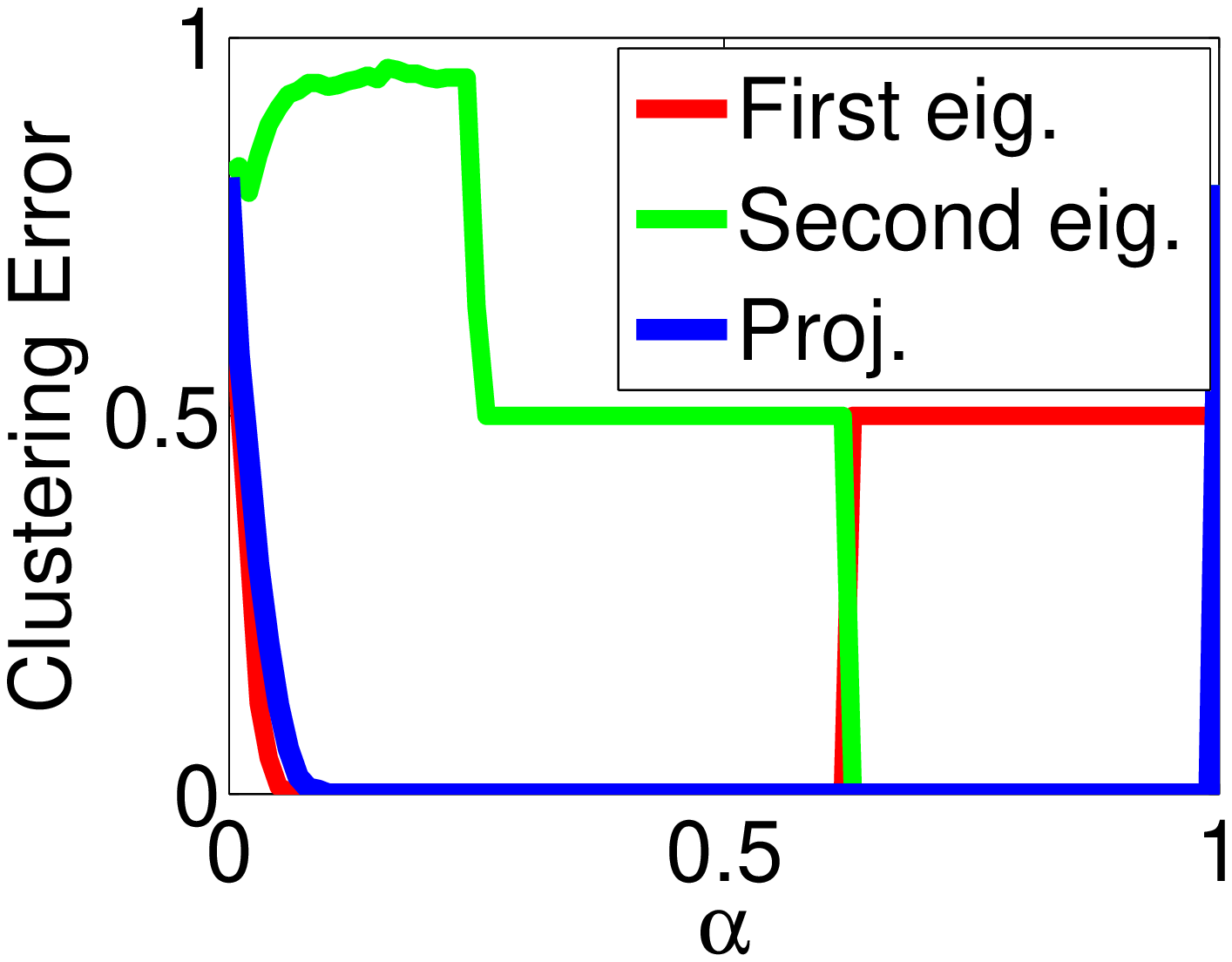}
   \end{minipage} 
   \begin{minipage}[c]{.31\linewidth}
\includegraphics[scale=0.32]{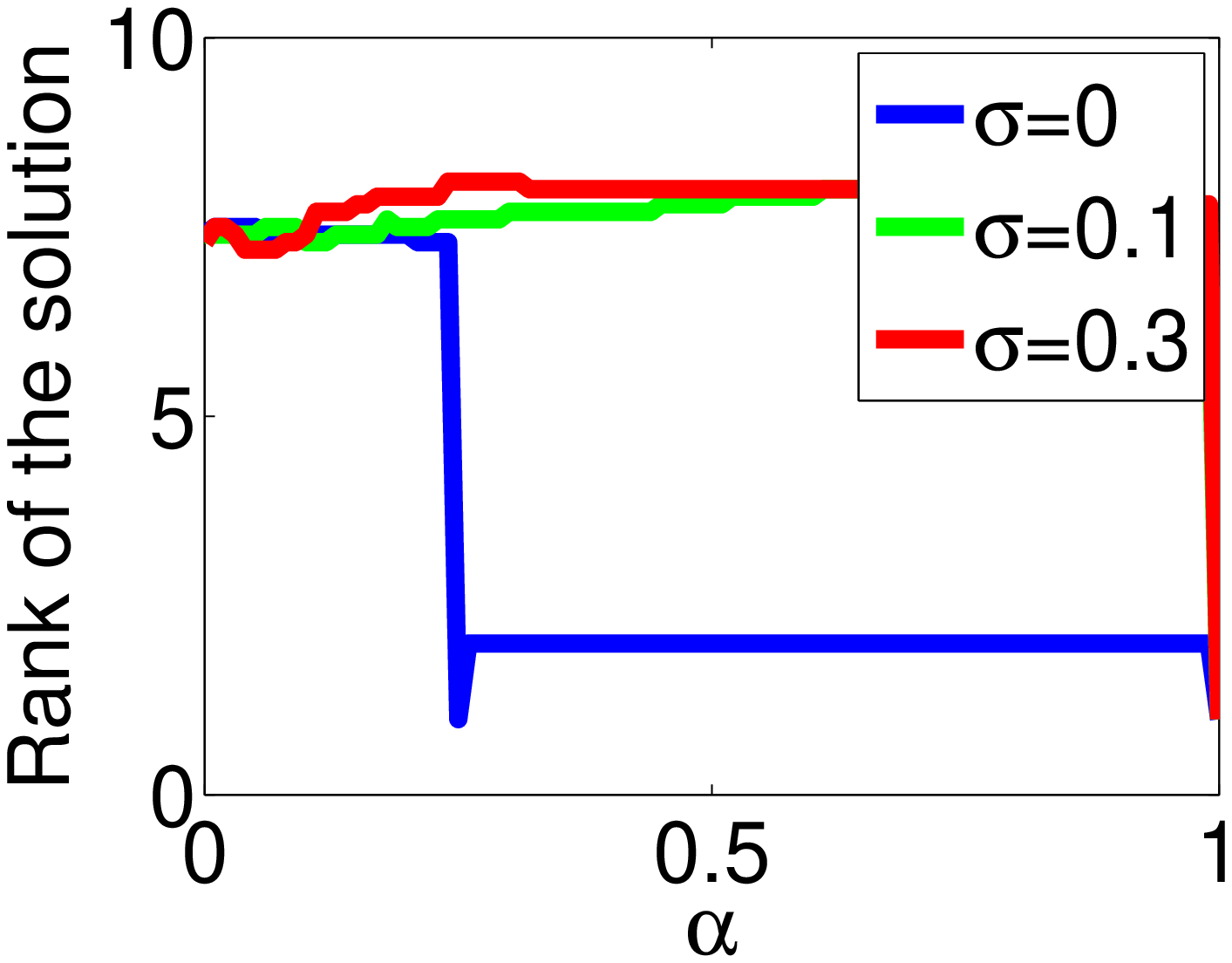}
      \end{minipage} 
   \begin{minipage}[c]{.31\linewidth}
\includegraphics[scale=0.32]{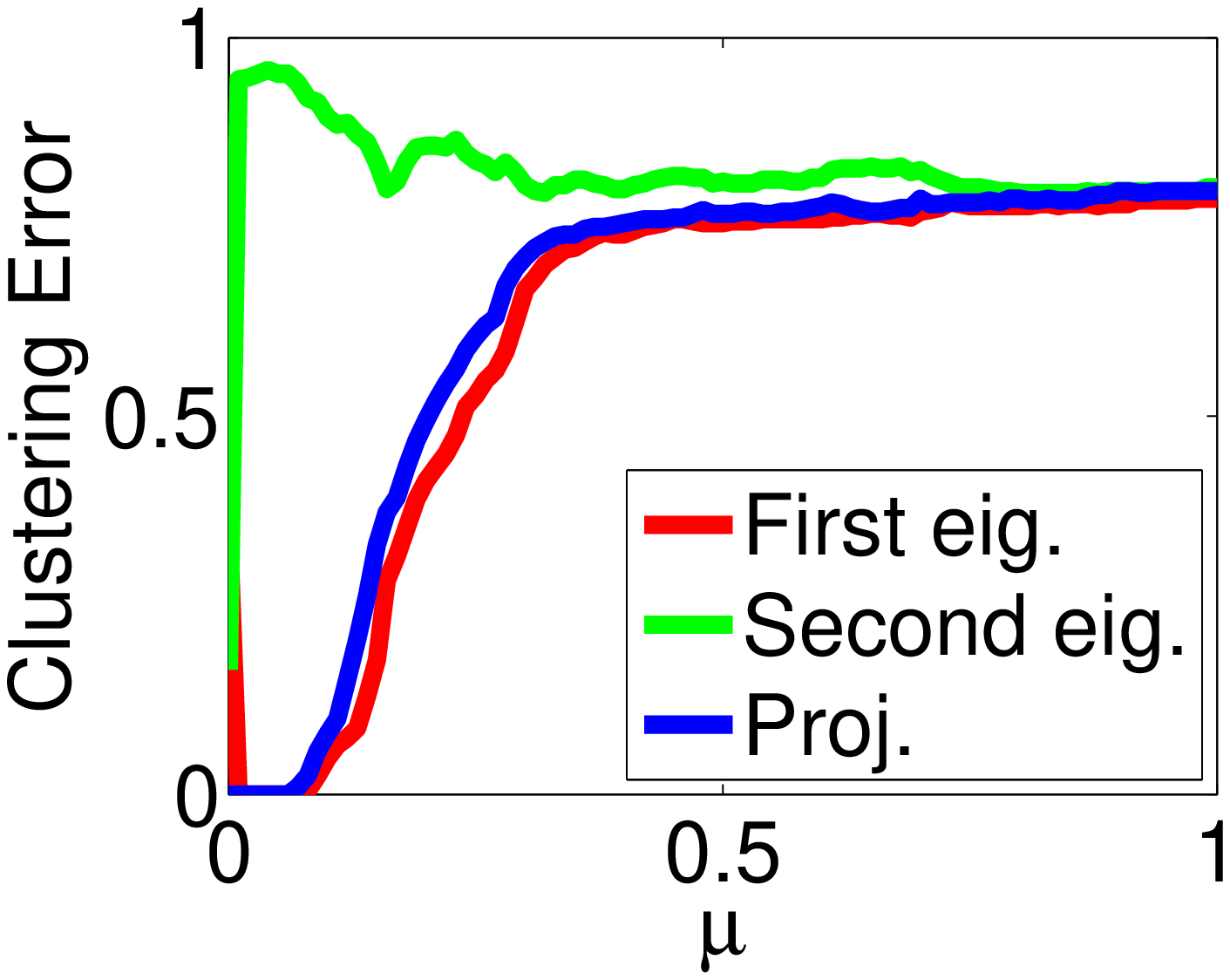}
      \end{minipage} 
                 
    \caption{Unbalanced problem for $n=80$, $d=10$ and $\alpha_*=0.25$. Left: Clustering error for the constrained relaxation. Middle: Rank of the solution for different level of noise $\sigma$. Right: Clustering error for the penalized relaxation.  }
   \label{fig:unbalanced}

\end{figure}

\paragraph{Runtime experiments.}
We also generated data with a $k$-sparse  direction of projection~$v$ by adding $d-k$ noise variables to a randomly generated and rotated $k$-dimension data.
The proposed optimization problem implemented using FISTA \citep{fista} was compared against a benchmark \texttt{cvx} solver  
to compare its scalability. 
Experiments were performed for $\lambda=0$ and $\lambda=0.001$, the coefficient associated with the sparse $\|V\|_1$ term.  
For a fixed $d$, \texttt{cvx} breaks down for large $n$ values (typically $n\geqslant 1000$). 
Similarly, the runtime required by \texttt{cvx} is generally high for $\lambda=0$ and is comparable to our method 
for $\lambda=0.001$. This behavior is illustrated in Figure~\ref{n_d_scalability}.

When $\lambda=0$, the problem reduces to the original Diffrac problem \citep{bach_diffrac_2007} and hence 
can be compared to an equivalent max-cut SDP \citep{manopt}. We observed that 
our method is comparable in terms of runtime and clustering performance of low-rank methods for max-cut 
(Figure~\ref{n_d_scalability}). However, for $\lambda>0$, the equivalence with max-cut disappears.

The plots in these figures show the behavior of FISTA for two different stopping criteria: 
$\varepsilon = 10^{-2} / \log(d)$ and $\varepsilon = 10^{-3} / \log(d)$. 
It is observed that the choice $10^{-3} / \log(d)$ gives a better accurate solution at the cost 
of more number of iterations (and hence higher runtime). For sparse problems in  \myfig{scalb}, we see that cvx gets a better clustering performance (while crashing for large $n$); the difference would be reduced with a smaller duality gap for FISTA.

\begin{figure}[!htbp]
   \begin{minipage}[b]{1\linewidth}
   \centering
      \begin{minipage}[b]{0.45\linewidth}
	  \includegraphics[scale=0.41]{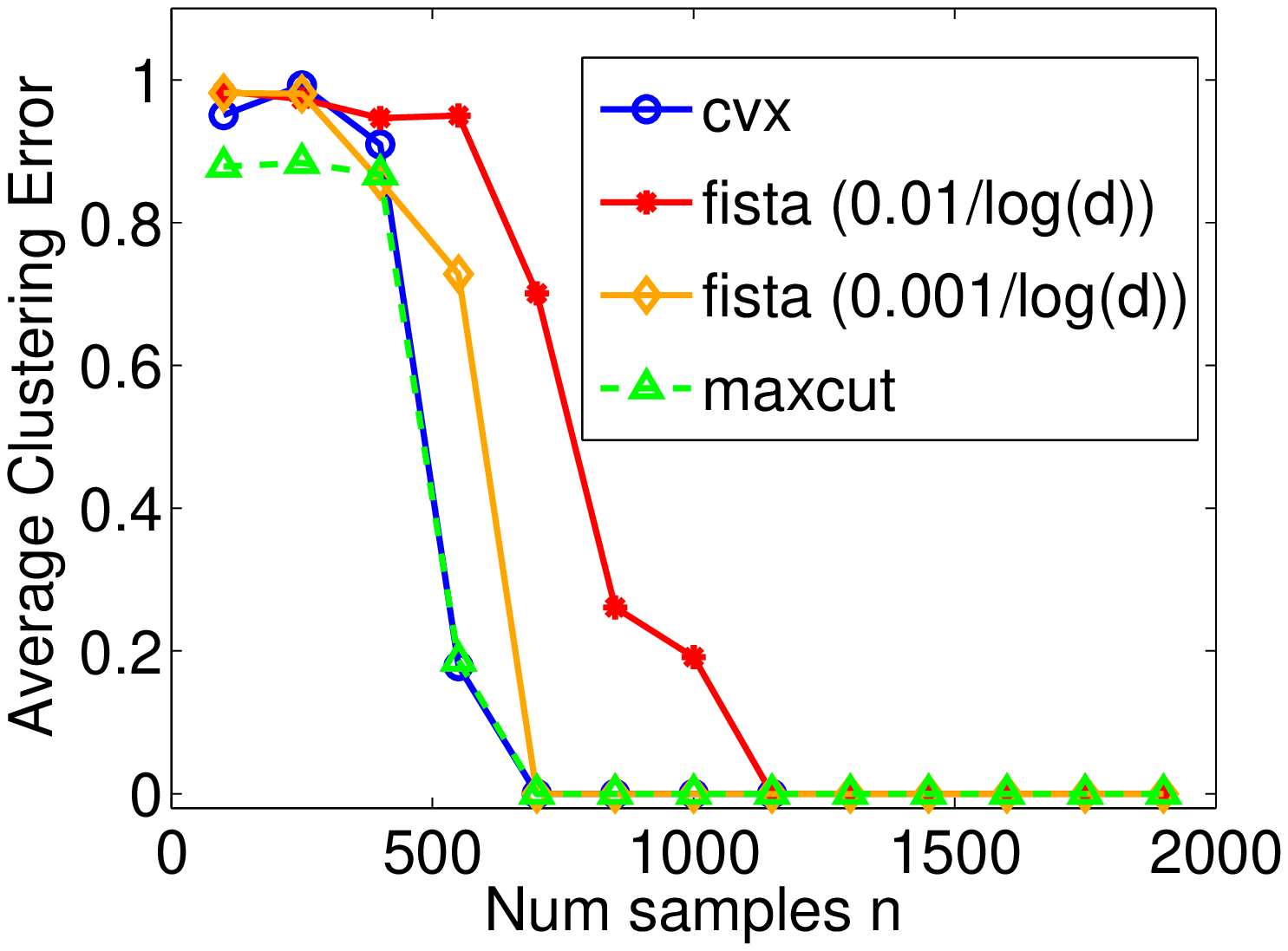}
      \end{minipage}
      \begin{minipage}[b]{0.4\linewidth}
	  \includegraphics[scale=0.41]{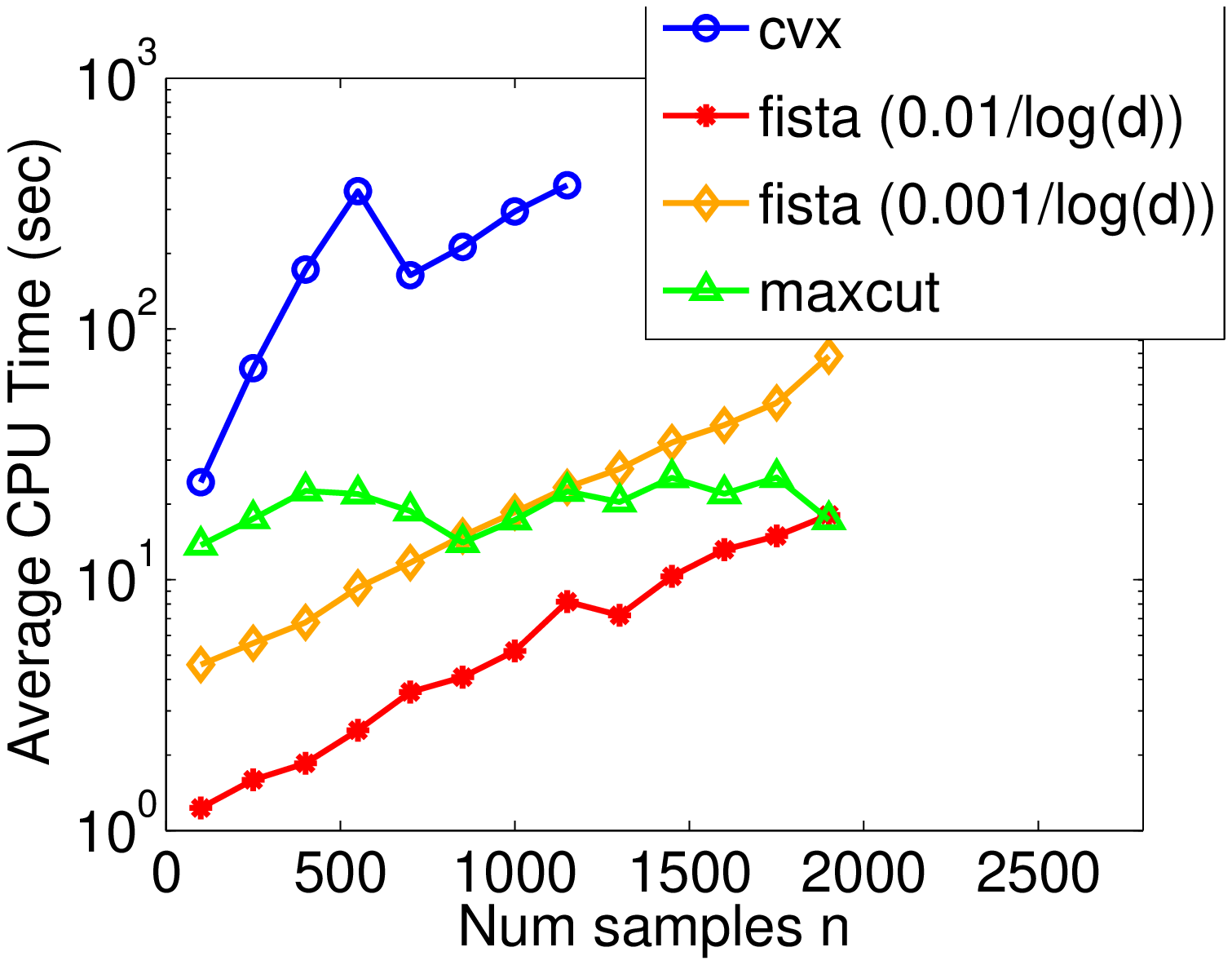}
      \end{minipage}
      \end{minipage}
      \begin{minipage}[b]{1\linewidth}
             \centering
      \begin{minipage}[b]{0.45\linewidth}
	  \includegraphics[scale=0.41]{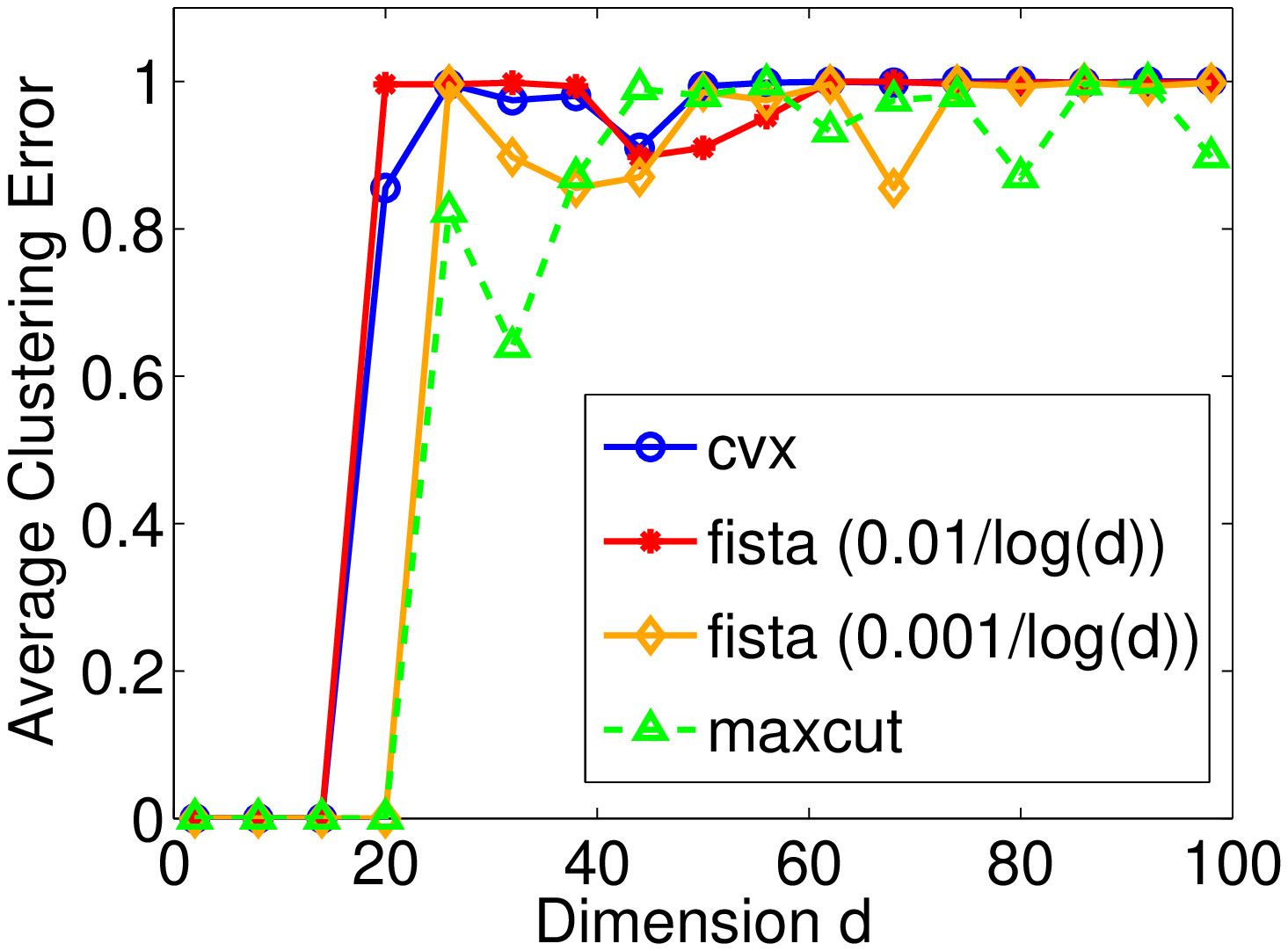}
      \end{minipage}
      \begin{minipage}[b]{0.4\linewidth}
	  \includegraphics[scale=0.41]{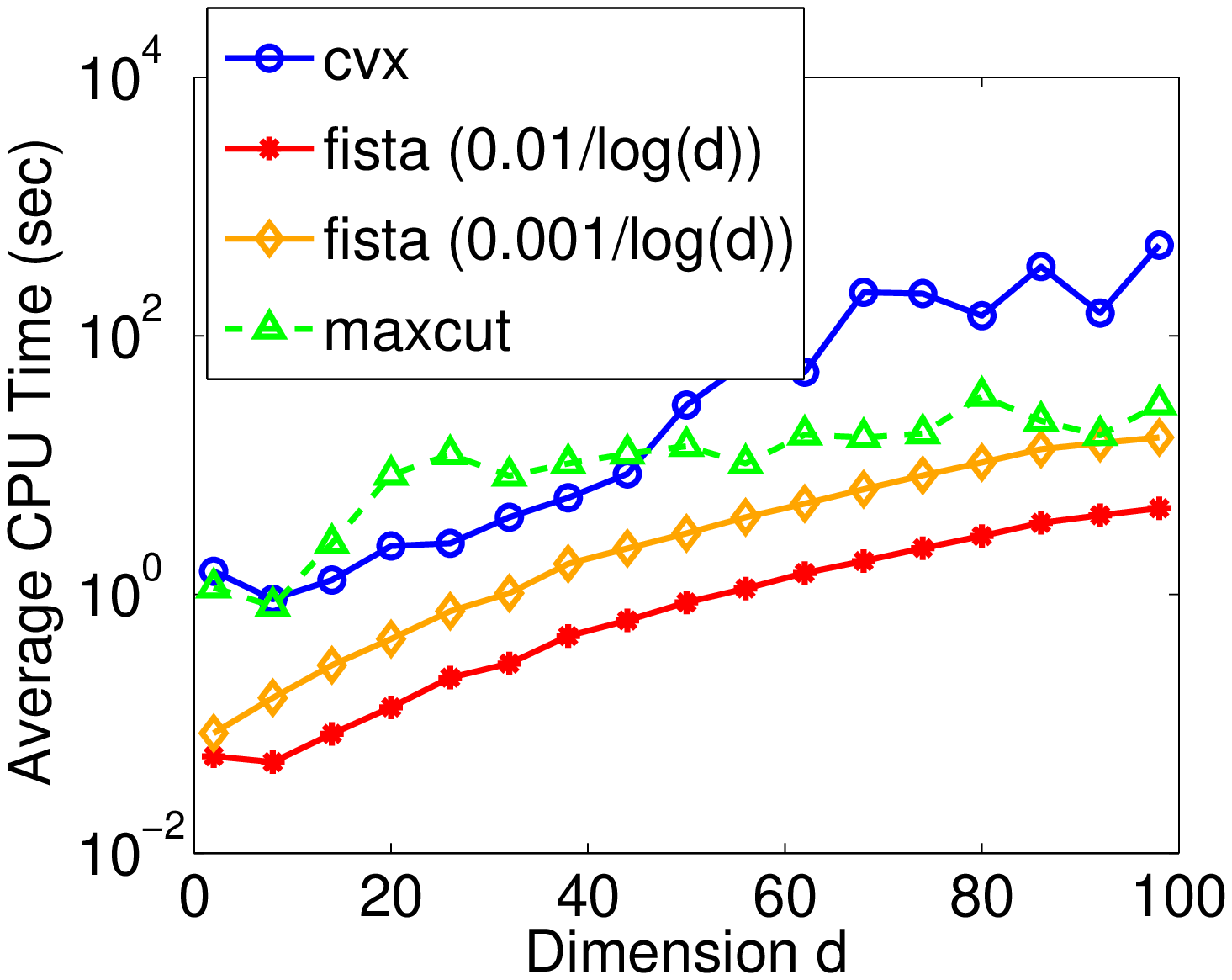}
      \end{minipage}
		\subcaption{\texttt{cvx}, max-cut comparison with $\lambda=0$. Top: $n$ varied with $d=50$, $k=6$. \texttt{cvx} crashed for $n\approx 1000$. Bottom: $d$ varied with $n=100$, $k=2$.} 
      \end{minipage}

   \begin{minipage}[b]{1\linewidth}
   \centering
      \begin{minipage}[b]{0.45\linewidth}
	  \includegraphics[scale=0.41]{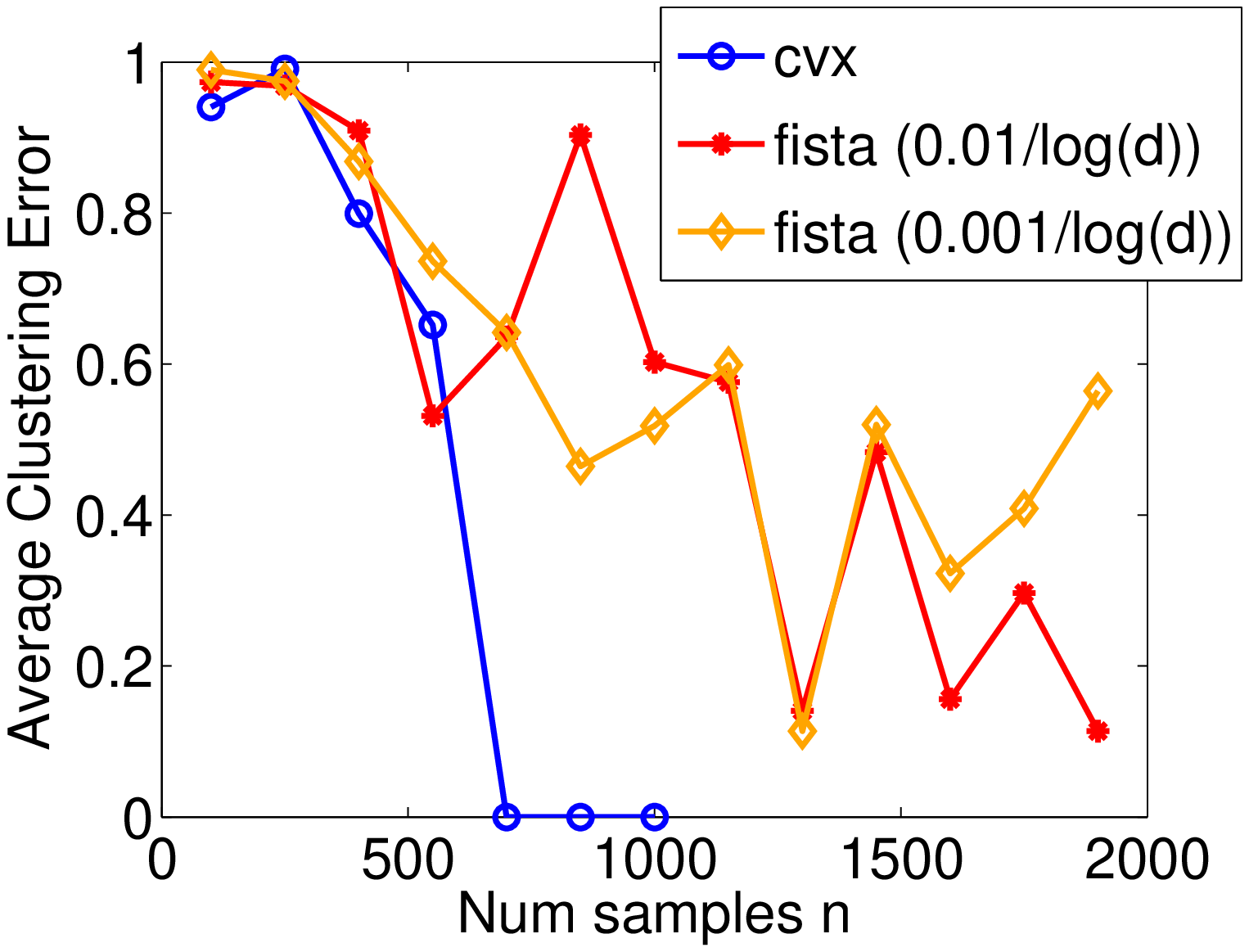}
      \end{minipage}
      \begin{minipage}[b]{0.4\linewidth}
	  \includegraphics[scale=0.41]{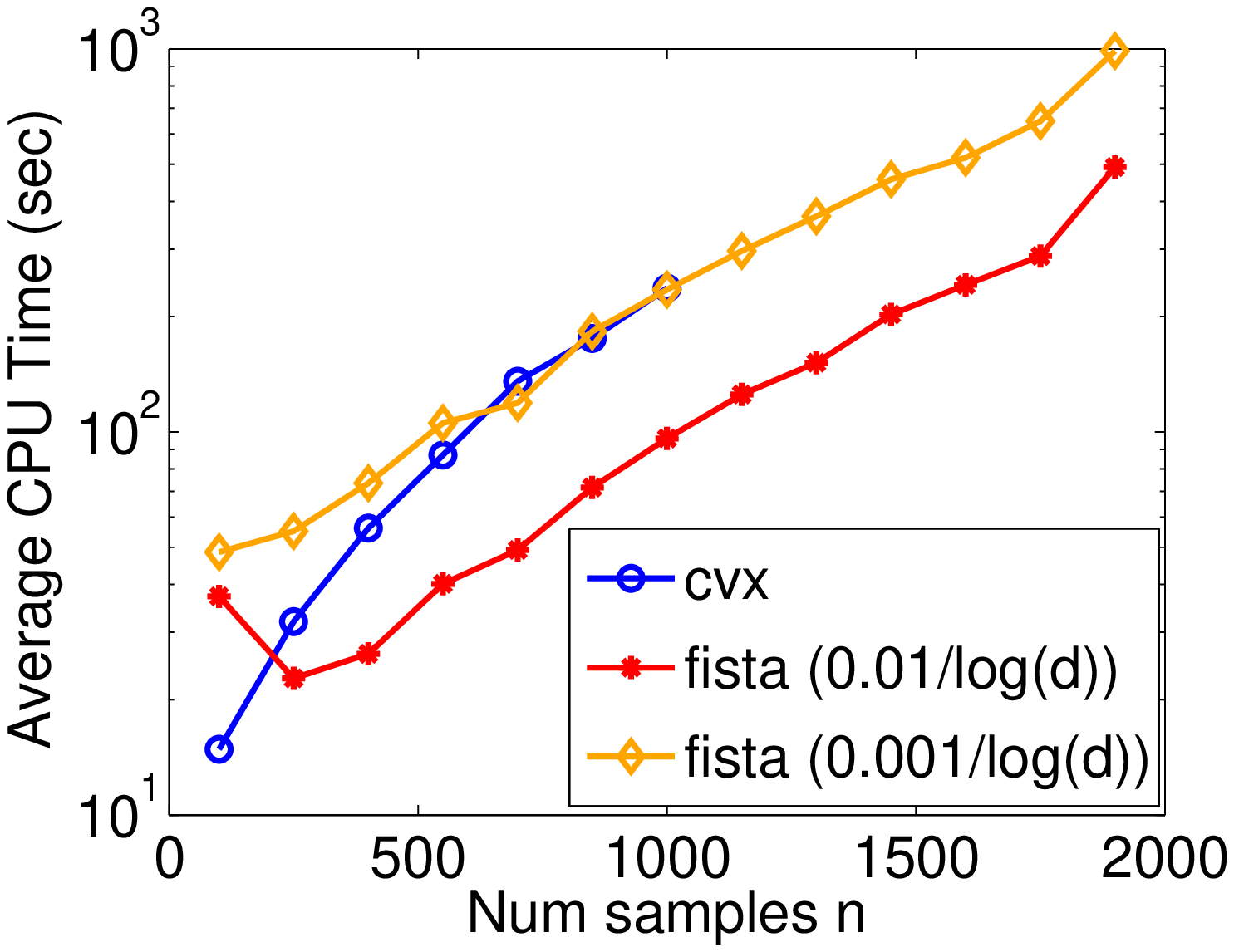}
      \end{minipage}
		\end{minipage}
      \begin{minipage}[b]{1\linewidth}
		\centering
      \begin{minipage}[b]{0.45\linewidth}
	  \includegraphics[scale=0.41]{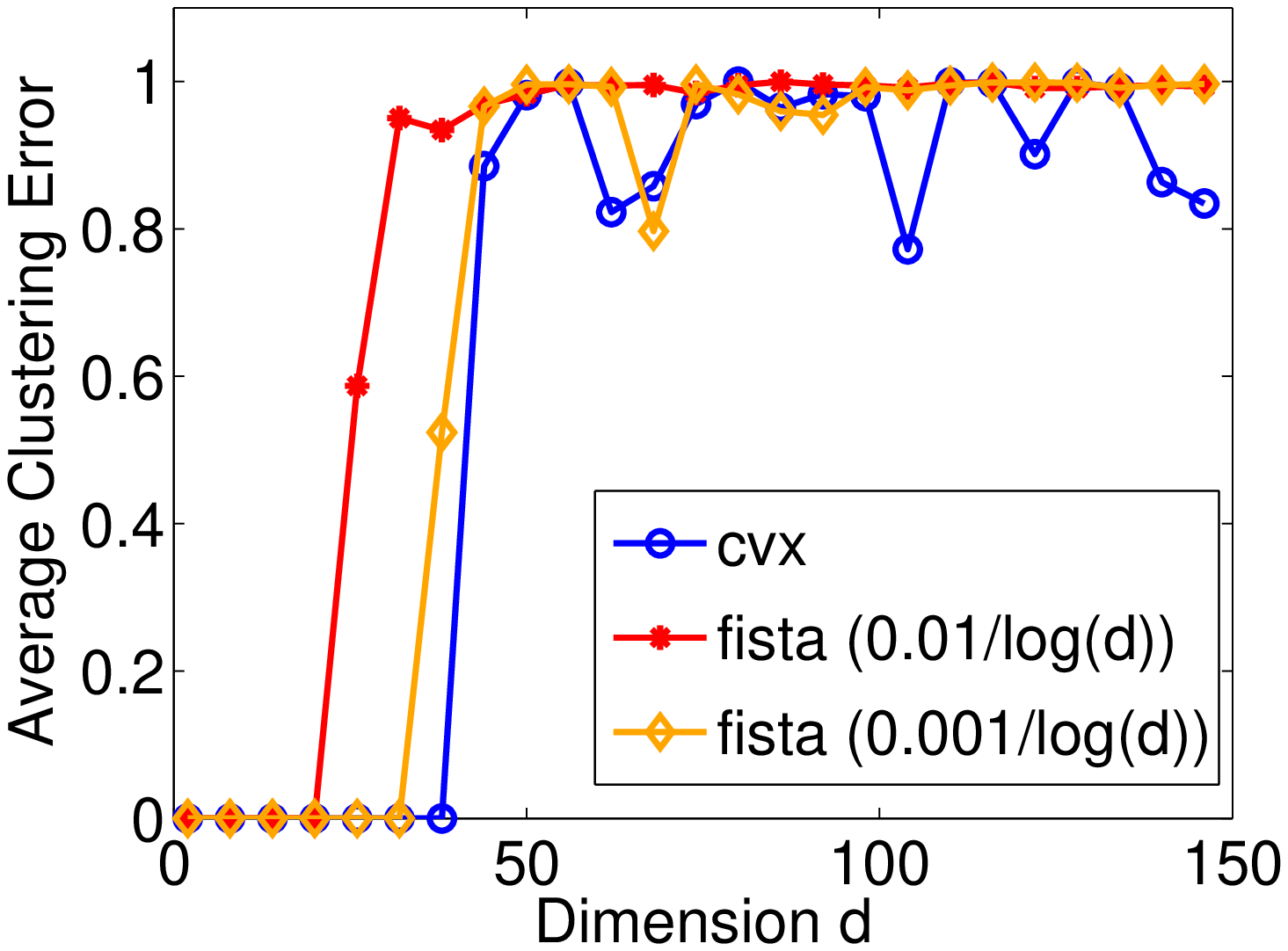}
      \end{minipage}
      \begin{minipage}[b]{0.4\linewidth}
	  \includegraphics[scale=0.41]{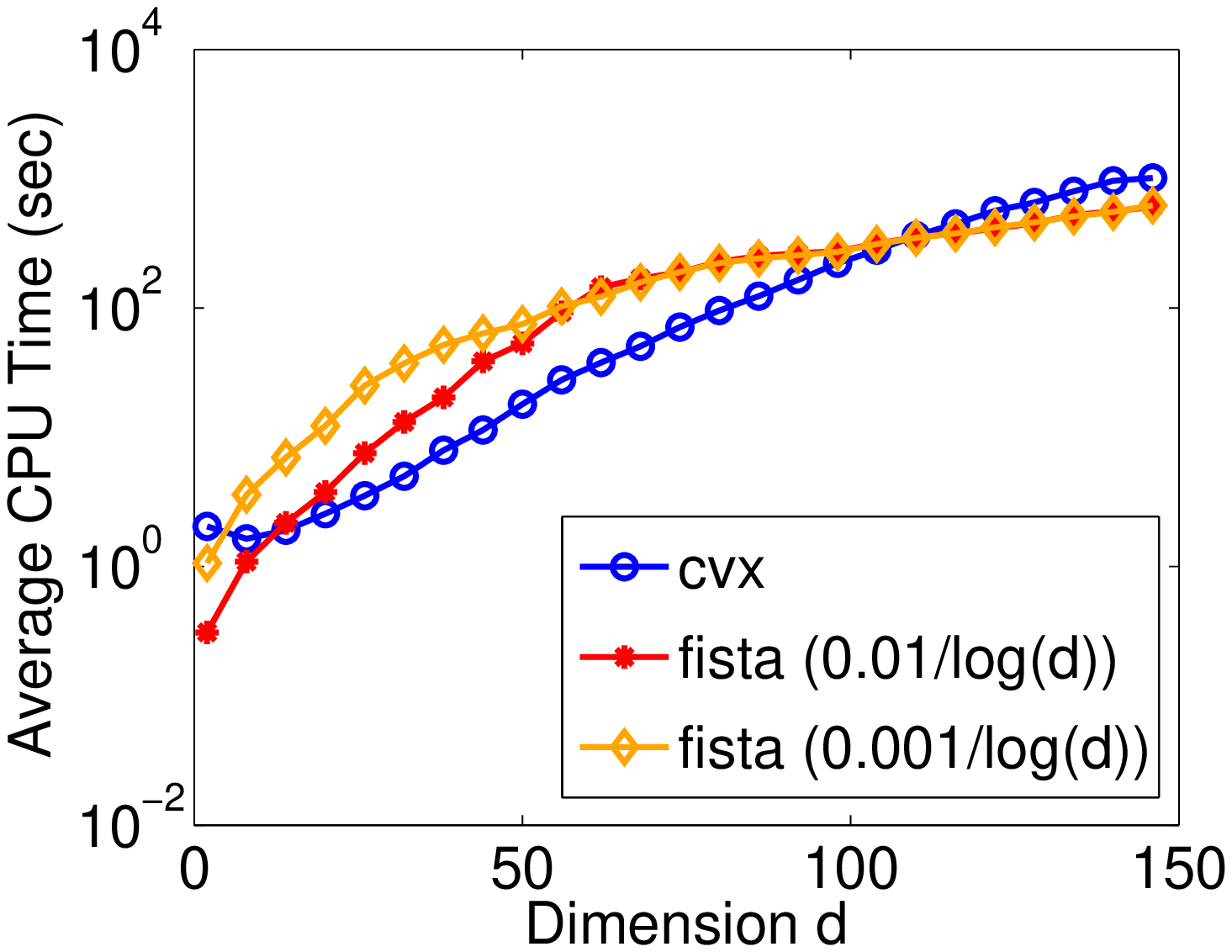}
      \end{minipage}
        \subcaption{\texttt{cvx} comparison with $\lambda=0.001$. Top: $n$ varied with $d=50$, $k=6$. \texttt{cvx} crashed for $n\approx 1000$. Bottom: $d$ varied with $n=100$, $k=2$.} 
        \label{fig:scalb}
        \end{minipage}
         \caption{Scalability experiments.} 
   \label{n_d_scalability}
            \vspace{-.15in}
\end{figure}

\paragraph{Clustering performance.}
Experiments comparing the proposed method with $K$-means and alternating optimization are given in 
Figure~\ref{n100_varyd_kmeansaltoptfista_perf}. $K$-means is run on the whitened variables in $\RR^d$.  
Alternating optimization is another popular method \citet{ye2008discriminative} 
for dimensionality reduction with clustering  
(where alternating optimization of $w$ and $y$ is performed to solve the non-convex formulation~\eqref{eq:maxcorr}). 
The plots show that both $K$-means and 
alternating optimization fail when only a few dimensions of noise variables are present. 
The plots also show that with the introduction of a sparse regularizer (corresponding to the non-zero $\lambda$)  
the proposed method becomes more robust to noisy dimensions. 
As observed earlier, the performance of FISTA is also sensitive to the choice of $\varepsilon$.  

\begin{figure}[h]
      \begin{minipage}[b]{0.45\linewidth}
	  \includegraphics[scale=0.4]{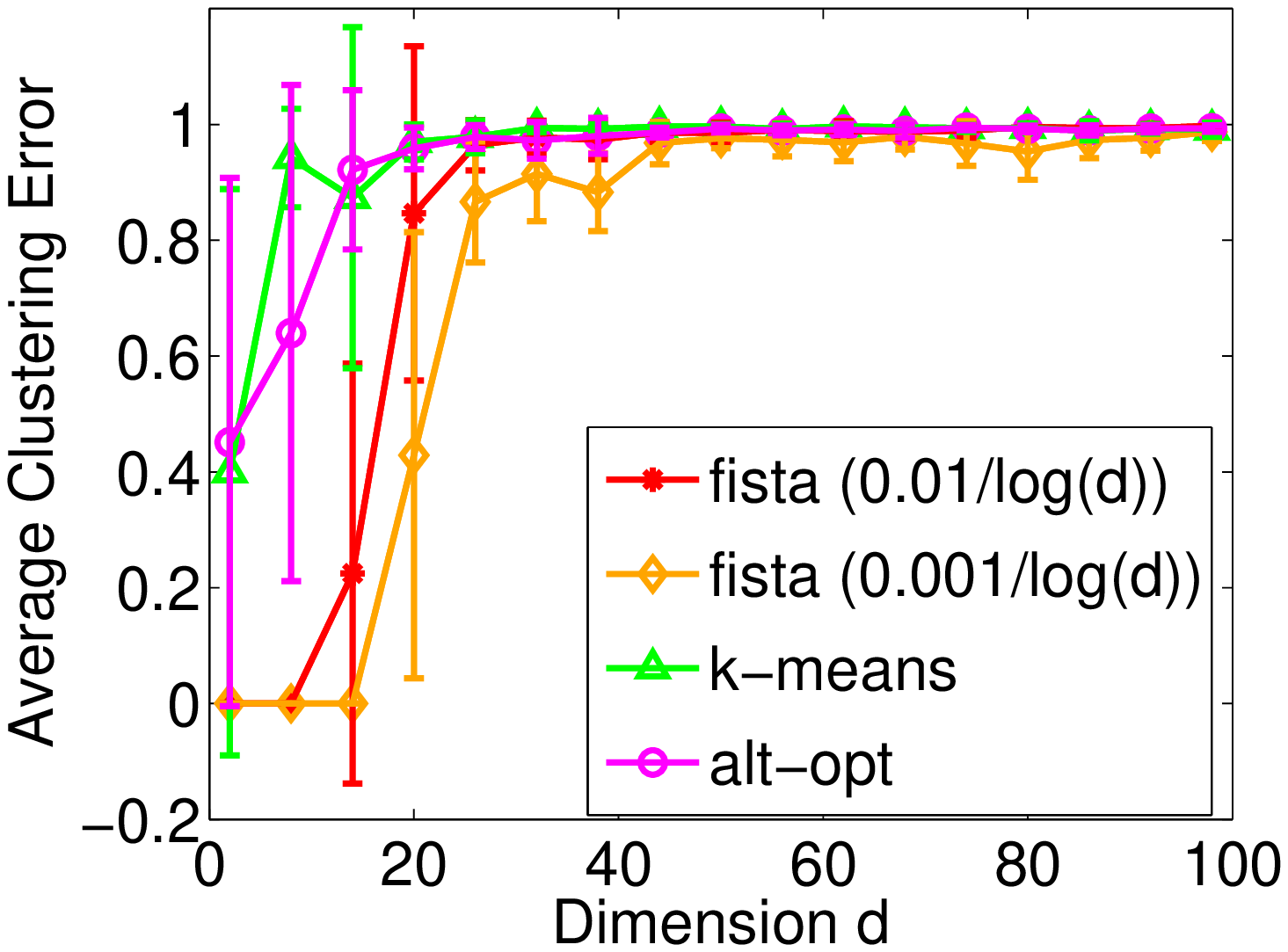}
	  \subcaption{$\lambda=0$}
      \end{minipage}
      \begin{minipage}[b]{0.4\linewidth}
	  \includegraphics[scale=0.4]{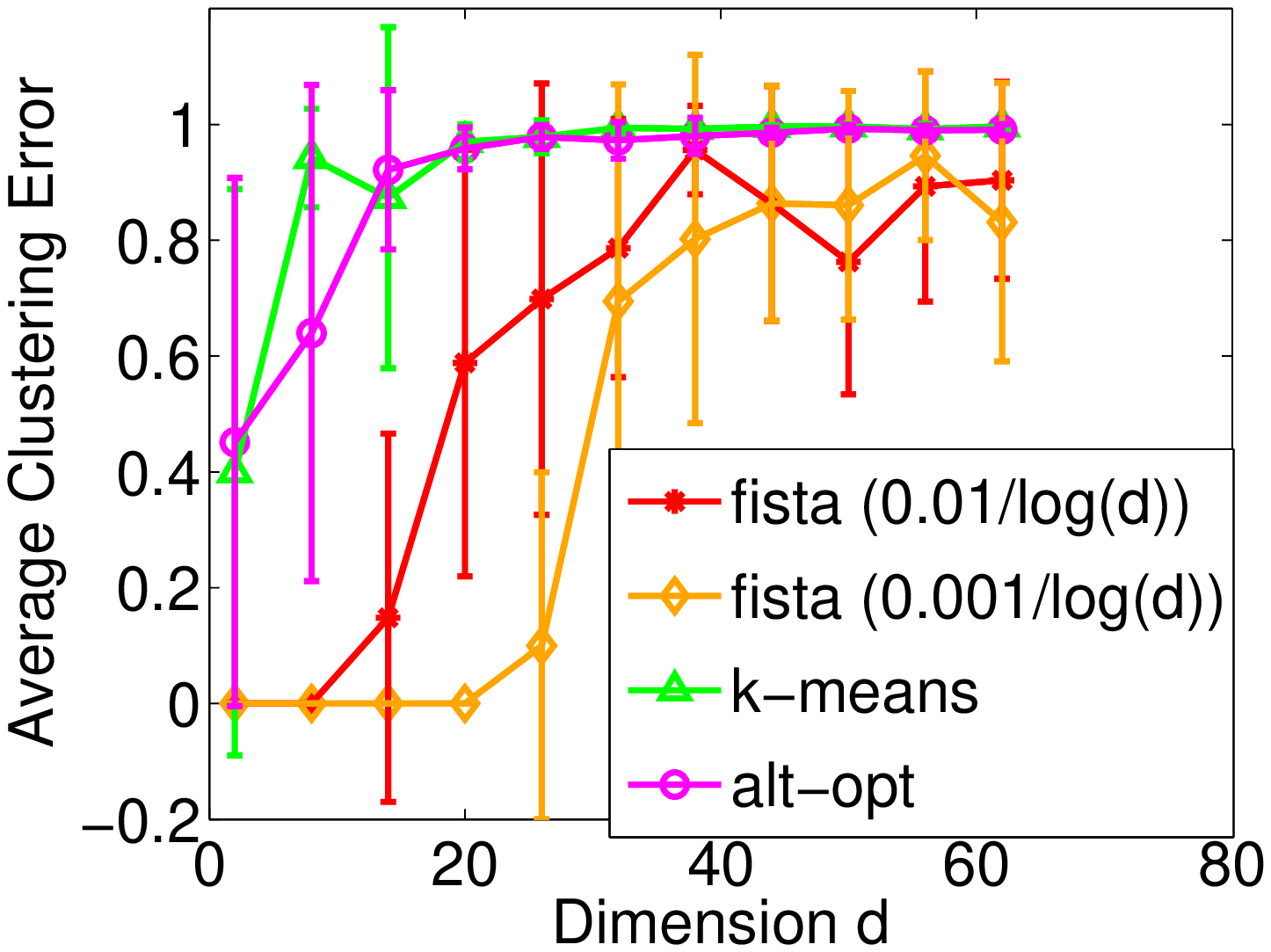}
	  \subcaption{$\lambda=0.01$}
      \end{minipage}
      \begin{minipage}[b]{1\linewidth}
      \center
	  \includegraphics[scale=0.4]{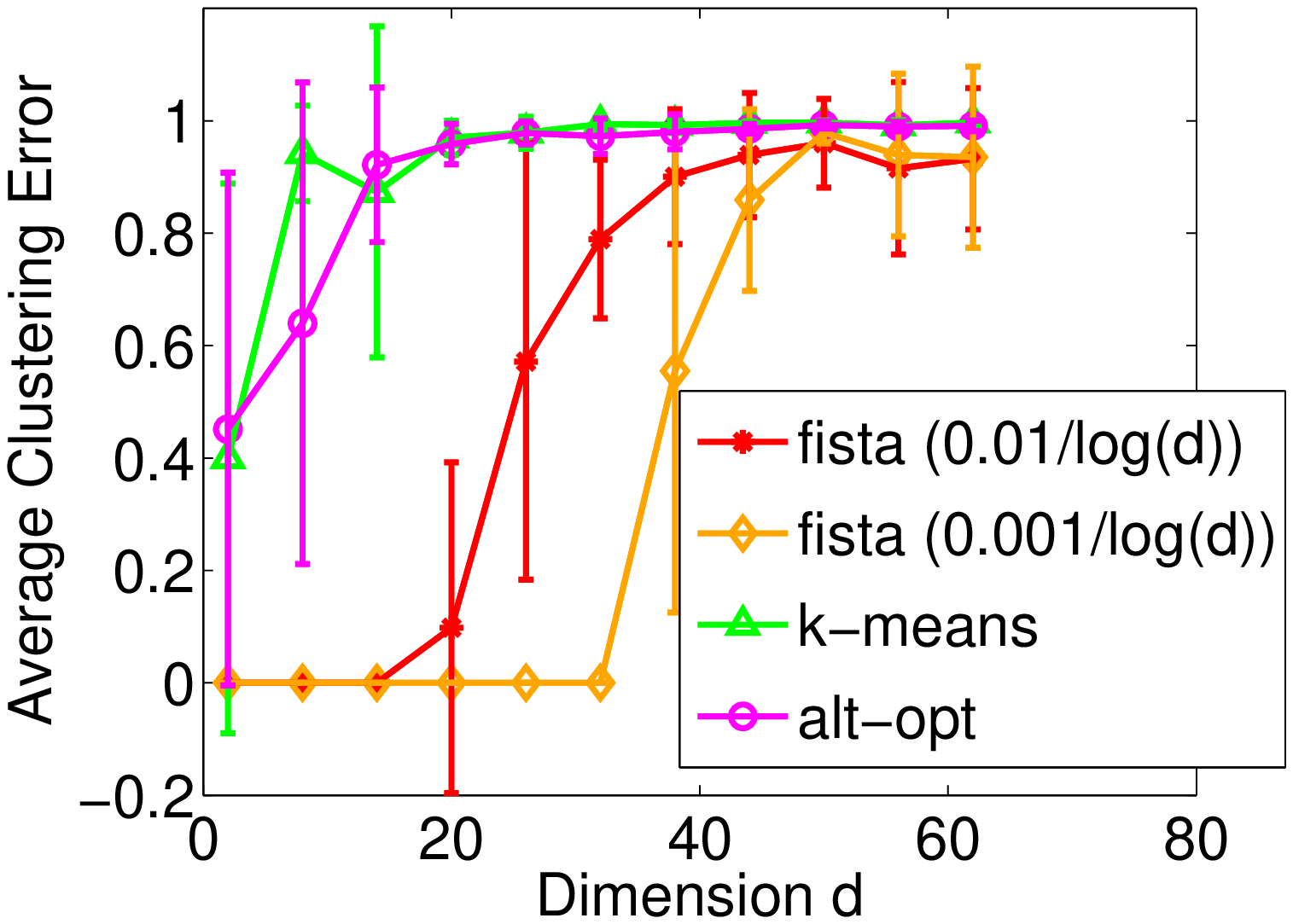}
	  \subcaption{$\lambda=0.001$}
      \end{minipage}
    \caption{Comparison with $k$-means and alternating optimization}
   \label{n100_varyd_kmeansaltoptfista_perf}
\end{figure} 

\subsection{Experiments on real-world data}\label{sec:realworld}

\paragraph{Experiments on two-class data.} Experiments were conducted on real two-class classification datasets\footnote{The data sets were obtained from~\url{https://www.csie.ntu.edu.tw/~cjlin/libsvmtools/datasets/}} to compare the performance of sparse discriminative  clustering against non-sparse discriminative clustering, alternating optimization and $K$-means algorithms. 
For the two-class datasets, the clustering performance for a cluster $\bar y \in \{+1,-1\}^n$ obtained from an algorithm under comparison, was computed as $1-(\bar y\t y/n)^2$, where $y$ is the original labeling. Here we explicitly compare the output of clustering with the original labels of the data points.

The dataset details and clustering performance results are summarized in Table~\ref{twoclass_results_table}. The experiments for discriminative clustering were conducted for different values of $a, c \in \{10^{-3}, 10^{-2}, 10^{-1}\} $ associated with the $\ell_2$-regularizer and $\ell_1$-regularizer respectively. The range of cluster imbalance parameter was chosen to be $\nu \in \{0.01, 0.25, 0.5, 0.75, 1 \}$. The results given in  Table~\ref{twoclass_results_table} pertain to the best choices of these parameters. The results for alternating optimization and $K$-means show the average cluster error (and standard deviation) over 10 different runs. These results show that the cluster error is quite high for many datasets. This is primarily due to the absence of an ambient low-dimensional clustering of the two-class data, which can be identified by the simple linear model presented in this paper. The results also show that adding sparse regularizers to discriminative clustering helps in a better cluster 
identification when compared to the non-sparse case and the other algorithms like alternating optimization and $K$-means.

\begin{table}[!h]
\small
    \caption{ {Experiments on two-class datasets}} 
    \label{twoclass_results_table}
    \begin{center}
    \begin{tabular}{l c c c c c c }
    \toprule 
		Dataset & $n$ &  $d$ & \multicolumn{4}{c}{Cluster Error} \\
		\cline{4-7}
		&  &  & Sparse & Non-sparse & Alternating & $K$-means \\
		&  &  & Discriminative & Discriminative & Optimization &  \\
		&  &  & Clustering & Clustering &  & \\
		\midrule
		Heart & 270  & 3 & \textbf{0.52} & 0.61 & 0.97 $\pm$ 0.03 & 0.91 $\pm$ 0.09 \\
		Diabetes & 768 & 8 & \textbf{0.88} & \textbf{0.88} & 0.91 $\pm$ 0.05 & 0.93 $\pm$ 0.06  \\
	    Breast-cancer & 683 & 10 & \textbf{0.15} & \textbf{0.15} & 0.48 $\pm$ 0.17 & 0.68 $\pm$ 0.24 \\
	    Australian & 690 & 14 & \textbf{0.5} & \textbf{0.5} & 0.88 $\pm$ 0.17 & 0.87 $\pm$ 0.21 \\
	    Liver-disorder & 345 & 6 & \textbf{0.97} & \textbf{0.97} & 0.99 $\pm$ 0.01 & 0.99 $\pm$ 0.01 \\
	    Sonar & 208 & 60 & \textbf{0.92} & 0.95 & 0.98 $\pm$ 0.02 & 0.99 $\pm$ 0.01 \\
    	    DNA(1 vs 2,3) &  1400 & 180  & \textbf{0.75} & 0.83 & 0.99 $\pm$ 0.01 & 0.98 $\pm$ 0.02 \\
	    a1a &  1605 &  113  &  \textbf{0.74} & 0.75 & 0.98 $\pm$ 0.02 & 0.8 $\pm$ 0.08 \\
	    w1a &  2270 &  290  &   \textbf{0.11}  & \textbf{0.11} & 0.92 $\pm$ 0.08 & 0.16 $\pm$ 0.06 \\\bottomrule
       \end{tabular}
    \end{center}
    \end{table} 

\clearpage

\paragraph{Experiments on real multi-label data.} Experiments were also conducted on the Microsoft COCO dataset\footnote{Dataset obtained from~\url{http://mscoco.org/dataset}} to demonstrate the effectiveness of the proposed method in discovering multiple labels. We considered $n=2000$ images from the dataset, each of which was labeled with a subset of $K=80$ labels. The labels identified the objects in the images like person, car, chair, table, etc. and the corresponding features for each image were extracted from the last layer of a conventional convolutional neural network (CNN). The CNN was originally trained over the imagenet data \citep{imagenet}. 

For each image in the dataset, we obtained $d=1000$ features. We then performed discriminative clustering on the 2000 $\times$ 1000 data matrix $X$ and obtained the label matrix $Y$ which was then subjected to the alternating optimization procedure (see Section \ref{subsec:multilabel_analysis}). 

It is clearly unlikely to recover perfect labels; therefore we now describe a way of measuring the amount of information which is recovered.
In order to extract meaningful cluster information from the result so-obtained, we computed the correlation matrix $Y_k \Pi_n Y_{true}$ where $Y_{true}$ is the $n \times K$ label matrix containing actual labels and $\Pi_n$ is the $n \times n$ centering matrix $I_n - \frac{1}{n}1_n 1_n^\top$. The $k$ predicted labels are present in the $Y_k$ matrix. In order to choose an appropriate value of $k$, we plotted Tr$(\Phi_{Y_{true}} \Phi_{Y_{k}})$ (shown in Figure \ref{trace_Y_Ykfig} along with a $K$-means baseline), where $\Phi_{Y_{k}}=Y_k({Y_k}^\top Y_k)^{-1}Y_k^\top$. 
From these plots, we chose $k=30$ to be a suitable value for our interpretation purposes.

\begin{figure} 
	\begin{minipage}[c]{.9\linewidth}
		\center
		\includegraphics[scale=0.3]{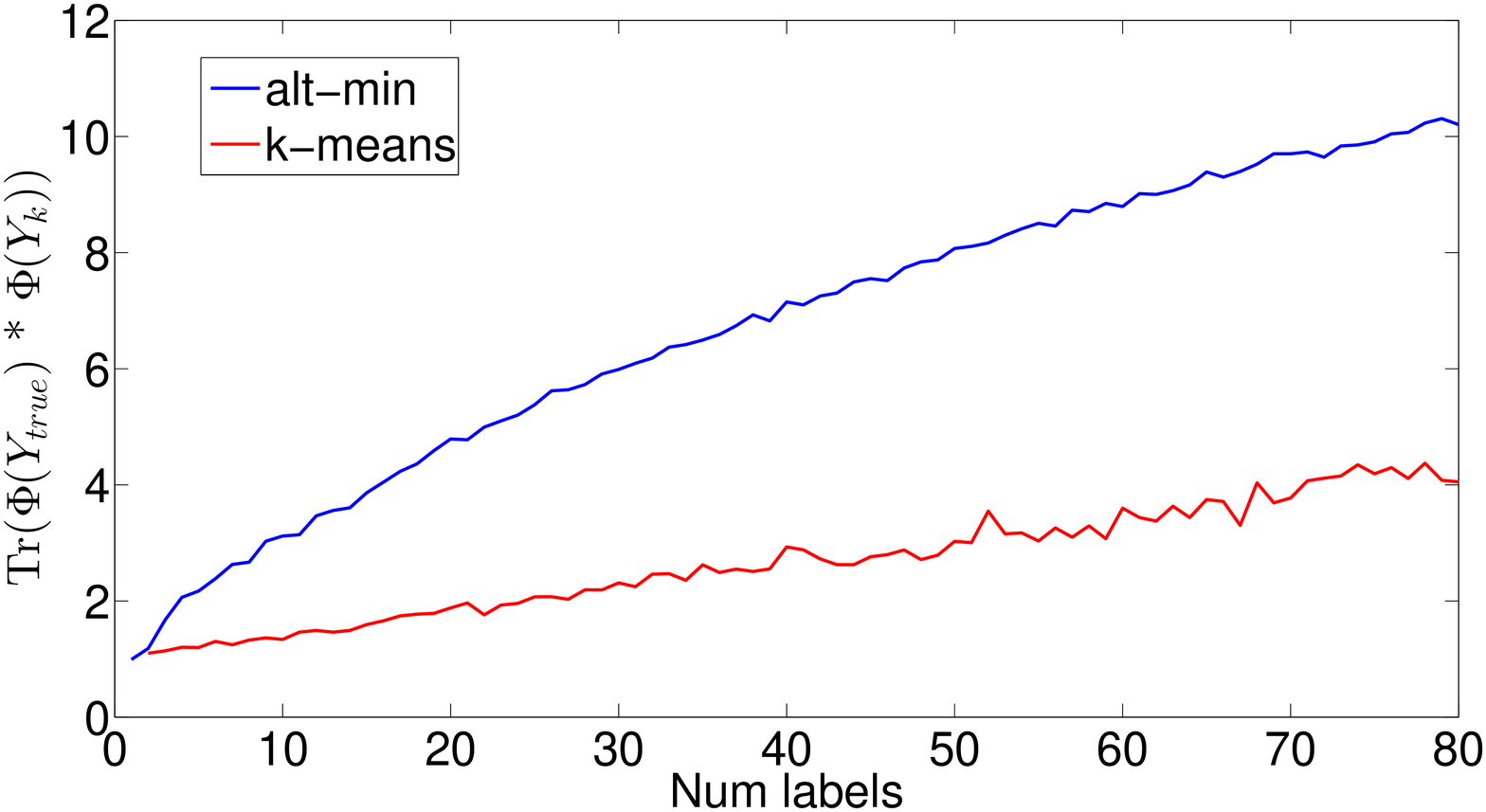}
	\end{minipage} 
	\caption{Plot of Tr$(\Phi_{Y_{true}} \Phi_{Y_{k}})$.}
	\label{trace_Y_Ykfig}
\end{figure}

After choosing an arbitrary value of $k=30$, we plotted the correlations between the actual and predicted labels. The heatmap of the normalized absolute correlations is given in Figure \ref{heatmap_absbyn_orderedrowcolfig}, where the columns and rows corresponding to the 80 true labels and 30 predicted labels respectively, are ordered according to the sum of squared correlations (the top-scoring labels appear to the left-bottom). From this plot, we extract following highly correlated labels: person, dining table, car, chair, cup,  tennis racket,  bowl, truck, fork, pizza, showing that these labels were partially recovered by our unsupervised technique (note that the CNN features are learned with supervision on the different dataset Imagenet, hence there is still some partial supervision).

\begin{figure} 
	\centering
	\begin{minipage}[c]{.9\linewidth}
		\includegraphics[scale=0.32]{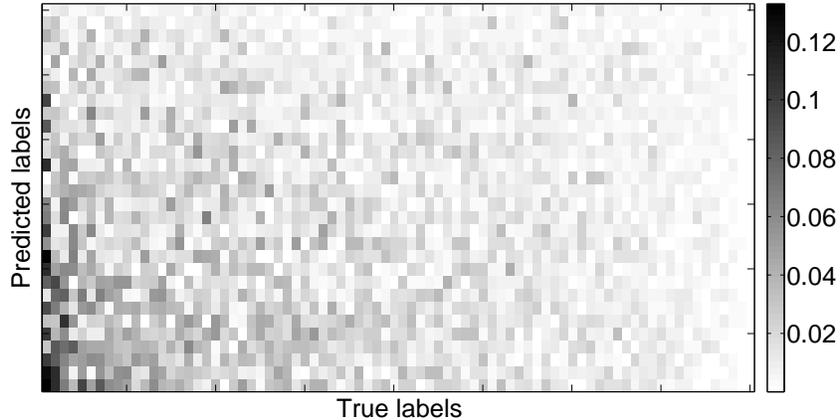}
	\end{minipage} 
	\caption{Heatmap of correlations, $Y_k \Pi_n Y_{true}$ with $k=30$, with columns and rows ordered according to the sum of squared correlations.}
	\label{heatmap_absbyn_orderedrowcolfig}
\end{figure}

\section{Conclusion}

In this paper, we provided a sparse extension of the discriminative clustering framework, and gave a first analysis of its theoretical performance in the totally unsupervised situation, highlighting provable scalings between ambient dimension $d$, number of observations and ``clusterability'' of irrelevant variables.
We also proposed an efficient algorithm which is the first of its kind to be linear in the number of observations. Our work could be extended in a number of ways, e.g., 
extending the sparse analysis to $l$-sparse case with higher $l$, 
considering related weakly supervised learning extensions \citep{joulin2012convex}, going beyond uniqueness of rank-one solutions, and improving the complexity of our algorithm to~$O(nd)$, for example using stochastic gradient techniques.

\newpage

\bibliography{diffrac2015}

\newpage
\appendix

\section{Joint clustering and dimension reduction}

\label{app:kmean}

Given $y$, we need to optimize the Rayleigh quotient $\frac{ w^\top X^\top y y^\top X w}{ w^\top X^\top X w}$ with a rank-one matrix in the numerator, which leads to $w=(X^\top X)^{-1} X^\top y$. 
Given $w$, we will show  that the averaged distortion measure of $K$-means once the means have been optimized is exactly equal to $ { (y^\top \Pi_n  X w)^2}/{\| \Pi_n  y\|_2^2 }$. Given the data matrix $X \in \rb^{n \times d}$, $K$-means to cluster the data into two components 
will tend to approximate the data points in $X$ by the centroids $\cplus \in \rb^{d}$ and $\cminus \in \rb^{d}$ such that 
\BEAS
 X &\approx& \frac{(y+1_n)}{2} \cplustrans - \frac{(y-1_n)}{2}\cminustrans \ (\text {since $y\in\{-1,1\}^n$})    \\
   &=& \frac{y}{2}(\cplustrans - \cminustrans) + \frac{1}{2} 1_n (\cplustrans + \cminustrans).   
\EEAS
The objective of $K$-means can now be written as problem $\calKM$:
\BEAS
 && \min_{y, \cplus, \cminus} \bigg \| X - \frac{y}{2}(\cplustrans - \cminustrans) - \frac{1}{2} 1_n (\cplustrans + \cminustrans) \bigg \|_F^2   \\ 
 &&= \min_{y, \cplus, \cminus} \bigg \| X - \frac{(y+1_n)}{2} \cplustrans - \frac{(1_n-y)}{2}\cminustrans \bigg \|_F^2 \\ 
 &&= \min_{y, \cplus, \cminus} \| X \|_F^2  + \| \cplustrans \|_F^2 \bigg \|\frac{(y+1_n)}{2} \bigg \|^2 + \|\cminustrans\|_F^2 \bigg \|\frac{(1-y_n)}{2} \bigg \|^2 + 2\cminustrans \cplus \frac{(y+1_n)}{2}^\top \frac{(1_n-y)}{2}  \\
   &&  \quad \quad \quad - 2 \tr X^\top \bigg (\frac{(y+1_n)}{2} \cplustrans + \frac{(1_n-y)}{2}\cminustrans \bigg )  \\
 &&= \min_{y, \cplus, \cminus} \| X \|_F^2  + \| \cplustrans \|_F^2 \half (n + 1_n^\top y) + \|\cminustrans\|_F^2 \half (n - 1_n^\top y)  
      - 2 \cplustrans X^\top \bigg (\frac{y+1_n}{2} \bigg ) \\
     &&  \quad \quad \quad - 2 \cminustrans X^{\top}  \bigg (\frac{1_n-y}{2} \bigg ).  \nonumber   
\EEAS
Fixing $y$ and minimizing with respect to $\cplus$ and $\cminus$, we get closed-form expressions for $\cplus$ and $\cminus$ as 
\BEAS
 \cplus = \frac{X^\top (y+1_n)}{(n+1_n^\top y)} \quad \text{and} \quad \cminus = \frac{X^\top (1_n-y)}{(n-1_n^\top y)}. \nonumber
\EEAS
Substituting these expressions in $\calKM$, we have the following optimization problem in $y$:
\BEAS
&&\min_{y} \| X \|_F^2  - \half \frac{\| X^\top (y+1_n)\|_F^2}{(n + 1_n^\top y)} - \half \frac{\| X^\top (1_n-y)\|_F^2}{(n - 1_n^\top y)}  \\ 
&&=\min_{y} \| X \|_F^2 - \half \frac{ \tr XX^\top (y+1_n) (y+1_n)^\top}{(n + 1_n^\top y)} - \half \frac{ \tr XX^\top (1_n-y) (1_n-y)^\top}{(n - 1_n^\top y)}   \\
&&=\min_{y} \| X \|_F^2 - \frac{2}{(n + 1_n^\top y)} \tr XX^\top \bigg(\frac{y+1_n}{2}\bigg) \bigg(\frac{y+1_n}{2}\bigg)^\top \\
 &&  \quad \quad \quad -  \frac{2}{(n - 1_n^\top y)} \tr XX^\top\bigg(\frac{1_n-y}{2}\bigg)\bigg(\frac{1_n-y}{2}\bigg)^\top   \\
&&=\min_{y} \ \tr XX^\top - \frac{2}{(n + 1_n^\top y)} \tr XX^\top \bigg(\frac{y+1_n}{2}\bigg) \bigg(\frac{y+1_n}{2}\bigg)^\top\\
 &&  \quad \quad \quad  -  \frac{2}{(n - 1_n^\top y)} \tr XX^\top\bigg(\frac{1_n-y}{2}\bigg)\bigg(\frac{1_n-y}{2}\bigg)^\top   \\
&&=\min_{y} \ \tr XX^\top \bigg(I - \frac{1}{2(n + 1_n^\top y)} (yy^\top + 1_n1_n^\top + y1_n^\top + 1_ny^\top)\\
 &&  \quad \quad \quad  -  \frac{1}{2(n - 1_n^\top y)}(1_n1_n^\top + yy^\top - 1_ny^\top - y1_n^\top)\bigg). \nonumber 
\EEAS
By the centering of $X$, we have $1_n^\top X=0$ and hence $\tr XX^\top 1_n1_n^\top = \tr XX^\top 1_ny^\top = \tr XX^\top y1_n^\top = 0$. 
Therefore, we obtain 
\BEAS
&&\min_{y} \ \tr XX^\top \bigg(I - \frac{1}{2(n + 1_n^\top y)} (yy^\top) -  \frac{1}{2(n - 1_n^\top y)}(yy^\top)\bigg) \\ 
&& = \min_{y}\ \tr XX^\top \bigg(I - (yy^\top) \bigg(\frac{1}{2(n + 1_n^\top y)} +  \frac{1}{2(n - 1_n^\top y)}\bigg)\bigg) \\ 
&& = \min_{y}\ \tr XX^\top \bigg(I - (yy^\top) \bigg(\frac{n}{n^2 - (1_n^\top y)^2)} \bigg)\bigg)  \\
&& = \min_{y} \ \tr XX^\top \bigg(I - \frac{n yy^\top}{n^2 - (1_n^\top y)^2}\bigg). \nonumber  
\EEAS
Thus we have the equivalent $K$-means problem as 
$$
\min_{y \in \{-1,1\}^n} \frac{1}{n}\tr Xww^\top X^\top \Big( \idm - \frac{n}{n^2 - (y^\top 1)^2} yy^\top \Big)
= 1  - \max_{y \in \{-1,1\}^n}  \frac{  (w^\top X^\top y )^2}{n^2 - (y^\top 1)^2}.
$$
Thus the averaged distortion measure of $K$-means with the optimized means is $\frac{ (y^\top \Pi_n  X w)^2}{\| \Pi_n  y\|_2^2 }$. 

\section{Full (unsuccessful) relaxation}

\label{app:unsucces}
It is tempting to find a direct relaxation of \eq{maxcorr}. It turns out to lead to a trivial relaxation, which we outline in this section. 
When optimizing \eq{maxcorr} with respect to $w$, we obtain
$\displaystyle \max_{y \in \{-1,1\}^n}\textstyle  \frac{ y^\top X (X^\top X )^{-1} X^\top y}{ y^\top \Pi_n  y},
$
 leading to a quasi-convex relaxation as
 $ \displaystyle
\max_{\substack{Y \succcurlyeq 0, \\ \diag(Y) = 1}} \textstyle \frac{\tr Y  X (X^\top X )^{-1} X^\top }{ \tr \Pi_n  Y}. $
 Unfortunately, this relaxation always leads to trivial solutions as described below.

Consider the quasi-convex relaxation
\begin{equation}\label{eq:relaxqc}
\underset{Y\succcurlyeq0, \diag(Y)=1}{\max}\ \frac{\tr YX(X\t X)^{-1}X\t}{\tr \Pi_n Y}.
\end{equation}
By definition of $\Pi_n$ this relaxation is equal to:
$$
\underset{Y\succcurlyeq0, \diag(Y)=1}{\max}\ \frac{1}{n}\frac{\tr YX(X\t X)^{-1}X\t}{1-\frac{1_n\t Y1_n}{n^2}}.
$$
Let $\mathcal{A}=\{Y\succcurlyeq0, \ \diag(Y)=1\}$ the feasible set of this problem and define $\mathcal{B}=\{M\succcurlyeq0, \ \diag(M)=1+\frac{1_n\t M 1_n}{n^2}\}$. 
Let $Y\in\mathcal{A}$, then $M$ defined by $M=\frac{Y}{1-\frac{1_n\t Y1_n}{n^2}}$ belongs to $\mathcal{B}$ 
since $1+\frac{1_n\t M 1_n}{n^2} = 1+\frac{1_n\t Y 1_n}{n^2-1_n\t Y 1_n} = \frac{1}{1-\frac{1_n\t Y 1_n}{n^2}} = \diag(M)$.
Reciprocally for $M\in\mathcal{B}$, we can define $Y=\frac{M}{1+\frac{1_n\t M 1_n}{n^2}}$, such that $\diag(Y)=1$ and $Y\in \mathcal{A}$ and then verify that $M=\frac{Y}{1-\frac{1_n\t Y1_n}{n^2}}$.  
Thus the problem \eq{relaxqc} is equivalent to the relaxation
\begin{equation}\label{eq:relaxM}
\underset{M\succcurlyeq0, \diag(M)=1+\frac{1_n\t M 1_n}{n^2}}{\max}\ \frac{1}{n}\tr MX(X\t X)^{-1}X\t.
\end{equation}
The Lagrangian function of this problem can be written as:
\BEAS
L(\mu)&=&  \tr MX(X\t X)^{-1}X\t -\frac{\mu}{n}\t[\diag(M)-1_n-\frac{1_n\t M 1_n}{n^2}1_n]\\
&=& \tr M[X(X\t X)^{-1}X\t-\Diag(\mu)+\frac{1_n\t \mu}{n^2}1_n 1_n\t]+\frac{1}{n}\mu\t 1_n.
\EEAS
Using $L(\mu)$ and the PSD constraint $M\succcurlyeq0$, the dual problem is given by
$$
\underset{\mu}{\min}\ \frac{\mu\t 1_n}{n} \ \ \text{ s.t. } \Diag(\mu)-\frac{1_n\t \mu }{n^2}1_n1_n\t \succcurlyeq X(X\t X)^{-1}X\t.
$$
Since $X(X\t X)^{-1}X\t\succcurlyeq0$, this implies for the dual variable $\mu$:
\BEAS
\Diag(\mu)-\frac{1_n\t \mu}{n^2}1_n1_n\t\succcurlyeq0 &\Leftrightarrow& 1_n\t \Diag(\mu)^{-1}1_n \leq \frac{n^2}{\mu\t1_n}\\
&\Leftrightarrow& \sum_{i=1}^n \frac{1}{\mu_i}\leq \frac{n^2}{\sum_{i=1}^n \mu_i}\\
&\Leftrightarrow& \frac{1}{n}\sum_{i=1}^n\frac{1}{\mu_i}\leq \frac{1}{\frac{1}{n}\sum_{i=1}^n \mu_i}.
\EEAS
However for $\nu\in\RR^n$ ,  the harmonic mean $\big[\frac{1}{n}\sum_{i=1}^n \frac{1}{\nu_i}\big]^{-1}$ is always smaller than the arithmetic mean $\frac{1}{n}\sum_{i=1}^n \nu_i$ with equality if and only if $\nu=c 1_n$ for $c\in\RR$. 

Thus the dual variable $\mu$ is constant and the diagonal constraint simplifies itself as a trace constraint. 
Therefore the problem is equivalent to the trivial relaxation whose each eigenvector of $X(X\t X)^{-1}X\t$ is solution
$$
\underset{M\succcurlyeq0, \ \tr(M)=n+\frac{1_n\t M1_n}{n}}{\max } \ \tr MX(X\t X)^{-1}X\t.
$$

\section{Equivalent relaxation}
\subsection{First equivalent relaxation}\label{sec:relaxeq1}
We start from the penalized version of \eq{diffracyv}, 
\BEQ
\label{eq:diffracyvpen2}
\min_{ y \in \{-1,1\}^n, \ v \in \rb^d}   \frac{1}{n}  \| \Pi_n  y - X v \|_2^2+ \nu \frac{(y^\top 1_n)^2}{n^2},
\EEQ
which we expand as:
\BEQ
\label{eq:noncvx2}
\min_{ y \in \{-1,1\}^n, \ v \in \rb^d}    \frac{1}{n}   \tr \Pi_n  yy^\top 
-    \frac{2}{n} \tr X v y^\top +    \frac{1}{n} \tr X^\top X v v^\top
+ \nu \frac{(y^\top 1_n)^2}{n^2},
\EEQ
and relax as, using $Y=y y\t $, $P=y v\t$ and $V=v v \t$,
\BEQ
\label{eq:diffracfull2}
\min_{ V, P, Y }    \frac{1}{n} \tr \Pi_n  Y -    \frac{2}{n} \tr P^\top X +    \frac{1}{n} \tr X^\top X V + \nu \frac{1_n^\top Y 1_n }{n^2}
\mbox{ s.t.} \ \bigg( \!\! \begin{array}{cc} Y \!\! &\!\! P \\ P^\top\!\! &\!\!  V \end{array} \!\!\bigg) \succcurlyeq 0, \ 
\diag(Y)=1.
\EEQ
When optimizing \eq{diffracfull2} with respect to $V$ and $P$, we get exactly \eq{diffracnu}. Indeed we solve this problem by fixing the matrix $Y$ such that $Y=Y_0$ and $\diag(Y_0)=1_n$. Then the Lagrangian function of the problem in \eq{diffracfull2} can be written as 
\BEAS
L(A)&=& \frac{1}{n} \tr \Pi_n  Y -    \frac{2}{n} \tr P^\top X +    \frac{1}{n} \tr X^\top X V + \nu \frac{1_n^\top Y 1_n }{n^2}+\tr A(Y-Y_0) \\
&=&  \begin{pmatrix}
                 Y & P\\
                 P\t & V
                \end{pmatrix}
                \begin{pmatrix}
                 \frac{1}{n}\Pi_n +\frac{\nu}{n^2}1_n1_n\t+A & \frac{-1}{n}X\\
                 \frac{-1}{n}X\t & \frac{1}{n}X\t X
                \end{pmatrix} -\tr A Y_0.              
\EEAS
Using $L(A)$ and the psd constraint $\ \bigg( \!\! \begin{array}{cc} Y \!\! &\!\! P \\ P^\top\!\! &\!\!  V \end{array} \!\!\bigg) \succcurlyeq 0$, we write the dual problem as 
$$
\min_A \tr AY_0
\mbox{ s.t.} \begin{pmatrix}
                 \frac{1}{n}\Pi_n +\frac{\nu}{n^2}1_n1_n\t+A & \frac{-1}{n}X\\
                 \frac{-1}{n}X\t & \frac{1}{n}X\t X
                \end{pmatrix} \succcurlyeq 0.
$$
From the Schur's complement condition of $\begin{pmatrix}
                 \frac{1}{n}\Pi_n +\frac{\nu}{n^2}1_n1_n\t+A & \frac{-1}{n}X\\
                 \frac{-1}{n}X\t & \frac{1}{n}X\t X
                \end{pmatrix} \succcurlyeq 0$, we obtain $\frac{1}{n}\Pi_n +\frac{\nu}{n^2}1_n1_n\t+A \succcurlyeq \frac{1}{n}X(X\t X)^{-1}X\t$. Substituting the bound for $A$ we get the optimal objective function value
$$
\mathcal{D}^*=\frac{1}{n}\tr X(X\t X)^{-1} X\t Y_0-\frac{1}{n}\tr \Pi_n Y_0 -\frac{\nu}{n^2}1_n\t Y_0 1_n.
$$
Note that the optimal dual objective value $\mathcal{D}^*$ corresponds to a fixed $Y_0$. Hence by maximizing with respect to $Y$ we obtain exactly \eq{diffracnu} and therefore, the convex relaxation in \eq{diffracfull} is equivalent to \eq{diffracnu}. Moreover the Karush-Kuhn-Tucker (KKT) conditions gives
$$
P\t -X+V X\t X=0 \mbox{ and } -YX+PX\t X=0
$$
Thus the optimum is attained for 
$P =  Y X  (X^\top X)^{-1}$ and $ V =   (X^\top X)^{-1} X^\top Y X  (X^\top X)^{-1} $.

\subsection{Second equivalent relaxation}\label{sec:relaxeq2}

For $\nu = 1$,
 we solve the problem in \eq{diffracfull2} by fixing the matrix $V=V_0$. Then the Lagrangian function of this problem can be written as
\BEAS
\hat L(\mu,B)&=& \frac{1}{n} \tr \Pi_n  Y -    \frac{2}{n} \tr P^\top X +    \frac{1}{n} \tr X^\top X V + \nu \frac{1_n^\top Y 1_n }{n^2}+\mu \t(\diag(Y)-1_n) +\tr B(V-V_0) \\
&=&  \begin{pmatrix}
                 Y & P\\
                 P\t & V
                \end{pmatrix}
                \begin{pmatrix}
                 \frac{1}{n}I_n+\diag(\mu) & \frac{-1}{n}X\\
                 \frac{-1}{n}X\t & \frac{1}{n}X\t X +B
                \end{pmatrix} -\mu \t 1_n-\tr B V_0.              
\EEAS
Using $\hat L(\mu,B)$ and the psd constraint  $\ \bigg( \!\! \begin{array}{cc} Y \!\! &\!\! P \\ P^\top\!\! &\!\!  V \end{array} \!\!\bigg) \succcurlyeq 0$, the dual problem is given by
$$
\min_{\mu,B} \mu \t 1_n+\tr BV_0
\mbox{ s.t.} \begin{pmatrix}
                 \frac{1}{n}I_n+\diag(\mu) & \frac{-1}{n}X\\
                 \frac{-1}{n}X\t & \frac{1}{n}X\t X +B
                \end{pmatrix} \succcurlyeq 0.
$$
From the Schur's complement condition of $\begin{pmatrix}
                 \frac{1}{n}I_n+\diag(\mu)& \frac{-1}{n}X\\
                 \frac{-1}{n}X\t & \frac{1}{n}X\t X +B
                \end{pmatrix} \succcurlyeq 0$, we obtain $B\succcurlyeq \frac{1}{n^2}X\t\diag(\mu+1_n/n)^{-1}X-\frac{1}{n}X\t X$. Substituting the bound for $B$ we get the dual problem as
\BEAS
& & \min_{\mu} \mu \t 1_n+\frac{1}{n^2}\tr V_0 X\t\diag(\mu+1_n/n)^{-1} X-\frac{1}{n}\tr V_0X\t X\\
& &\min_{\mu} \sum_{i=1}^n \Bigg( \mu_i +\frac{1}{n^2\mu_i+n} x_i\t V_0x_i\Bigg)-\frac{1}{n}\tr V_0 X\t X.
\EEAS 
Solving for $\mu_i$, we get 
$$
\mu_i^*=\frac{1}{n}\sqrt{x_i\t V_0 x_i} -\frac{1}{n}.
$$
Substituting $\mu_i^*$ into the dual obkective function, we get the optimal objective function value
$$
\hat D=\frac{2}{n}\sum_{i=1}^n \sqrt{ ( XVX^\top)_{ii} }-1-\frac{1}{n}\tr V_0 X\t X.
$$
Furthermore the KKT conditions gives
$$
Y\diag(\nu+1_{n}/n) -\frac{1}{n}PX\t=0 \mbox{ and } P\t \diag(\nu+1_{n}/n) -\frac{1}{n}V X\t=0.
$$
Thus we obtain the following closed form expressions:  
\BEAS
P & = & \Diag(\diag(X V X ^\top))^{-1/2} X V \\
Y & = & \Diag(\diag(X V X ^\top))^{-1/2} X V X ^\top \Diag(\diag(X V X ^\top))^{-1/2}.
\EEAS
The optimal dual objective value $\hat D$ corresponds to a fixed $V_0$. Therefore, maximizing with respect to $V$ leads to the problem: 
\BEQ
\min_{V \succcurlyeq 0} \ \ 
1 -    \frac{2}{n} \sum_{i=1}^n \sqrt{ ( XVX^\top)_{ii} }  +    \frac{1}{n} \tr ( V X^\top X).
\EEQ
\section{Auxilliary results for \mysec{theoanalnonsparse}}
\label{app:nnsparse}
\subsection{Auxilliary lemma}

The matrix $X(X\t X)^{-1} X\t$  has the following properties (see e.g. \citep{freedman}).
\begin{lemma}\label{lemma:hatmatrix}
The matrix $H=X(X\t X)^{-1} X\t$ is the orthogonal projection onto the column space of the design matrix X since:
\begin{itemize}
 \item $H$ is symmetric.
 \item $H$ is idempotent $(H^2)=H$.
 \item $X$ is invariant under $H$, that is $HX=X$.
\end{itemize}
\end{lemma}

\subsection{Rank-one solution of the relaxation \eq{diffracnu} }\label{sec:roy}

We denote by $(x_i)_{i=1...n}$ the lines of $X$.
\begin{lemma}\label{lemma:ysol}
The rank-one solution $Y_*=yy\t$ is always solution of the relaxation \eq{diffracnu}.
\end{lemma}
\begin{proof}
We give an elementary proof of this result without using convex optimization tools. 
Using lemma \ref{lemma:hatmatrix} we have $Hy=y$, thus
$$
\tr HY_*= \tr Hyy^{\top} = \tr y y\t =n.
$$
Moreover all $M\succcurlyeq0$ can always be decomposed as $\sum_{i=1}^n \lambda_i u_i u_i^{\top}$ with $\lambda_i\geq0$ and $(u_i)_{i=1,...,n}$ an orthonormal familly. Since $H$ is an orthogonal projection $(u_i)^{\top}H u_i =(Hu_i)^{\top}Hu_i=\Vert Hu_i\Vert^2\leq \Vert u_i\Vert^2\leq 1$.
Thus $\tr HM=\sum_{i=1}^n \lambda_i \tr H u_i (u_i)^{\top}=\sum_{i=1}^n \lambda_i (u_i)^{\top} H u_i\leq \sum_{i=1}^n \lambda_i = \tr M $.

Then for all matrix $M$ feasible we have   $\tr H M \leq n$  since $\diag(M)=1_{n}$ and $\tr H Y_*=n$ which conclude the lemma. 
\end{proof}

\subsection{Rank-one solution of the relaxation \eq{relaxW2} }\label{sec:row}
\begin{lemma}\label{lemma:wsol}
 The rank-one solution $V_*=vv\t$ is always solution of the relaxation \eq{relaxW2}.
\end{lemma}
\begin{proof}
 The Karush-Kuhn-Tucker (KKT) optimality conditions for the problem are for the dual variable $A\preccurlyeq 0$:
 $$
\frac{1}{n}\sum_{i=1}^n \frac{x_i x_i^{\top}}{\sqrt{x_i^{\top}Vx_i}}-\frac{1}{n}XX^{\top}= A \text{ and } AV=0 \ \text{ (Complementary Slackness)}.
$$
Since $x_i\t w=y_i$, $\sqrt{x_i^{\top}V_*x_i}=\vert y_i\vert=1$, $V_*$ and  the dual variable  $A=0$ satisfy the KKT conditions and then $V_*$ is solution of this problem.
\end{proof}

\subsection{Proof of Proposition \ref{prop:unicitew}}

In  the following lemma, we use a Taylor expansion to lower-bound $f$ around its minimum. 
\begin{lemma}\label{lem:unicitew}
For $d\geq3$ and   $\delta\in[0,1)$. 

If $\beta \geq 3$ and  $m^2\leq \frac{\beta-3}{2(d+\beta-4)}$, then  with probability  at least
$ 1-d \exp\big(-\frac{ \delta^2n m^2 }{ 2 R^4d^2}\big)$, for any symmetric matrix $\Delta$:
$$
 f(V_*)-f(V_*+\Delta)>  2(1-\delta)m^2  \Vert \Delta \Vert_F^2+o(\Vert \Delta \Vert^2)\geq0.
$$

Otherwise  with probability  at least
$ 1-d \exp\big(-\frac{ \delta^2n \mu_1}{ 4R^4d^2}\big)$, for any symmetric matrix $\Delta$:
$$
 f(V_*)-f(V_*+\Delta)>  (1-\delta) \mu_1 \Vert \Delta \Vert_F^2+o(\Vert \Delta \Vert^2)\geq0,
$$
with $\mu_1\geq \frac {m^2(\beta-1)}{1+(d+\beta-2)m^{2}}$.
Moreover we also have with probability at least $ 1-d \exp\big(-\frac{ \delta^2n \mu_2}{ 4R^4d^2} \big)$, for any symmetric matrix $\Delta\in \Delta_{\min}^\perp$: 
$$
 f(V_*)-f(V_*+\Delta)>  (1-\delta)\mu_2\Vert \Delta \Vert_F^2+o(\Vert \Delta \Vert^2)\geq0,
$$
where $\mu_2= \min\{2m^2,m^2(\beta-1),2m \}$ and  $\Delta_{\min}=\begin{pmatrix}
                                               1&0\\0&c_{\min}I_{d-1}
                                              \end{pmatrix}$ is defined in the proof and satisfies
                                               \[
                                            \vert c_{\min}\vert \leq  \frac{m}{\vert (d+\beta-2)m^{2}-1\vert }.
                                              \]

\end{lemma}
This lemma directly implies Proposition \ref{prop:unicitew}.

\begin{proof}

For $\Delta\in\mathcal{S}(d)$ and $\delta\in\RR$ we compute for $f(V)=\frac{1}{n}\sum_{i=1}^n\sqrt{x_i^{\top}Vx_i}$,
$$
\frac{d^2}{d\delta^2}f(V+\delta \Delta)=-\frac{1}{4n}\sum_{i=1}^n\frac{(x_i\t\Delta x_i)^2}{\sqrt{x_i^{\top}(V+\delta\Delta)x_i}^3}.
$$
Thus the second directional derivative in $V=V_*$ along $\Delta$ is
$$
\nabla^2_\Delta f(V_*)=\lim_{\delta \to 0}\frac{d^2}{d\delta^2}f(V+\delta \Delta)=-\frac{1}{4n}\sum_{i=1}^n{(x_i\t\Delta x_i)^2}.
$$
Let  $\mathcal T_x$ be the semidefinite positive quadratic form  of $\mathcal S(d)$ defined for $\Delta\in \mathcal S(d) $,  by 
\begin{equation}\label{eq:defT}
\mathcal T_x:\Delta\mapsto (x^\top \Delta x)^2. 
\end{equation}
Then it exists a positive linear operator $T_x$ from $ \mathcal S(d)$ to $ \mathcal S(d)$  such that $\mathcal T_x(\Delta)=\langle \Delta, T_x \Delta\rangle$. 

Therefore the function $f$ will be stricly concave if for all directions $\Delta\in\mathcal{S}(d)$
\begin{equation}\label{eq:sum}
 \frac{1}{n}\sum_{i=1}^n\mathcal T_{x_i}(\Delta)>0.
\end{equation}

We will bound the empirical expectation in \eq{sum} by first showing that its expectation remains away from $0$. Then we will  use a concentration inequality for matrices to control the distance between the sum in \eq{sum} and its expectation.
 
 We first derive conditions so that the result is true in expectation, i.e. for the operator $\mathcal{T}$ defined by  $\mathcal{T}=\E \mathcal{T}_x$ for $x$ following the same law as $(y,z\t)\t$. We denote by $m=\E z^{2}$ and by $\beta=\E z^{4}/{m^2}$ its kurtosis.

We let $\Delta=\begin{pmatrix}
                                               a&b\t\\b&C
                                              \end{pmatrix}$ and then have $x\t\Delta x=a+2y b\t z +z\t C z$. Thus
 $$
  \mathcal{T}_x(\Delta)=a^2+4ay b\t z +2a z\t C z+4b\t(zz\t)b+(z\t C z)^2+4yb\t z(z\t C z).
 $$
 Therefore we can express the value of the operator $\mathcal{T}$ only in function of the elements of $\Delta$:
 $$ \mathcal{T}(\Delta)=(a+m\Tr C)^2+4m \Vert b \Vert_2^2+2m^2\Vert C-\Diag(\diag(C)) \Vert_F^2+m^2(\beta-1)\Vert \diag(C)\Vert^2,
 $$
where we have used
 \BEAS
 \E (z\t C z)^2&=&\E \sum_{i,j,k,l}z_iz_jz_kz_lc_{i,j}c_{k,l}\\
 &=&\E\sum_{i}(z_i)^4c_{i,i}^2+\E \sum_{i,k\neq i}z_i^2z_k^2c_{i,i}c_{k,k}+2\E \sum_{i,j\neq i}z_i^2z_j^2c_{i,j}^2\\
 &=&\beta m^2\sum_{i}c_{i,i}^2+m^2 \sum_{i,k\neq i}c_{i,i}c_{k,k}+2m^2\sum_{i,j\neq i}c_{i,j}^2\\ 
 &=&m^2(\beta-3)\sum_{i}c_{i,i}^2+m^2 \sum_{i,k}c_{i,i}c_{k,k}+2m^2\sum_{i,j}c_{i,j}^2\\
 &=&m^2(\beta-3)\Vert \diag(C)\Vert^2+m^2\big(2\Vert C\Vert_F^2+\tr(C)^2\big)\\
 &=&m^2(\beta-3)\Vert \diag(C)\Vert^2+m^2\big(2\Vert C-\Diag(\diag(C))\Vert_F^2+\tr(C)^2\big).
 \EEAS
 Since $\beta\geq 1$, we get 
 $$
\mathcal{T}(\Delta)\geq (a+m\Tr C)^2+4m \Vert b \Vert_2 ^2+2m^2(\Vert C\Vert_F^2-\Vert \diag(C)\Vert^2).
 $$
 Thus $ \mathcal{T}(\Delta)=0$ if and only if $\beta=1$ with $b=0_{d-1}$ and $C=\diag(c)$ with $c\t 1_d=-\frac{a}{m_2}$. With the condition $\beta=1$ meaning that $\Var(z^2)=0$ and thus $z^2$ is constant a.s., i.e. $z$ follows a Rademacher law. 
 
 However we would like to bound  $ \mathcal{T}(\Delta)$ away from zero by some constant and for that we are looking for the smallest eigenvalue of the operator $\E T_x$. Unfortunately we are not able to solve the optimization problem 
 \[
 \min_{\Delta\in \mathcal{S}(d),\Vert \Delta \Vert_F^2=1} \mathcal T(\Delta),
 \]
 and we have to compute all the spectrum of this operator to be able to find the smallest using $\E T_x \Delta=1/2\mathcal \nabla \mathcal T(\Delta)$ .
 
 We have  
 \[
 1/2\mathcal \nabla \mathcal T(\Delta)= \begin{pmatrix}
                                               a + m \tr(C) & 2m b\t \\ 2mb & (a + m\tr (C))m_{2}I_{d-1} + 2m^2 C \\ & + m^2(\beta-3)\Diag(\diag(C))
                                              \end{pmatrix}.
\]

 \begin{itemize}
 \item
 For all $b\in\RR^{d-1}$ we have for $\Delta=\begin{pmatrix}
                                               0&b\t\\b&0
                                              \end{pmatrix}$,  $ 1/2\nabla \mathcal T(\Delta)=2m \Delta $. Thus $2m$ is an eigenvalue of multiplicity $d-1$. 
 \item
 For all $C\in\RR^{(d-1)\times (d-1)}$ with $\diag(C)=0_{d-1}$ we have for $\Delta=\begin{pmatrix}
                                               0&0\\0&C
                                              \end{pmatrix}$,  $1/2\nabla \mathcal T(\Delta)=2m^{2} \Delta $. Thus $2m^{2}$ is an eigenvalue of multiplicity $\frac{(d-1)(d-2)}{2}$.                                             
 \item
  For all $c\in\RR^{d-1}$ with $c\t 1_{d-1}=0$ we have for $\Delta=\begin{pmatrix}
                                               0&0\\0&\diag(C)
                                              \end{pmatrix}$,  $1/2 \nabla \mathcal T(\Delta)=m^2(\beta-1) \Delta$. Thus $m^2(\beta-1)$ is an eigenvalue of multiplicity $d-2$. 
 
 \item                                              
 For all $a,c\in\RR^{2}$ we have for $\Delta=\begin{pmatrix}
                                               a&0\\0&cI_{d-1}
                                              \end{pmatrix}$,   
\BEAS       
1/2\nabla \mathcal T(\Delta)&=& \begin{pmatrix}
                                               a + m (d-1)c & 0 \\ 0 & [m a + m^{2}(d+\beta -2)c] I_{d-1} \end{pmatrix} \\
                    &=&  \Diag \Big[\begin{pmatrix} 1& m 1_{d-1}\t  \\ m 1_{d-1} &(d+\beta -2) m^{2}I_{d-1}\end{pmatrix}   \begin{pmatrix}a \\ c 1_{d-1}\end{pmatrix}  \Big].
\EEAS                    
Thus an eigenvalue of $\begin{pmatrix} 1&&(d-1)m\\ m&&(d+\beta-2)m^{2} \end{pmatrix}$ with an eigenvector $[a,c]^\top$  would be an eigenvalue  of the operator $\E T_x$ with a corresponding eigenvector $\begin{pmatrix}
                                               a&0\\0&cI_{d-1}
                                              \end{pmatrix}$.  This matrix has two simple eigenvalues 
\BEQ
                                               \mu_{\pm}=\frac{1+(d+\beta -2)m^{2}\pm\sqrt{(1+(d+\beta -2)m^{2})^{2}-4m^2(\beta-1)}}{2}.                                
\EEQ
\end{itemize}
 Moreover when we add all the multiplicity of the found eigenvalues we get $d-1+\frac{(d-1)(d-2)}{2}+d-2+2=\frac{d(d+1)}{2}$ which is the dimension of $S(d)$, therefore we have found all the eigenvalues of the linear operator $\E T_x$. 
 
 We will prove now than the smallest eigenvalue is $\mu_-$ when the dimension $d$ is large enough with regards to $m^2$ and  $2m^2$ otherwise.
 \begin{lemma}
  Let $\mu_1$ and $\mu_2$ be the two smallest eigenvalues of the operator $\E T_x$. Let us assume that $d\geq 3$ (the case $d=2$ will also be done in the proof).
  
  If $\beta \geq 3$ and  $m^2\leq \frac{\beta-3}{2(d+\beta-4)}$ then 
  \[
   \mu_1=2m^2.
  \]
Otherwise 
  \[
   \mu_1=\mu_-  \geq\frac {m^2(\beta-1)}{1+(d+\beta-2)m^{2}} \text{ and } \mu_2= \min\{2m^2,m^2(\beta-1),2m \}.
  \]
  
  Moreover we denote by $\Delta_{\min}=\begin{pmatrix}
                                               1&0\\0&c_{\min}I_{d-1}
                                              \end{pmatrix}$  the  eigenvector associated to $\mu_-$ for which  we have  set without loss of generality the first component  $a=1$. Then 
                                               \[
                                            \vert c_{\min}\vert \leq  \frac{m}{\vert (d+\beta-2)m^{2}-1\vert }.
                                              \]
 \end{lemma}

  Unfortunately $\mu_-$ can become small when the dimension increases as explained by the tight bound $\mu_-\geq \frac {m^2(\beta-1)}{1+(d+\beta-2)m^{2}}$. However the corresponding eigenvector  have a particular structure we will be able to exploit. 
 \begin{proof}
 First we note that $\mu_{-}\leq m^2(\beta-1)$ and compute 
\BEAS
\mu_-\geq 2m^2 &\Leftrightarrow&1+(d+\beta -2)m^{2}-\sqrt{(1+(d+\beta -2)m^{2})^{2}-4m^2(\beta-1)}-4m^2\geq 0\\
 &\Leftrightarrow&1+(d+\beta -2)m^{2}-4m^2\geq \sqrt{(1+(d+\beta -2)m^{2})^{2}-4m^2(\beta-1)} \\
  &\Leftrightarrow&(1+(d+\beta -2)m^{2}-4m^2)^2\geq {(1+(d+\beta -2)m^{2})^{2}-4m^2(\beta-1)} \\
  &&\text{and }  1+(d+\beta -6)m^{2}\geq0\\
    &\Leftrightarrow&16m^4-8m^2(1+(d+\beta -2)m^{2})\geq -4m^2(\beta-1) \\
  &&\text{and }  1+(d+\beta -6)m^{2}\geq0\\
      &\Leftrightarrow&2(d+\beta-4)m^2\leq \beta-3 \text{ and }  1+(d+\beta -6)m^{2}\geq0. 
\EEAS
\begin{itemize}
\item 
If $d=2$, 
\begin{itemize}
 \item If $\beta\leq 3$ we have necessary that $\beta\leq 2 $ and the first equation gives $m^2\geq \frac{3-\beta}{2(2-\beta)}$ and the second $m^2\leq 1/(4-\beta)$. Thus we should have $(4-\beta)(3-\beta)\leq 2(2-\beta)$ which is not possible since the polynomial $\beta^2-5\beta+8\geq 0$. 
 \item
 If $\beta\geq 3$, the first equation gives $m^2\leq \frac{\beta-3}{2(\beta-2)}\leq 1$ and the second $m^2\leq 1/(4-\beta)\leq \frac{\beta-3}{2(\beta-2)}\leq1$ for $\beta\leq 4$ and is always satisfied otherwise. 
\end{itemize}
 \item 
 If $d\geq3$, the first equation implies that $\beta\geq 3$ for which the second equation is always satisfied. It also implies that $m^2\leq \frac{\beta-3}{2(d+\beta-4)}\leq 1$.
\end{itemize}

We denote by $\Delta_{\min}=\begin{pmatrix}
                                               1&0\\0&c_{\min}I_{d-1}
                                              \end{pmatrix}$  the  eigenvector for which  we have  set without loss of generality  $a=1$ and
                                             \[
                                             c_{\min}= \frac{-1}{2(d-1)m}\Big[\sqrt{((d+\beta -2)m^{2}-1)^2+4(d-1)m^2}-(d+\beta-2)m^{2}+1\Big].
                                              \]
Consequently $ c_{\min}\leq 0$ and by convexity of the square root we have $\sqrt{((d+\beta-2)m^{2}-1)^2+4(d-1)m^2}\leq ((d+\beta-2)m^{2}-1)+\frac{2(d-1)m^2}{\vert (d+\beta-2)m^{2}-1\vert}$. Therefore   
 \[
                                            \vert c_{\min}\vert \leq  \frac{m}{\vert (d+\beta-2)m^{2}-1\vert }.
                                              \]
\end{proof}
We will control now the behavior of  the empirical expection by its expectation thanks to concentration theory. By definition $T_x$ is a symmetric positive linear operator as its projection $T^{\perp}_x$ onto the orthogonal space of $\Delta_{\min}$.  We can thus apply the Matrix Chernoff inequality from \citet[Theorem 5.1.1]{tropp} to these two operators using
$\Vert T_x\Vert_{op}\leq \Vert x x\t\Vert^2\leq \tr(xx\t)^2\leq \Vert x\Vert_2^4\leq R^4 d^2$
Then:
 $$
 \P\Bigg(\lambda_{\min}\Big(\sum_{k=1} T_{x_k}\Big)\leq n \delta \mu_1\Bigg)\leq d \Big[\frac{e^{-(1-\delta)}}{\delta^{\delta}}\Big]^{n\mu_1/(2R^4d^2)}\leq d e^{-(1-\delta)^2n\mu_1/(4R^4d^2)},
 $$
  $$
 \P\Bigg(\lambda_{\min}\Big(\sum_{k=1} T^{\perp}_{x_k}\Big)\leq n \delta \mu_2\Bigg)\leq d \Big[\frac{e^{-(1-\delta)}}{\delta^{\delta}}\Big]^{n\mu_2/(2R^4d^2)}\leq d e^{-(1-\delta)^2n\mu_2/(4R^4d^2)},
 $$

For $m=1$ and $d\geq 3$ we have
$
\mu_1=\mu_-\geq  \frac {\beta-1}{\beta+d} 
\geq  \min\{ \frac{\beta-1}{2\beta} ,  \frac{\beta-1}{2d} \}\geq  \min\{ 1/3 ,  \frac{\beta-1}{2d} \}.
$
\end{proof}

\subsection{Noise robustness for the $1$-dimensional balanced problem}\label{sec:noisey}
We want a condition on $\varepsilon$ such that the solution of the relaxation recovers the right $y$.
We recall the dual problem of the relaxation \eq{diffracnu}
$$
\min \mu \t 1_n \text{ s.t. } \Diag(\mu)\succcurlyeq X(X\t X)^{-1}X\t.
$$
The KKT conditions are: 
\begin{itemize}
 \item Dual feasibility: $\Diag(\mu)\succcurlyeq X(X\t X)^{-1}X\t$.
 \item Primal feasibility: $\Diag(Y)=1_n$ and $Y\succcurlyeq 0$.
 \item Complimentary slackness : $ Y[\Diag(\mu)-X(X\t X)^{-1}X\t]=0$
\end{itemize}
For $Y=yy\t$ a rank one matrix, the last condition implies $\Diag(\mu)y=Hy$ and
$$
\mu_i=\frac{(X(X\t X)^{-1}X\t y)_i}{y_i}.
$$
For $X=y+\varepsilon$, we denote by $\tilde y=y+\varepsilon$, then $X(X\t X)^{-1}X\t=\frac{\tilde y \tilde y \t}{\Vert \tilde y \Vert ^2}$ and $X(X\t X)^{-1}X\t y=\frac{\tilde y \t y}{\Vert \tilde y \Vert ^2}\tilde y $. Thus
$$\mu_i=\frac{\tilde y \t y}{\Vert \tilde y \Vert ^2}\frac{\tilde y_i }{y_i}.$$
Assume that all $\tilde y_i y_i $  have the same sign, without loss of generality we assume $\tilde y_i y_i >0$. By definition of $\mu$, $\mu\geq0$.
To show the dual feasibility we have to show that $\Diag(\mu)\succcurlyeq H$ which is equivalent to $ \Diag(\frac{\tilde y_i}{y_i})\succcurlyeq \frac{\tilde y \tilde y \t}{\tilde y \t y}$, to 
$I_n-\Diag(\sqrt{\frac{y_i}{\tilde y_i}})\frac{\tilde y \tilde y \t}{\tilde y \t y}\Diag(\sqrt{\frac{y_i}{\tilde y_i}})\succcurlyeq 0$ and to $\sum {y_i \tilde y_i}\leq \tilde y \t y$ which is obviously true. 
Reciprocally if $\mu$ is dual feasible then $\Diag(\mu)\succcurlyeq0$ and all the  $\tilde y_i y_i $ have the same sign. 

Therefore we have shown that $y$ is solution of the relaxation \eq{diffracnu} if and only if all the $\tilde y_i y_i $ have the same sign. If $\varepsilon$ and $y$ are independent this is equivalent to $\Vert \varepsilon \Vert_{\infty} \leq1$ a.s. 

 \subsection{The rank-one candidates are not solutions of the relaxation}\label{sec:nomore}
 We assume now that $1_n\t y\neq0$ thus $y\neq \Pi_n y$, which means we do not have the same proportion in the two clusters. Let us assume that $\Pi_n y$ takes two values $\{ \pi y_-,\pi y_+\}$ 
 that is by definition of $\Pi_n$ $\pi y_+=1-\frac{1_n\t y}{n}$ and $\pi y_-=-1-\frac{1_n\t y}{n}$ .
For $V_*$ defined as before, we get $x_i\t V_*x_i=(\pi y_i)^2$
and with $I_\pm$ the set of indices such that $\Pi_n y_i=\pi y_\pm$, the KKT conditions for $V=V_*$ can be written as
$$
\frac{1}{n}\Big[\sum_{i\in I_+}\Big(\frac{1}{\pi y_+}-1\Big)x_ix_i\t+\sum_{i\in I_-}\Big(\frac{1}{-\pi y_-}-1\Big)x_ix_i\t\Big]=A_n \preccurlyeq 0 \text{ and } A_nV_*=0.
$$
We check that with $n_\pm=\#\{ I_\pm\}$:
\BEAS
 w\t A_nw =0&=&\sum_{i\in I_+}\Big(\frac{1}{\pi y_+}-1\Big)(\pi y_+)^2+\sum_{i\in I_-}\Big(\frac{1}{-\pi y_-}-1\Big)(\pi y_-)^2\\
&=& n_+\Big(\frac{1}{\pi y_+}-1\Big)(\pi y_+)^2+n_-\Big(\frac{1}{-\pi y_-}-1\Big)(\pi y_-)^2\\
&=&n_+\pi y_+-n_-\pi y_--\big(n_+(\pi y_+)^2+n_-(\pi y_-)^2\big)\\
&=&{y\t \Pi_n y}-{(\Pi_n y)\t\Pi_n y}=y\t \Pi_n y-y\t \Pi_n y=0.
\EEAS
And $A_n=\frac{1}{2n}\big[\sum_{i\in I_+}\alpha_+x_ix_i\t+\sum_{i\in I_-}\alpha_-x_ix_i\t\big]$ with $\alpha_+=\big(\frac{1}{\pi y_+}-1\big)$ and $\alpha_-=\big(\frac{1}{-\pi y_-}-1\big)$.
Unfortunately $\alpha_+\alpha_-\leq0$, and $A_n$ is not necessary negative.
Even worse we will show that $\E A$ is not semi-definite negative which will conclude the proof since by the law of large number $\underset{n\to\infty}{\lim}\frac{1}{n}A_n =\E A$.      
Assume that the proportions of the two clusters stay constant with $n_\pm=\rho_\pm n $, then 
$$
\E A= \rho_+\alpha_+ \begin{pmatrix} (\pi y_+)^2 & 0\\ 0 & I\end{pmatrix}+\rho_-\alpha_- \begin{pmatrix} (\pi y_-)^2 & 0\\ 0 & I\end{pmatrix}.
$$
And $\rho_+\alpha_+(\pi y_+)^2+\rho_-\alpha_- (\pi y_-)^2=0$ since $w\t A_n w=0$. Then
\BEAS
\rho_+\alpha_++\rho_-\alpha_- &=& \frac{\rho_+\pi y_--\rho_-\pi y_+-\pi y_+ \pi y_-}{\pi y_+ \pi y_-}\\
&=&\frac{-(\rho_++\rho_-)-\frac{1_n\t y}{n}(\rho_+-\rho_-)+(1-(1_n\t y)^2)}{-(1-(\frac{1_n\t y}{n})^2)}\\
&=&\frac{\frac{1_n\t y}{n}(\rho_+-\rho_-)+(\frac{1_n\t y}{n})^2)}{(1-(\frac{1_n\t y}{n})^2)}=\frac{2(\frac{1_n\t y}{n})^2}{(1-(\frac{1_n\t y}{n})^2)}\geq0.
\EEAS
Thus 
$A=\frac{2(1_n\t y)^2}{(n^2-(1_n\t y)^2)}\begin{pmatrix} 0 & 0\\ 0 & I\end{pmatrix}$ is not semi-definite negative and $V_*$ is not solution of the relaxation \eq{relaxW2}.

\section{Auxilliary results for sparse extension }\label{app:sparse}

\subsection{There is a rank-one solution of the relaxation \eq{relaxYY}}\label{app:solws}

\begin{lemma}\label{lemma:wsols}
 The rank-one solution $V_*=v^*{v^*}\t$ is solution of the relaxation \eq{relaxYY} if the design matrix $X$ is such that $\frac{1}{n}X\t X$ has all its diagonal entries less than one.
\end{lemma}

\begin{proof}
The KKT conditions are
$$
\frac{1}{n}\sum_{i=1}^n \frac{x_i x_i\t}{\sqrt{x_i\t W x_i}}- \lambda U- \frac{1}{n}X\t X =A\preccurlyeq 0  \text{ and } AW=0,
$$
with $U$ such that $U_{ij}=\sign(W_{ij})$ if $W_{ij}\neq0$ and $U_{ij}\in [-1,1]$ otherwise. 
For $V_*=v^*{v^*}\t$ this gives
$$
A=\frac{(1+\lambda)}{n}X\t X- \lambda U- \frac{1}{n}X\t X=\lambda \Big[\frac{X\t X}{n}-U\Big] \text{ with } U_{1,1}=1 \text{ and } U_{i,j}\in[-11] \text{ otherwise}.
$$
We check that $AV_*=0$. If the design matrix $X$ satsifies assumption (A1), we can choose a sub-gradient $U$ such that the dual variable  $A=0$ and thus $V_*$ is solution. 
Otherwise by property of semi-definite matrices, there is a diagonal entry of $\frac{1}{n}X\t X$ which is bigger than $1$ which prevents $A$ to be semi-definite negative since the corresponding diagonal entry of $\frac{X\t X}{n}-U$ will be positive.
This shows that $V_*$ does not solve the problem. 
\end{proof}

\subsection{Proof of proposition \ref{prop:sparse}}

\begin{lemma}\label{lemma:sparse}
For $\delta\in[0,1)$, with probability $1-5d^2\exp\big(-\frac{\delta^2 n(\beta-1)}{2dR^4(1/m^2+\beta+d)}\big)$, for any direction $\Delta$ such that $V_*+\Delta\succcurlyeq0$, we have:
$$
 g(V_*)-g(V_*+\Delta)>  (1-\delta) \Big[\lambda\Vert \Delta-\Diag(\Delta) \Vert_1+\frac{\beta-1}{\beta+d+1/m^2}\frac{(1+\lambda)^3}{4}\Vert \Diag(\Delta) \Vert_2^2\Big]+o(\Vert \Delta \Vert^2)\geq 0.
$$
Moreover we also have with probability at least $1-5d^2\exp\big(-\frac{\delta^2 nm^2(\beta-1)}{2dR^4}\big)$, for any symmetric matrix $\Delta$ such that $V_*+\Delta\succcurlyeq0$ and $\Diag(\Delta)\in (e_{\min})^\perp$: 
$$
 g(V_*)-g(V_*+\Delta)>  (1-\delta) \Big[\lambda\Vert \Delta-\Diag(\Delta) \Vert_1+m^2(\beta-1)\frac{(1+\lambda)^3}{4}\Vert \Diag(\Delta) \Vert_2^2\Big]+o(\Vert \Delta \Vert^2)\geq 0.
$$
where $e_{\min}=[
                                               1,c_{\min}1_{d-1}]$
                                              is defined in the proof and satisfies
                                              
                                            \[
                                            \vert c_{\min}\vert \leq  \frac{m}{\vert (d+\beta-2)m^{2}-1\vert }.
                                              \]

\end{lemma}

\subsubsection{Proof outline}
We will investigate under which conditions on $X$ the solution is unique, first for a deterministic design matrix.
 We make the following deterministic assumptions on  $X$ for $\delta,\zeta\geq0$ and $\mathcal S \subset \mathcal \RR^d$:
 
 \vspace{.5cm}
 
 \begin{tabular}{clcl}
&\textbf{(A1)} $\Vert \frac{X\t X}{n}\Vert_{\infty}\leq1$&\ & \textbf{(A3)} $\Vert \frac{ Z\t Z}{n} -\Diag(\diag(\frac{1}{n} Z \t Z))\Vert_{\infty}\leq \delta$\\
&\textbf{(A2)} $\Vert \frac{Z\t y}{n}  \Vert_{\infty}\leq \delta$&\ \ & \textbf{(A4)} $ \lambda^{\mathcal S}_{\min}\big( \frac{X^{\odot 2}(X^{\odot 2})\t}{n}\big)\geq\zeta>0$. \\
 \end{tabular} 
 
 \vspace{.5cm}
Where we denoted by $\odot$ the Hadamard (i.e., pointwise) product  between matrices and $ \lambda^{\mathcal S}_{\min}$ the minimum eigenvalue of a linear operator restricted to a subspace $\mathcal S$.
 Then with $g(V)=\frac{2}{n}\sum_{i=1}^n\sqrt{x_i\t V x_i}- \lambda \Vert V\Vert_1 -\frac{1}{n}\tr X\t X V$, we can certify that $g$ will decrease around the solution $V_*$. 
 \begin{lemma}\label{lemma:sparsedet}
 Let us assume that the noise matrix verifies assumption (A1,A2,A3,A4), then for all direction $\Delta$ such that $V_*+\Delta\succcurlyeq0$ and $\diag(\Delta)\in \mathcal S$ we have:
 $$
  g(V_*)-g(V_*+\Delta)\geq \lambda (1-\delta) \Vert \Delta-\Diag(\diag(\Delta)) \Vert_1+\zeta\frac{(1+\lambda)^3}{4}\Vert \Diag(\Delta) \Vert_2^2+o(\Vert \Delta \Vert^2)>0.
$$
 \end{lemma}

Let us assume now that $(z^i)_{i=1,.,d}$ are  i.i.d of law $z$ symmetric with $\E z=\E z^3=0$, $\E z^2=m=1$, $\E z^4/(\E z^2)^2=\beta$ and such that  $\Vert z\Vert_{\infty}$ is a.s. bounded by $0\leq R\leq1$.
Then the matrix $X$ satisfies a.s. assumption (A1).
Using multiple Hoeffding's inequalities we have
\begin{lemma}\label{lemma:sparseproba}
 If $z$ does not follow a Rademacher law, the design matrix $ X$ satsifies assumptions (A1,A2,A3,A4) with probability greater than $1-8d^2 \exp\big(-\frac{\delta^2 n (\beta-1)}{2d(\beta+d)R^4}\big)$ for $\mathcal S=\RR^d$, and with probability greater than $1-8d^2 \exp\big(-\frac{\delta^2 n \min\{\beta-1,2\}}{2dR^4}\big)$ for $\mathcal S=[1,c_{\min}1_{d-1}]^{\perp}$ where $c_{min}$ is defined in the proof and satisfies 
 \[
 \vert e_{\min}\vert \leq \frac{1}{d+\beta-3}.
 \]

\end{lemma}

This lemma concludes the proof of proposition \ref{prop:sparse}. We will now prove these two lemmas. 

\subsubsection{ Proof of lemma \ref{lemma:sparsedet}}

\begin{proof}
Since the dual variable $A$  for the PSD constraint is $0$ (see the proof of lemma \ref{lemma:wsols}), this constraint $W\succcurlyeq 0$ is not active and we will show that the function decreases in a set of directions $\Delta$ which include the one for which $V_*+\Delta\succcurlyeq 0$.

Therefore we  consider a direction $\Delta=\begin{pmatrix}
                                         a & b\t \\ b & C
                                        \end{pmatrix}$, with $C\succcurlyeq 0$, which  is slightly more general than $V_*+\Delta\succcurlyeq 0$.
We denote by $f(W)=\frac{2}{n}\sum_{i=1}^n\sqrt{x_i\t W x_i}-\frac{1}{n}\tr X\t X W$ the smooth part of $g$. By Taylor-Young, we have for all $W$:
$$
f(W)-f(W+\Delta)=-\langle f'(W), \Delta\rangle-\frac{1}{2}\langle \Delta, f''(W) \Delta\rangle+o(\Vert \Delta \Vert ^2).
$$
Thus:
$$
g(W)-g(W+\Delta)=-\langle f'(W), \Delta\rangle-\frac{1}{2}\langle \Delta, f''(W) \Delta\rangle+ \lambda (\Vert W+\Delta \Vert_1-\Vert W\Vert_1 )+o(\Vert \Delta \Vert ^2).
$$
In $W=V_*$ this gives with $X\t X=\begin{pmatrix}
                                   n & y\t Z \\ Z\t y& Z\t Z
                                  \end{pmatrix}$,
\BEAS
g(W)-g(W+\Delta)&=&-\lambda \langle \frac{X\t X}{n}, \Delta\rangle-\frac{1}{2}\langle \Delta, f''(V_*) \Delta\rangle+\lambda (a+2\Vert b \Vert_1+\Vert C\Vert_1) +o(\Vert \Delta \Vert ^2)\\ 
&=& \lambda \big[2(\Vert b \Vert_1- \frac{1}{n} b\t Z\t y)+\Vert C\Vert_1-\frac{1}{n}\tr(Z\t Z C)\big]-\frac{1}{2}\langle \Delta, f''(V_*) \Delta\rangle +o(\Vert \Delta \Vert ^2).\\
\EEAS
And with H\"{o}lder's inequality and assumption (A2)
$$
\Vert b \Vert_1- \frac{1}{n} b\t Z\t y\geq \Vert b \Vert_1(1-\Vert\frac{1}{n}  Z\t y \Vert_{\infty})\geq (1-\delta) \Vert b \Vert_1.
$$
Nevertheless we will show in lemma  \ref{lemma:det1} that $\Vert C\Vert_1-\frac{1}{n}\tr(Z\t Z C)\geq (1-\delta) \Vert C-\diag(C)\Vert_1$, thus
\begin{equation}\label{eq:cdiag}
g(W)-g(W+\Delta)\geq \lambda (1-\delta) (2\Vert b \Vert_1+\Vert C-\diag(C)\Vert_1 )+o(\Vert \Delta \Vert^2).
\end{equation}
However in \eq{cdiag}, $g(W)-g(W+\Delta)=0$ for $b=0$ and $C$ diagonal, therefore we  have to investigate second order conditions, i.e. to show for $\Delta=
                                        \diag(e)
                                        $ with  $e\in \RR^d$ that $-\langle \Delta, f''(V_*) \Delta\rangle >0$.

And with assumption (A4)
\BEAS
-\frac{4}{(1+\lambda)^3}\langle \diag(e), f''(V_*) \diag(e)\rangle&=&\frac{1}{n}\sum_{i=1}^n (x_i\t \diag(e) x_i)^2 \\
&=&\frac{1}{n}\sum_{i=1}^n (\sum_{j=1}^d e_j (x_i^j)^2)^2 \\
&=&\frac{1}{n}\sum_{i=1}^n e\t[x_i^{\odot 2} (x_i^{\odot 2})\t]e\\
&\geq& \lambda_{\min}\big( \frac{X^{\odot 2}(X^{\odot 2})\t}{n}\big)\Vert e\Vert^2\geq\zeta \Vert e\Vert_2^2.
\EEAS
Thus we can conclude:
$$
 g(W)-g(W+\Delta)\geq \lambda (1-\delta) (2\Vert b \Vert_1+\Vert C-\diag(C)\Vert_1 )+\zeta \frac{(1+\lambda)^3}{4}\Vert e\Vert_2^2+o(\Vert \Delta \Vert^2).
$$
\end{proof}
\subsubsection {Auxilliary lemma}

\begin{lemma}\label{lemma:det1}
For all matrix C symmetric semi-definite positive we have under assumptions (A1) and (A3):
$$
\tr\Big(S-\frac{Z\t Z}{n}\Big)C\geq (1-\delta) \Vert C-\diag(C) \Vert_1>0.
$$
\end{lemma}

\begin{proof}
We denote by $\Sigma^n=\frac{Z\t Z}{n}$.
We always have $\Vert C\Vert_1- \tr (\Sigma^n C )=\tr(S-\Sigma^n)C$ where $S_{i,j}=\sign(C_{i,j})$, thus if $\diag(C)>0$ then $\diag(S)=1$ and $\diag(S-\Sigma^n)\geq0$ from assumption (A1).
Moreover since ${\Sigma^n}_{i,j}\in[-1,1]$ then $\sign(S-\Sigma^n)=\sign(S)$.

Thus $\tr(S-\Sigma^n)C=\sum_{i} C_{i,i} (S-\Sigma^n)_{i,i}+\sum_{i\neq j} C_{i,j} (S-\Sigma^n)_{i,j}\geq\sum_{i\neq j} C_{i,j} (S-\Sigma^n)_{i,j}\geq0$. Furthermore from assumption (A3) $\vert (\Sigma^n)_{i,j}\vert \leq \delta$ for $i\neq j$. Therefore
$$
\tr(S-\Sigma^n)C\geq\sum_{i\neq j} C_{i,j} (S-\Sigma^n)_{i,j}\geq \sum_{i\neq j} \vert C_{i,j}\vert (1-\delta) \geq (1-\delta)  \Vert C-\diag(C) \Vert_1>0.
$$
If there is a diagonal element of $C$ which is $0$, then all the corresponding line and column in $C$ will also be $0$ and we can look at the same problem as before by erasing of $C$ and $\Sigma^n$ the corresponding column and line. 
\end{proof}

\subsubsection{Proof of lemma \ref{lemma:sparseproba}}

\begin{proof}
 We will first show that the noise matrix $ Z$ satisfies assumptions (A2,A3).
 By Hoeffding's inequality we have  with probability $1-2\exp(-\delta^2 n/(2R^2))$
$$
\frac{1}{n}\vert\sum_{i=1}^n z^j_i \vert\leq \delta.
$$
Then, since the law of $z$ is symmetric $y_iz_i$ will have the same law as $z_i$ and with probability $1-2\exp(-\delta^2 n/(2R^2))$, the design matrix $ Z$ satisfies assumption (A2):
$$
\Vert \frac{ Z \t y}{n}\Vert_\infty\leq \delta.
$$
Likewise we have with probability $1-2\exp(-\delta^2 n/(2R^4))$ that for $j\neq j'$
$$
\vert\frac{1}{n}\sum_{i=1}^n z^j_iz^{j'}_{i}\vert \leq \delta.
$$
Thus we also have with probability $1-2d^2\exp(-\delta^2 n/(2R^4))$ that $ Z$ satisfies assumption (A3):
$$
\Vert \frac{1}{n}  Z \t  Z-\diag(\frac{1}{n}  Z \t  Z)\Vert_\infty\leq \delta.
$$
Thus with probability $1-4d^2\exp(-\delta^2 n/(2R^4))$, the noise matrix $ Z$ satisfies assumptions (A1, A2, A3).

We proceed as in the proof of proposition \ref{prop:unicitew} to show that $ X$ satisfies assumption (A4).
We first derive a condition to have the result in expectation, then we use an inequality concentration on matrix to bound the empirical expectation. 
This will be very similar, but we will get a better scaling since $\Delta$ is diagonal. 

Using the same arguments as in the proof of proposition \ref{prop:unicitew} we have for the diagonal matrix $\Delta=\diag(e)$ with $e=(a,c)\in\RR^d$:
$$
e\t \E( x^{\odot 2}(  x^{\odot 2})\t)e=\E( x\t \Delta  x)^2=(a+m c\t1_{n-1})^2+m^2(\beta-1) \Vert c\Vert_2^2>0 \ \text{ if } \ \beta>1.
$$
 We can show  that $m^2(\beta-1)$ is an eigenvalue of multiplicity $d-2$ and $\mu_{\pm}$ are eigenvalues of multiplicity one of the operator $\Delta \mapsto \E (x\t\Delta x)^2$ with eigenvectors $e_\pm$ . 
Thus we have 
 \BEA
   \label{eq:mubal2}
   \lambda_{\min}(\E x^{\odot 2}(  x^{\odot 2})\t)&=&\frac{1+(d+\beta -2)m^{2}-\sqrt{(1+(d+\beta -2)m^{2})^{2}-4m^2(\beta-1)}}{2}\\
   &\geq& \frac {m^2(\beta-1)}{1+(d+\beta-2)m^{2}}  \nonumber,
   \EEA
and 
$$
  \lambda^{e_-^\perp}_{\min}(\E x^{\odot 2}(  x^{\odot 2})\t)=m^2(\beta-2).
$$
Moreover 
$$
\lambda_{\max}\Big(  x^{\odot 2}(  x^{\odot 2})\t\Big)=(  x^{\odot 2})\t  x^{\odot 2}\\
=\sum_{j=1}^d ( x_i)^4
\leq d R^4.
$$
Thus we can apply the Matrix Chernoff inequality from \citep{tropp} for $\mu_{\mathcal S}= \lambda^\mathcal S_{\min}(\E x^{\odot 2}(  x^{\odot 2})\t)$:
 $$
 \P\Bigg(\lambda^\mathcal S_{\min}\Big(\frac{X^{\odot 2}( X^{\odot 2})\t}{n}\Big)\leq (1-\delta )\mu_{\mathcal S}\Bigg)\leq de^{-\delta ^2 n\mu_{\mathcal S}/(2dR^4)}.
 $$

Thus with probability $1-5d^2\exp(-\delta^2 n \mu_{-}/(2d R^4))$ the design matrix $ X$ satisfies assumption (A1,A2,A3,A4) with $\zeta=(1-\delta)\mu_-$ and $\mathcal S=\RR^d$. And  with probability $1-5d^2\exp(-\delta^2 n\min\{\beta-1, 2\}  /(2d R^4))$ the design matrix $ X$ satisfies assumption (A1,A2,A3,A4) with $\zeta=(1-\delta)\min\{\beta-1, 2\}$ and $\mathcal S=e_-^\perp$.
\end{proof}

\section{Proof of multi-label results}

We first prove the lemma \ref{lemma:diagdiag}:

\begin{proof}
 Let $A\in \RR^{k\times k}$ symmetric semi-definite positive such that $\diag (\tilde y A \tilde y\t)=1_n$, then 
$$
 \diag (\tilde y A \tilde y\t)= \sum_{i=0}^k a_{i,i}1_n+2\sum_{i=1}^k a_{0,i} y_i +2 \sum_{1\leq i<j\leq k} a_{i,j} y_i \odot y_j
$$ 
thus 
$$
2\sum_{i=1}^k a_{0,i} y_i +2 \sum_{1\leq i<j\leq k} a_{i,j} y_i \odot y_j=(1-\sum_{i=0}^k a_{i,i})1_n
$$
And this system admits as unique solution $0_n$ if and only if the family $\{1_n,(y_i)_{1\leq i\leq k}, (y_iy_j)_{1\leq i<j\leq k}\}$ is \emph{linearly independent}.
\end{proof}
Then we  prove the lemma \ref{lemma:multilabel}:
\begin{proof}
Since $a_0+\sum_{i=1}^k a^2_i \alpha_i\geq \alpha_{\min} \sum_{i=0}^k a^2_i=\alpha_{\min}$ we should have $\alpha\geq\alpha_{\min}$.
 We have already seen that such $Y$ satisfies the constraint. The KKT conditions are: $B=\diag(\mu) - H -\nu 11\t \succcurlyeq0$ and $BY=0$.
 Since $y_i=\Pi_n y_i +\frac{(y_i\t 1_n)}{n} 1_n$.
\BEAS
Hy_i&=&H\Pi_n y_i +(y_i\t 1_n)H1_n\\
&=&\Pi_n y\\
&=&( y_i-\frac{1_n\t y_i}{n}1_n).
\EEAS 
Thus 
\BEAS
HY&=&\sum_{i=1}^ka^2_i H y_iy_i\t\\
&=&\sum_{i=1}^ka^2_i ( y_i-\frac{1_n\t y_i}{n}1_n)y_i\t\\
&=&\sum_{i=1}^ka^2_i ( y_iy_i\t-\frac{1_n\t y_i}{n}1_ny_i\t)\\\EEAS
and $\tr(HY)=\sum_{i=1}^k a^2_i(n-n\alpha_i)=n(1-a^2_0+a^2_0-\alpha)=n(1-\alpha)$.

Furthermore since $1_n \t \diag(Y)=n$ and $1_n\t M 1_n =n^2 \alpha$, for $\mu=1_n$ and $\nu=1/n$, $B.Y=n-n(1-\alpha)-n\alpha=0$. And since  $B=I_n-\frac{1}{n}1_n1_n\t -H$, $B^2=B$ and $B\t=B$, thus B is a symmetric projection and consequently symmetric semi-definit positive. 

Hence the primal variable $Y$ and the dual variables $\mu=1_n$ and $\nu=1/n$ satisfy the KKT conditions, thus $M$ is solution of this problem.
\end{proof}

\section{Efficient optimization problem}
\subsection{Dual computation }
\label{app:dualsmoothing}
We consider the following strongly-convex approximation of \eq{relaxYref}, augmented with the von-Neumann entropy:
\BEAS
\max_{V \succcurlyeq 0} \frac{1}{n} \sum_{i=1}^{n} \sqrt{( XVX^\top)_{ii}} -   \| \Diag(c)  V\Diag(c)  \|_1  - \varepsilon \tr [(A^\half V A^\half)\log(A^\half V A^\half)]   
\ {\rm{s.t.}} \ \tr (A^\half V A^\half) = 1. 
\EEAS
Introducing dual variables, we have 
\BEAS
\min_{u\in\rb^{n}_{+}, C:|C_{ij}| \leqslant c_i c_j} \ \max_{V \succcurlyeq 0}  && \frac{1}{2n} \sum_{i=1}^{n} \Big(u_i({( XVX^\top)_{ii}}) + \frac{1}{u_i}\Big) - \tr CV  - \varepsilon \tr [(A^\half V A^\half)\log(A^\half V A^\half)]  \\
{\rm{s.t.}} && \tr (A^\half V A^\half) = 1.
\EEAS
By fixing $u$ and $C$, and letting $Q=A^\half V A^\half$, we can write the $\max$ problem as 
\BEAS
\max_{Q \succcurlyeq 0} && \tr A^{-\half}(\frac{1}{2n}X^\top \Diag(u) X - C)A^{-\half}Q  - \varepsilon \tr [Q\log(Q)]  \\
{\rm{s.t.}} && \tr Q = 1.
\EEAS
This problem is of the form 
\BEAS
\max_{Q \succcurlyeq 0} \tr DQ  - \varepsilon \sum_{i=1}^{n} \sigma_i(Q)\log\sigma_i(Q)  \\
{\rm{s.t.}} \ \tr Q = 1 
\EEAS
where $D=A^{-\half}(\frac{1}{2n}X^\top \Diag(u) X - C)A^{-\half}$ and $\sigma_i(Q)$ denotes the $i$-th largest eigen value of the matrix $Q$. 
If we consider the matrix $D$ to be of the form $D=U\Diag(\theta)U^\top$ with $\theta$ denoting the vector of ordered eigen values of $D$, 
then it turns out that at optimality $Q$ has the form $Q=U\Diag(\sigma)U^\top$, with $\sigma$ denoting 
the ordered vector of eigen values of $Q$. 

Therefore the above optimization problem can be cast in terms of $\sigma$ as:
\BEAS
\max_{\sigma \in \rb^n} \theta^\top \sigma  - \varepsilon \sum_{i=1}^{n} \sigma_i\log\sigma_i  \\
{\rm{s.t.}} \ \sum_{i=1}^{n} \sigma_i = 1. 
\EEAS
The solution of this problem is $\sigma_i = \frac{e^{\theta_i / \varepsilon}}{\sum_{j=1}^{n}e^{\theta_j / \varepsilon}}$, which leads to 
\BEAS
\min_{\theta \in \rb^n} \phi^{\varepsilon}(\theta) = \varepsilon \log \sum_{i=1}^{n} \Big(e^{\frac{\theta_i}{\varepsilon}}\Big).  
\EEAS
In terms of the original matrix variables, we have 
\BEAS
\min \phi^{\varepsilon}(D) = \varepsilon \log \tr e^{\frac{D}{\varepsilon}}.  
\EEAS
Using the appropriate expansion of $D$, we have the overall optimization problem as 
\BEQ
\min_{u\in\rb^{n}_{+}, C:|C_{ij}| \leqslant c_i c_j}  \frac{1}{2n} \sum_{i=1}^{n} \frac{1}{u_i} + \phi^{\varepsilon}(A^{-\half}(\frac{1}{2n}X^\top\Diag(u)X - C)A^{-\half}).
\label{eq:overall_opt}
\EEQ

At optimality, we have 
\BEAS
A^{\half}VA^{\half} = \Big(e^{\frac{(A^{-\half}(\frac{1}{2n}X^\top\Diag(u)X - C)A^{-\half})}{\varepsilon}}\Big) / \tr \Big(e^{\frac{(A^{-\half}(\frac{1}{2n}X^\top\Diag(u)X - C)A^{-\half})}{\varepsilon}}\Big).
\EEAS
The error of approximation is at most $\varepsilon\log d$ and the Lipschitz constant associated with the function $\phi^{\varepsilon}(\cdot)$ is $\frac{1}{\varepsilon}$. 

\subsection{Algorithm details}
We write the optimization problem~\eq{overall_opt} as:
\BEAS
\min_{u\in\rb^{n}_{+}}  F(u,C) + H(u,C) 
\EEAS 
where 
\BEAS
H(u,C)= \phi^{\varepsilon}(A^{-\half}(\frac{1}{2n}X^\top\Diag(u)X - C)A^{-\half})
\EEAS
is the smooth part and 
\BEAS
F(u,C) = \mathbb{I}_{C:|C_{ij}| \leqslant c_i c_j} +  \frac{1}{2n} \sum_{i=1}^{n} \frac{1}{u_i}
\EEAS
is the non-smooth part. 

\label{app:algo}
The gradient $\nabla_u$ of $H(u,C)$ with respect to $u$ is 
\BEAS
\nabla_u = \diag(B^\top U \Diag(\sigma) U^\top  B). 
\EEAS
where $B=\frac{1}{\sqrt{2n}}A^{-\half}X^\top$
and the gradient of $H(u,C)$ with respect to $C$ is 
\BEAS
\nabla_C = (A^{-\half} U \Diag(\sigma) U^\top  A^{-\half}). 
\EEAS

The Lipschitz constant  $L$ associated with the gradient $\nabla H(u,C)$ is 
\begin{align}
L = \frac{2}{\varepsilon} \max\Big(\lambda_{max}(B^\top B \odot B^\top B), \lambda_{max}^2(A^{-1}) \Big),
\end{align}
where $\lambda_{max}(M)$ denotes the maximum eigen value of matrix $M$. Computing $L$ takes $O(\max(n,d)^3)$ time and $L$ needs to be computed once at the beginning of the algorithm. 

\begin{algorithm}[!htp]
	\caption{\textit{FISTA Algorithm to solve~\eq{overall_opt}}}
	\label{fista_algo}
 	\begin{algorithmic}[1]
 	\renewcommand{\algorithmicrequire}{\textbf{Input:}}
 	\State Input $X$. 
 	\State Compute Lipschitz constant $L$.  
	\State Let $(u^0, C^0)$ be an arbitrary starting point. 
	\State Let $(\baru^0, \barC^0) = (u^0, C^0)$, $t_0 = 1$. 
	\State Set the maximum iterations to be $K$. 
	\For{$k = 1,2,\ldots,K$} \Comment{The loop can also be terminated based on duality gap.} 
	  \State $(\baru^{k-\half}, \barC^{k-\half})  = \Big(\baru^k - \frac{1}{L} \nabla_{\baru^k}, \barC^k - \frac{1}{L} \nabla_{\barC^k}\Big).$ 
	  \State Obtain $u^k = \argmin_{u\in\rb^{n}_{+}}  \Big \{ \frac{L}{2} \|u-\baru^{k-\half}\|^2 +  \frac{1}{2n} \sum_{i=1}^{n} \frac{1}{u_i}  \}$ by Algorithm~\ref{newton_algo}.  
	  \State Obtain $C^k = \argmin_{C}  \Big \{ \mathbb{I}_{C:|C_{ij}| \leqslant c_i c_j} + \frac{L}{2} \|C - \barC^{k-\half} \|_F^2 \Big \}$ by thresholding.  
	  \State $t_k = \frac{1+\sqrt{1+4t_{k-1}^2}}{2}$. 
	  \State $(\baru^k, \barC^k) = (u^k, C^k) + \frac{(t_{k-1} - 1)}{t_k} \Big( (u^k, C^k) - (u^{k-1}, C^{k-1})\Big)$. 
	\EndFor
	\State Output $(u^K, C^K)$. 
        \end{algorithmic} 
        \end{algorithm}
        
\begin{algorithm}[!htp]
	\caption{\textit{Newton method to solve $u$ sub-problem}}
	\label{newton_algo}
 	\begin{algorithmic}[1]
 	\renewcommand{\algorithmicrequire}{\textbf{Input:}}
 	\State Input $u^{k-\half}$, $n$, $L$.
 	\State $u_i^0 = \max(u_i^{k-\half},\frac{1}{(2nL)^{\frac{1}{3}}}), \ i=1,2,\ldots,n$. 
 	\State Set $\mathcal{M}$ to be the max number of Newton steps. 
 	\For {$t=1,2,\ldots,\mathcal{M}$}
	  \For {$i=1,2,\ldots,n$} 
	    \State $u_i^t = \frac{2nL(u_i^{t-1})^3u_i^{k-\half} + 3u_i^t} {2(nL(u_i^{t-1})^3 +1)}$. 
	  \EndFor
 	\EndFor
 	\State Output $\max(u^\mathcal{M},0)$. 
        \end{algorithmic} 
\end{algorithm}
        
The resultant FISTA procedure is described in Algorithm~\ref{fista_algo}. Note that the FISTA procedure first computes intermediate iterates $(\baru^{k-\half}, \barC^{k-\half})$ (Step 7, Algorithm 1) by taking descent steps along the respective gradient directions. Then two distinct problems in $u$ and $C$ (respectively Steps~8 and 9 in Algorithm 1) are solved. The sub-problem in $u$ (Step 8) can be efficiently solved using a Newton procedure followed by a thresholding step, as illustrated in Algorithm~\ref{newton_algo}. The sub-problem in $C$ (Step 9) can also be solved using a simple thresholding step. 

\end{document}